\documentclass[10pt,journal,compsoc]{IEEEtran}
\usepackage{microtype}
\usepackage{graphicx}
\usepackage{booktabs}
\usepackage{bm}
\usepackage{bbm}
\usepackage{amsmath}
\usepackage{amsthm}
\usepackage{commath}
\usepackage{amssymb}
\usepackage{array,multirow}
\usepackage{multicol}
\usepackage{tabu}
\usepackage{color}
\usepackage[bb=boondox]{mathalfa}
\usepackage[tableposition=top]{caption}
\usepackage[table]{xcolor}
\usepackage[font={footnotesize}]{caption}
\usepackage{float}
\usepackage{blindtext}
\usepackage{graphicx,float}
\usepackage{wrapfig}
\usepackage{parskip}
\usepackage{enumitem}
\usepackage{booktabs}
\usepackage{color}
\usepackage[tight,footnotesize]{subfigure}
\newcommand{\beps}{\boldsymbol{\varepsilon}}
\newcommand{\eps}{{\varepsilon}}
\theoremstyle{plain}
\newtheorem{prop}{Proposition}
\newtheorem{lemma}{Lemma}
\newtheorem{obs}{Observation}% for professional tables
\definecolor{Gray}{gray}{0.80}
\newcolumntype{a}{>{\columncolor{Gray}}c}
\definecolor{rev_color}{rgb}{0,0,0}

\usepackage{setspace}
\ifCLASSOPTIONcompsoc
  \usepackage[nocompress]{cite}
\else
  \usepackage{cite}
\fi
\ifCLASSINFOpdf
\else
\fi
\begin{document}
\title{Stopping Criterion Design for Recursive Bayesian Classification:\\
Analysis and Decision Geometry}

\author{Aziz~Ko\c{c}anao\u{g}ullar\i,
        Murat Akcakaya,
        Deniz Erdo\u{g}mu\c{s}% <-this % stops a space
\IEEEcompsocitemizethanks{\IEEEcompsocthanksitem Aziz Ko\c{c}anao\u{g}ullar\i \ and Deniz Erdo\u{g}mu\c{s} were with Northeastern University Department of Electrical and Computer Engineering 409 Dana Research Center 360 Huntington Avenue Boston, MA 02115.

E-mail: akocanaogullari@ece.neu.edu, erdogmus@ece.neu.edu
\IEEEcompsocthanksitem Murat Akcakaya was with Pittsburg University Department of Electrical and Computer Engineering 1238 Benedum Hall Pittsburgh, PA 15261.

E-mail:akcakaya@pitt.edu}% <-this % stops an unwanted space
\thanks{This work is supported by NSF (IIS-1149570, CNS-1544895, IIS-1715858, IIS-1717654, IIS-1844885, IIS-1915083), DHHS (90RE5017-02-01), and NIH (R01DC009834).}
\thanks{\scriptsize{This work is accepted in IEEE Transactions on Pattern Analysis and Machine Intelligence (DOI: 10.1109/TPAMI.2021.3075915).

\noindent
© 20XX IEEE.  Personal use of this material is permitted.  Permission from IEEE must be obtained for all other uses, in any current or future media, including reprinting/republishing this material for advertising or promotional purposes, creating new collective works, for resale or redistribution to servers or lists, or reuse of any copyrighted component of this work in other works.}}}

\IEEEtitleabstractindextext{%
\begin{abstract}
Systems that are based on recursive Bayesian updates for classification limit the cost of evidence collection through certain stopping/termination criteria and accordingly enforce decision making. Conventionally, two termination criteria based on pre-defined thresholds over (i) the maximum of the state posterior distribution; and (ii) the state posterior uncertainty are commonly used. In this paper, we propose a geometric interpretation over the state posterior progression and accordingly we provide a point-by-point analysis over the disadvantages of using such conventional termination criteria. For example, through the proposed geometric interpretation we show that confidence thresholds defined over maximum of the state posteriors suffer from stiffness that results in unnecessary evidence collection whereas uncertainty based thresholding methods are fragile to number of categories and terminate prematurely if some state candidates are already discovered to be unfavorable. Moreover, both types of termination methods neglect the evolution of posterior updates. We then propose a new stopping/termination criterion with a geometrical insight to overcome the limitations of these conventional methods and provide a comparison in terms of decision accuracy and speed. We validate our claims using simulations and using real experimental data obtained through a brain computer interfaced typing system.
\end{abstract}

}

\maketitle

\IEEEdisplaynontitleabstractindextext

% \fi
\IEEEpeerreviewmaketitle


\section{Introduction}
% Introduction
Recursive Bayesian inference for classification (RBC) is beneficial in gradually increasing decision quality by incorporating more evidence into the decision process in scenarios where data or evidence is acquired sequentially over time. The most recent belief, represented by the latest label posterior probability distribution, is obtained by incorporating new evidence in a Bayesian manner at each update step~\cite{duda2012pattern,van2013detection,sarkka2013bayesian,alaa2016balancing}. The trade-off between decision confidence and evidence acquisition cost is controlled by a stopping criterion that controls when to terminate evidence collection and return an estimated class label. Fundamental components of RBC include:  
$(S)$ A stopping criterion based on the posterior probability to stop evidence collection; $(Q)$ a querying step to decide how to collect further evidence from relevant sources to benefit speed and accuracy objectives of RBC; $(C)$ a classification objective based on the posterior distribution and loss values attributed to each true label and decision option pair to determine the optimal decision once the stopping criterion has been satisfied. The iterative process can be summarized as follows:
    \vspace{-0.07in}
    \begin{equation*}
            \textbf{while } (S) \text{ not satisfied} \lbrace \textbf{do } (Q) \rbrace;
            \textbf{return } \text{state with } (C) 
            \vspace{-0.07in}
    \end{equation*} 
In this paper, we focus on designing stopping criteria for $(S)$ based on the latest label posterior {\color{rev_color}distribution} $p=[p_1,p_2,\cdots,p_n]$, which is a categorical probability distribution. In the remainder of this paper, in order to keep the illustrations and derivations simple, we assume that the objective in (C) is to minimize probability of error (therefore 0-1 loss is assumed), thus when a decision is made, it will be based on the maximum a posteriori (MAP) classification rule; that is the then-most-likely-label in the latest posterior will be selected as the decision. The presented approach can be generalized to the more general expected loss minimization classification setting {\color{rev_color}\cite{vapnik1992principles,lugosi1995nonparametric}}, but is outside the scope of this paper. We also do not discuss different querying strategies {\color{rev_color}\cite{higger2017recursive,tong2001support,wilson2012bayesian,tsiligkaridis2014collaborative}}, as they too are outside the scope.

{
\begin{table*}[!t]
\begin{center}
\begin{tabular}{c c c c c}
\multicolumn{4}{c}{Stopping Region: $S_R:= \lbrace p| \mathcal{S}_C(p) = true \rbrace$} \vspace{0.05cm} \\
\hline
Identifier &Root &Stopping Criterion $(\mathcal{S}_C)$ &Reference\\ 
(M1) &Confidence Level & $\max_{\sigma}p(\sigma\lvert\mathcal{H}_s)>\tau$ &\cite{jha2009bayesian}\\
(M2) &Renyi Entropy ($\alpha=2$) & $H_\alpha(p(\sigma\lvert\mathcal{H}_s)) < c_{H_2}$ &\cite{sason2017arimoto}\\ 
(M3) &Shannon's Entropy & $H(p(\sigma\lvert\mathcal{H}_s)) < c_{H}$ &\cite{geisser2017predictive}\\ 
(M4) &Renyi Entropy ($\alpha=0.2$) & $H_\alpha(p(\sigma\lvert\mathcal{H}_s)) < c_{H_{0.2}}$ &\cite{sason2017arimoto}\\ 
(M5) &Kullback Liebler Div. & $\delta_{\text{KL}}(p(\sigma\lvert\mathcal{H}_s), p(\sigma\lvert\mathcal{H}_{s-1})) <c_{\delta_{\text{KL}}}$ &\cite{golovin2010near} \vspace{0.05cm}\\ 
\hline
\end{tabular}
\end{center}
\caption{Conventional stopping criteria, here $c$s $\in\mathbb{R}$ denote the limits of the criteria. These methods are further used in the experiments section with the respective identifiers as a baseline to the proposed perspective of the stopping method. Limiting criterion $\mathcal{S}_C$ defines a set with all distributions $p$ that satisfy the condition. In recursive estimation, if the posterior probability becomes an element of the set, the system terminates to make an inference/classification.}
\label{tab:competitor_methods_table}
\vspace{-0.6cm}
\end{table*}}

The RBC procedure will terminate when the stopping criterion in (S) (namely $(\mathcal{S}_C$) is met by the current label posterior distribution ($\mathcal{S}_C(p)=\text{\emph{true}}$). The most commonly used approach is to require that a confidence threshold has been exceeded ($\mathcal{S}_C(p)=\text{\emph{true} if }\max_i p_i > \tau$). All label probability distributions $p$ that satisfy $\mathcal{S}_C$ form a set called the stopping region $\mathcal{S}_R = \lbrace p | \mathcal{S}_C(p)=\text{\emph{true}}\rbrace$. If the posterior distribution falls into this set, the system terminates and a classification-decision is made based on (C). 
{\color{rev_color}Uncertainty termination rule is another canonical stopping decision that utilizes classification uncertainty (e.g. $\mathcal{S}_C(p)=\text{\emph{true} if } \ H_\alpha(p) < c_{H_\alpha}, \text{ where } H_\alpha(.) \text{ is Renyi entropy} )$. Disfavoring candidate categories (assigning small probability masses in the latest posterior) will yield lower classification uncertainty. However, this is independent of the classification confidence level and hence can sacrifice accuracy. In this paper, we analyze $\mathcal{S}_C$ with the aid of analytical representations that arise from the geometry of the probability simplex. A probability simplex for three categories ($\{ a,b,c \}$) is illustrated in Fig.~\ref{fig:simplex_repr}. Confidence and entropy decision boundaries are marked in the illustration with (M1) and (M3) respectively. Especially, to achieve high accuracy in classification, a high confidence threshold is usually preferred. However, based on the prior information if the true class is one of the least favored, confidence on true class increases gradually and hence through RBC the system sacrifices speed. This is illustrated in Fig.~\ref{fig:simplex_repr}-(a) where the prior probability starts away from the corner of interest and gradually gets closer to $a$. Uncertainty based approaches are more flexible in termination but sensitive to number of classes. Considering drawbacks of conventional stopping criteria, it is preferred to optimize stopping in different RBC scenarios. }
The design of an optimal stopping region ($\mathcal{S}_R$, see Table~\ref{tab:competitor_methods_table}) has been referred to as "\textit{optimal stopping criterion design}"~\cite{shiryaev2007optimal}.
{\color{rev_color} Optimum $\mathcal{S}_R$ design has many practical applications. For example, in any resource constrained state estimation or recursive classification setting, stopping criteria could be preferred to be optimized to minimize time requirements of the system with a marginal expense in estimation confidence/accuracy. Some specific example applications include but not limited to, radar tracking with potential adaptive waveform design and radar subset selection or optimum radar placement  \cite{irci2010study,charlish2017cognitive,xiang2019target},  brain-computer interfaced typing systems \cite{marghi2019history}, text recognition from video stream \cite{bulatov2019optimal}.}

% In this 
{\color{rev_color}{In this paper we propose a stopping criterion building on the strengths of the existing conventional criteria but improving their shortcomings. That is, the method we propose is robust to potential decrease in accuracy while promoting early stopping.} Motivated by problem geometry and RBC mechanics, we propose a stopping criteria function $S_C(p)$ that yields a bent-in stopping boundary towards the uniform distribution ($u_n$ in Fig.\ref{fig:simplex_repr}-(d)) as illustrated by the blue curves in Fig.\ref{fig:simplex_repr}-(c). Unlike conventional methods, this enables early stopping decisions especially in the absence of a strong contender. When there exists a strong contender (posterior moving near the edges of simplex) it does not sacrifice accuracy.} 

\subsection*{Problem Formulation and Related Work}
\label{subsec:problem_formulation}

Let class label $\sigma$ be an element of a finite set $\mathcal{A}$. We refer to each iteration in RBC as a \textit{sequence}, indexed by $s\in \mathbb{N}$. Each sequence may include a set of measurements acquired in response to queries $\Phi_s \triangleq \{\phi^{1}_{s}, \dots, \phi^{N}_{s}\}$, where $N$ denotes number of queries. These mesurements provide evidence (raw data or processed features) $\beps_s \triangleq \{\eps^{1}_{s}, \dots, \eps^{N}_{s}\}$ conditioned on the queries, and the true label. RBC posterior updates use this evidence. For notation simplicity, we use $\mathcal{H}_{s} \triangleq \{\beps_{1:s}, \Phi_{1:s}, \mathcal{H}_0\}$ to represent the combination of all evidence collected in sequences 1 to $s$, as well as the prior, $\mathcal{H}_0$.
%
%
% TODO: make the aspect ratio equal for the figures.
\begin{wrapfigure}{r}{3cm}
\centering
\vspace{-0.51cm}
\includegraphics[width=1\linewidth]{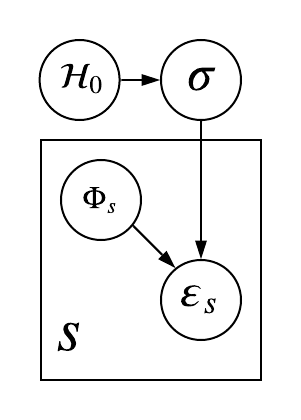}
\caption{\color{rev_color}Proposed probabilistic graphical model representing the $s^{th}$
sequence for a RBC.}
\vspace{-0.3cm}
\label{fig:graphicalModel}
\end{wrapfigure}
{\color{rev_color} The relationship between RBC variables is illustrated in Fig.~\ref{fig:graphicalModel}. Through this paper we assume evidences conditioned on state and query tuples are independent of each other $p(\beps_i,\beps_j | \sigma, \Phi_i,\Phi_j) = p(\beps_i\lvert \sigma, \Phi_i)p(\beps_j\lvert \sigma, \Phi_j) \ \forall i\neq j$. There are no limiting assumptions that are set on prior information $p(\sigma\lvert \mathcal{H}_0)$. The importance of the prior information in the RBC paradigm is that as prior probability is the initial confidence assessment of the estimation, it determines the evolution of the estimation as it will be discussed in Sec.\ref{subsec:posterior_motion}.}
RBC steps (assuming MAP classification rule) are {\color{rev_color}presented in the following}:

\begin{equation}
    \label{eq:convRBE}
    \begin{array}{ll}
    (S): & p(\sigma| \mathcal{H}_s) \in (\mathcal{S}_R:=\lbrace{p | \mathcal{S}_C(p)=\text{\emph{true}}}\rbrace) \\[.05in]
    (C): &\hat{\sigma} =  \displaystyle{\arg\max_{\sigma\in\mathcal{A}}} \  p(\sigma| \mathcal{H}_s) \\
    (Q): &\Phi_{s+1}\rightarrow \varepsilon_{s+1}~\text{with the anticipated joint}\\
    &p(\sigma\lvert\mathcal{H}_{s+1}) = p(\sigma\lvert\mathcal{H}_{s})\frac{p(\varepsilon_{s+1} \lvert \sigma, \Phi_{s+1})}{p(\varepsilon_s| \Phi_{s+1})} 
    \end{array}
\end{equation}
\vspace{-0.15in}

\begin{figure*}[t!]
      \centering
      \subfigure[]{\includegraphics[width=.48\columnwidth]{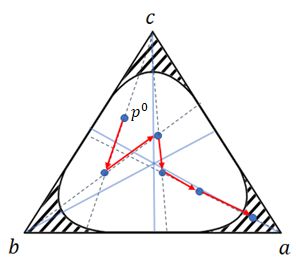}}
        \subfigure[]{\includegraphics[width=.48\columnwidth]{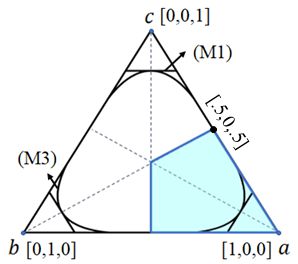}}
        \subfigure[]{\includegraphics[width=.48\columnwidth]{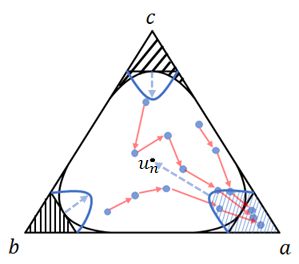}}
        \subfigure[]{\includegraphics[width=.52\columnwidth]{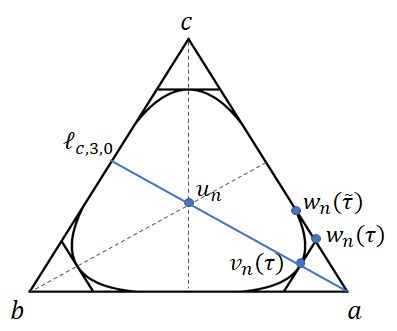}}
        \caption{Decision geometry for recursive classification where true class is $a$. (a) The initialization of the system with a prior probability ($p^0$) and $(S)$ region. The system terminates once the posterior yields within the $(S)$ region (dashed). (b) {\color{rev_color}Stopping boundaries for (M1) and (M3) are visualized. Highlighted region is the decision region for '$a$' enforced by $(C)$.} (c) Three different perspectives in $S_R$ design {\color{rev_color}(entropy:diagonal-stripes/top, confidence:vertical-stripes/left, proposed:blue-stripes/right)}, we want to tilt the {\color{rev_color} confidence} boundary {\color{rev_color} towards $u_n$ as illustrated with dashed arrows} to achieve early stopping. (d) The origin of the simplex $u_n$ and special points defined in \eqref{eq:special_probs_in_simplex}. Equi-entropy contours and corresponding posterior threshold lines are in solid black. The intersection point is at $v_n$ (See Observation~\ref{obs:intersection_entropy_conf}).}
        \label{fig:simplex_repr}
        \vspace{-0.5cm}
\end{figure*}
{In this paper we focus on $(S)$.} {\color{rev_color} Specificially, we discuss the limitations of the conventional criteria summarized in Table~\ref{tab:competitor_methods_table}, and propose a new family of stopping criteria, then illustrate theoretically, and with numerical examples and experiments, the benefits of the proposed approach.}
The most common $\mathcal{S}_C$ is thresholding the highest value in the latest posterior by $\tau$ : $\mathcal{S}_C(p)=\text{\emph{true} if }\max_i p_i > \tau$ (M1) \cite{lai1997optimal,jha2009bayesian}.
This rule directly enforces that a decision is made with a pre-specified confidence level, and does not consider the distribution of probability mass among other options.
{To minimize false alarm rate, the threshold is usually set to a large number.
{\color{rev_color}This prevents early stopping; and to achieve high accuracy, furthermore leads to}
redundant evidence collection \cite{lai2001sequential}.
{\color{rev_color}On the other hand, information theoretic objectives utilize more than a single category in termination decision. Unlike confidence based approaches, uncertainty based stopping allows early stopping if the system uncertainty is below a desired threshold where the confidence is not yet high enough for confidence stopping.}}
Uncertainty measures, including Shannon entropy~\cite{shannon1948mathematical}, can be used to have termination based on the spread of posterior probability mass. Golovin~\cite{golovin2010near} uses Shannon entropy as a stopping criterion: "$\mathcal{S}_C(p)=\text{\emph{true} if } H(p)<c_{H}$" (M3). Yingzhen~\cite{li2016renyi} and Igal~\cite{sason2017arimoto} use Renyi entropy: "$\mathcal{S}_C(p)=\text{\emph{true} if } \ H_\alpha(p) < c_{H_\alpha}$". Uncertainty based stopping criteria based on Renyi entropy $H_\alpha(p)$ family with $\alpha \geq 0$ includes, as special cases, Shannon entropy (in the limit as $\alpha \rightarrow 1$), and confidence thresholding (in the limit as $\alpha \rightarrow \infty$) methods (M2), (M4). 
{\color{rev_color}It is possible to terminate if the changes in posterior distribution between sequences become smaller over iterations}
, Weinshall~\cite{weinshall2009beyond} and Geisser~\cite{geisser2017predictive} apply a threshold to the Kullback-Leibler (KL) divergence between two consecutive posterior distributions to terminate evidence collection:  "$\mathcal{S}_C(p)=\text{\emph{true} if } D_{KL}(p_s) \lvert \lvert p_{s-1})<c_{D_{KL}}$",
{\color{rev_color} therefore terminate RBC when the distance between two consecutive classification posteriors becomes small (M5).}
Pavlichin~\cite{pavlichin2016chained} proposes a chained-KL divergence to monitor posterior progression in the probability simplex. Banerjee~\cite{banerjee2005clustering} proposes using Bregman divergences in the context of clustering. All of these stopping criteria focus on how far the posterior is from the uniform distribution that is at the center of the probability simplex, as opposed to assessing how close the posterior to a vertex of the simplex (corresponding to a one-hot distribution, see Fig.~\ref{fig:simplex_repr}-(b))~\cite{pawlowsky2001geometric}.

\subsection*{Contributions}
We introduce a geometrical representation for recursive Bayesian classification {\color{rev_color}in Sec.~\ref{sec:Problem-geo}}{. \color{rev_color}In Sec.~\ref{sec:PreviousWork}} we show that: (i) uncertainty based methods are sensitive to the number of possible classes, and (ii) conventional approaches ignore the posterior update trajectory in the recursive classification task. Using (i) and (ii) together, we  show that the stopping regions defined based on uncertainty methods diverge from the region formed by confidence threshold defined over the posterior distribution resulting in decrease in classification accuracy. 
{\color{rev_color}In Sec.~\ref{sec:Method}, we propose a stopping criterion approach which utilizes the strengths of conventional stopping criteria but robust to their drawbacks. Our approach is also inline with the progression of the posterior update (blue curves in Fig.\ref{fig:simplex_repr}-(c)) to achieve this balance between speed and accuracy compared to other approaches. Throughout Sec.~\ref{sec:Experiments} first with synthetic experiments we show that our approach allows RBC to stop earlier where there is no close competitor. This results in faster decision cycles without sacrificing accuracy. Therefore, our approach maintains a high classification accuracy where uncertainty based methods stop early with a significant accuracy loss. We also include results from a brain-computer interfaced typing system in which the user is trying to type a letter with electroencephalogram (EEG) responses to stimuli. Due to observation noise, multiple EEG acquisitions are required. Accurate Bayesian classification of the target letter relies on prior information provided by a language model and recursive evidence collection. The language model and the EEG characteristics of the user determines the path to classification in simplex. We show that with the proposed method, it is possible to commit to a letter sooner with marginal accuracy loss.}

Detailed demonstrations and proofs of the analytical results are provided in the appendix Sec.~\ref{Sec:Proofs} for neat presentation.

%

%% Explaining the problem geometry
\section{Problem geometry}
\label{sec:Problem-geo}

To promote the representative power of visualization, in this paper we use information geometric representation of recursive classification problem \cite{amari2007methods}. As will be discussed, this representation allows us to represent progression of the posterior probability distribution $p(\sigma \lvert \mathcal{H}_s)$ as the evidence collection steps $s$ increases and to introduce a mean for analytical reasoning for stopping criterion $\mathcal{S}_C$ design. We use simplex as a domain for probability distributions;
\begin{equation}
    \label{eq:SimplexRepr}
    \Delta_n = \lbrace (p_1,p_2,\cdots,p_n)\in\mathbb{R}^n \lvert p_i>0 \ \forall i, \sum_i p_i = 1  \rbrace
\end{equation}
Here $\Delta_n$ represents the set that includes categorical probability distributions, where the simplex is an $n-1$ dimensional geometrical object in $\mathbb{R}^n$. We visualize an example simplex for a category with 3 elements in Fig.~\ref{fig:simplex_repr}. To preserve neat visualization, we interpret our reasoning using the triangle throughout the paper, but all the results can be generalized to n-dimensional case without loss of generality. Moreover, we use the following addition operation that allows us to represent posterior updates;

\textit{Addition:} Given $p,q\in\Delta_n$ the addition operation is defined as:
    \begin{equation}
    \label{eq:simplex_add}
    \begin{split}
        p \oplus q = \frac{[p_1q_1,p_2q_2,\cdots,p_nq_n]}{\sum_i p_iq_i}
        \end{split}
    \end{equation}
Observe that addition satisfies the Bayes' theorem: $p(\sigma \lvert \varepsilon) = p(\sigma) \oplus p(\varepsilon\lvert\sigma)$ which allows us to represent the evolution of the posterior distribution using the addition operation within the simplex. For algebra to work $p(\varepsilon\lvert \sigma)$ does not need to be unit $\ell_1$ norm and hence $\not\in\Delta_n$, but what we interpret here is the normalized vector over $\sigma$ even algebraically $\oplus$ is applicable. Rigorous definition of this simplification is shown in the appendix Sec.~\ref{sec:how_algebra_works}.The zero element in $\Delta_n$ is denoted by $u_n$ being the uniform distribution as shown in Fig.~\ref{fig:simplex_repr}-(d).

\begin{prop}[\cite{aitchison1982statistical}] \label{prop:SimplexVSpace} 
Simplex $\lbrace \Delta_n \lvert \oplus,\otimes \rbrace$ forms a vector space.
\end{prop}
In the presence of a prior distribution, recursive classification starts with a distribution probably different than uniform distribution as presented in Fig.~\ref{fig:simplex_repr}-(a). To make a correct decision, the posterior for respective element is required to be the most likely (e.g. the region for $a$ is highlighted in Fig.~\ref{fig:simplex_repr}-(b)). Through evidence collections, the posterior distribution moves within the simplex visualized in Fig.~\ref{fig:simplex_repr}-(c) that is algebraically denoted with $\oplus$ above. $(S)$ criterion forms the decision lines within the simplex and the inequality condition covers the area of evidence collection termination that are visualized by dashed areas in Fig.~\ref{fig:simplex_repr}. Once the posterior probability distribution reaches that area, the system terminates evidence collection, makes an inference outputs the classification decision.

\section{Conventional Stopping Criteria}
%Entropy based decision boundary

\label{sec:PreviousWork}

In this section we discuss the limitations of the conventional methods that are used for  stopping criteria $\mathcal{S}_{C}$ design. Consider the following 2 motivating examples assuming 10 class classification problems;

\textit{example1}: $p,q\in\Delta_{10}$ 
{\color{rev_color}(i.e. 10 dimensional probability mass functions) }
such that $p=[.6,.4-8\epsilon,\epsilon,\cdots]$ where $0<\epsilon<<1$ and $q=[.7,.0\bar{3},\cdots]$. It is apparent that $q$ is a better stopping point to decide on the $1^\text{st}$ element as the classification result. However, if Shannon's entropy is used to measure the distance from uniform distribution $u_n$, and accordingly utilized as a stopping criterion, then  $0.97\approx H(p)<H(q)\approx 1.82$. This means that the evidence collection would have been terminated when $p$ is reached. Hence, the uncertainty based $(S)$ suffers from not taking the probability mass index of interest into consideration.

{\color{rev_color}\textit{example2}: $p,q,x,y\in\Delta_{10}$ with $p=[.5,.5-8\epsilon,\epsilon,\cdots]$, $q=[.6,.4-8\epsilon,\epsilon,\cdots]$ where $0<\epsilon<<1$ and $x=[.5,0.0\bar{5},\cdots]$, $y=[.6,0.0\bar{4},\cdots]$ and we compare $p$ with $q$ then $x$ with $y$. Assume that $p$ evolved into $q$ and $x$ evolved into $y$ through Bayesian recursions. Therefore, the propagation from p to q, and x to y, correspond to different paths on the probability simplex. Specifically, from p to q, the path is close to the edges (See Fig.~\ref{fig:simplex_repr}-(b) for 3 class example with [0.5,0,0.5]), and the path from x to y is close to the center. Observe that the confidence level for the true class (the $1^\text{st}$ class) increased the same amount $(.1)$ for both of these cases. However for $p$ and $q$ the $2^\text{nd}$ best class is still a close competitor having almost all the remaining probability mass. Whereas for $x$ and $y$ there is no close competitor. Therefore, even for these two cases the confidence levels for the true class is the same, one can imply that the in case 2 ($x$, $y$) we are still confident there exists no strong competitor. This means that the location of posterior distributions in $\Delta_n$ matters and with an appropriate design of stopping criterion, the system evidence collection could be stopped when there exists no close competitor.}

In this section we analytically show what these examples mean for the stopping criterion $(\mathcal{S}_C)$ design.

\subsection{Frail Confidence}
\label{subsec:frail_confidence}
The confidence behavior of uncertainty based methods are similar. Shannon entropy on the other hand is the most commonly used and hence in this section we specifically focus on (M3), as defined in Table.1.  We refer the reader to the appendix Sec.~\ref{sec:why_entropy_is_enough} for decision boundary similarities between (M2)-(M3) and (M4). To point out the relationship between the confidence in classification and  entropy based stopping region $S_R(H(.)):= \lbrace p | H(p)<c_H \rbrace$ (M3) we define two special probability points in $\Delta_n$;
\begin{equation}
\label{eq:special_probs_in_simplex}
    \begin{split}
    v_n(\tau) = \left[\tau, \frac{1-\tau}{n-1},\cdots,\frac{1-\tau}{n-1}\right]\\
    w_n(\tau) = \left[\tau, 1-\tau,0,\cdots, 0\right]
    \end{split}
\end{equation}
Here, $v_n(\tau)$ corresponds to distributions where one class has likelihood of $\tau$ and the others share the remaining probability uniformly, $w_n(\tau)$ corresponds to distributions where only two class exist with probability values $\tau, 1-\tau$. These definitions are also visualized in Fig.~\ref{fig:simplex_repr}-(d).

In this paper, we bundle the uncertainty and confidence $\tau$ via $\tau' = H(v_n(\tau))$. We show the following weak points: $H(.)$ is sensitive to $n$ and $\tau$, and $H(.)$ conflicts with the inference step $(C)$ of Eqn. (1) for a given set of parameters. To analyze these, we introduce the relationship between confidence and entropy $\mathcal{S}_C$s with the help of \eqref{eq:special_probs_in_simplex}. We state for a confidence $\tau$ that entropy achieves its maximum value at the point $v_n(\tau)$ and for an equi-entropy contour with value of $H(v_n(\tau))$ maximum achievable confidence is $\tau$. See the following proposition. 
\begin{figure}[t]
      \centering
      \subfigure[]{\includegraphics[width=.48\columnwidth]{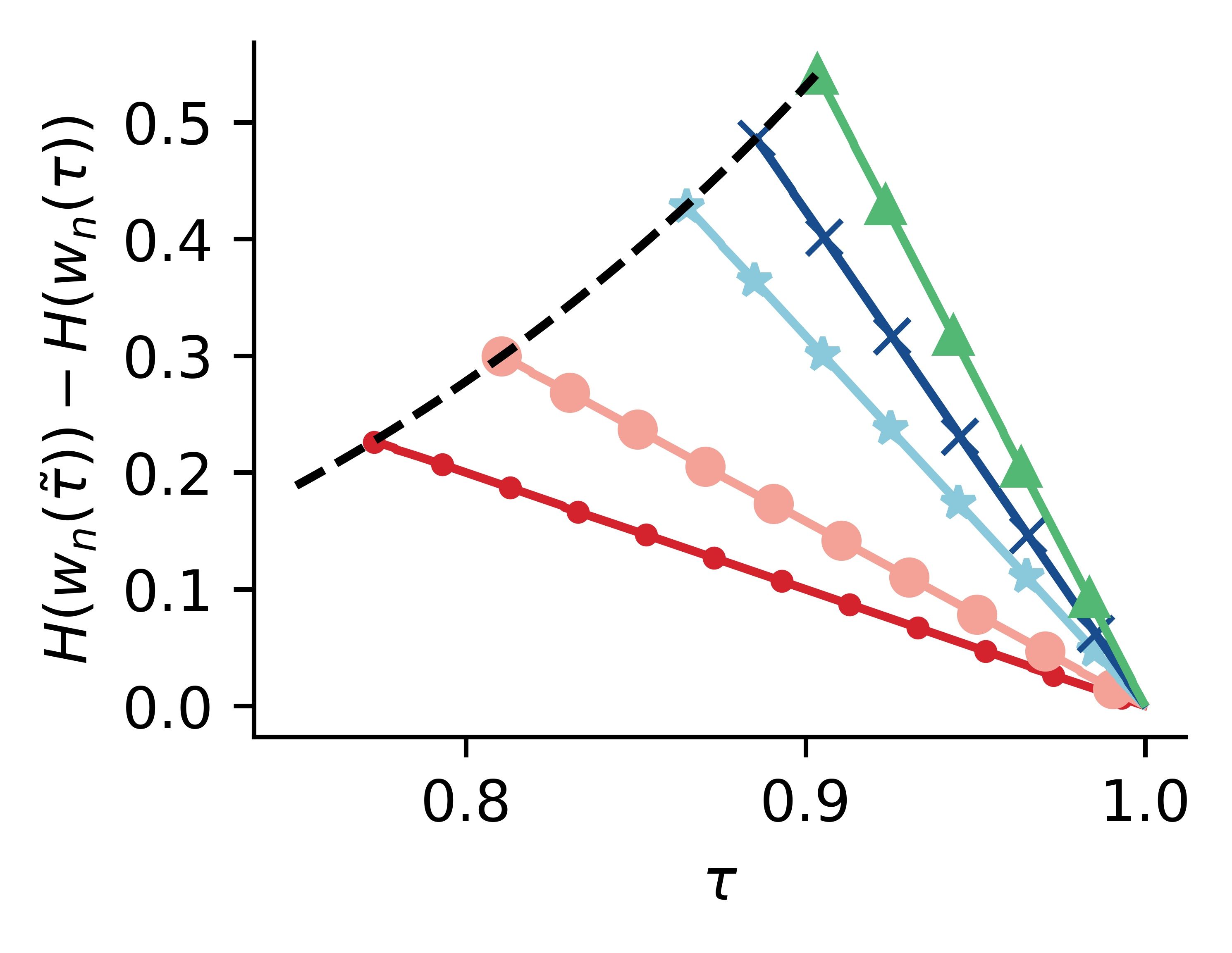}}
      \subfigure[]{\includegraphics[width=.48\columnwidth]{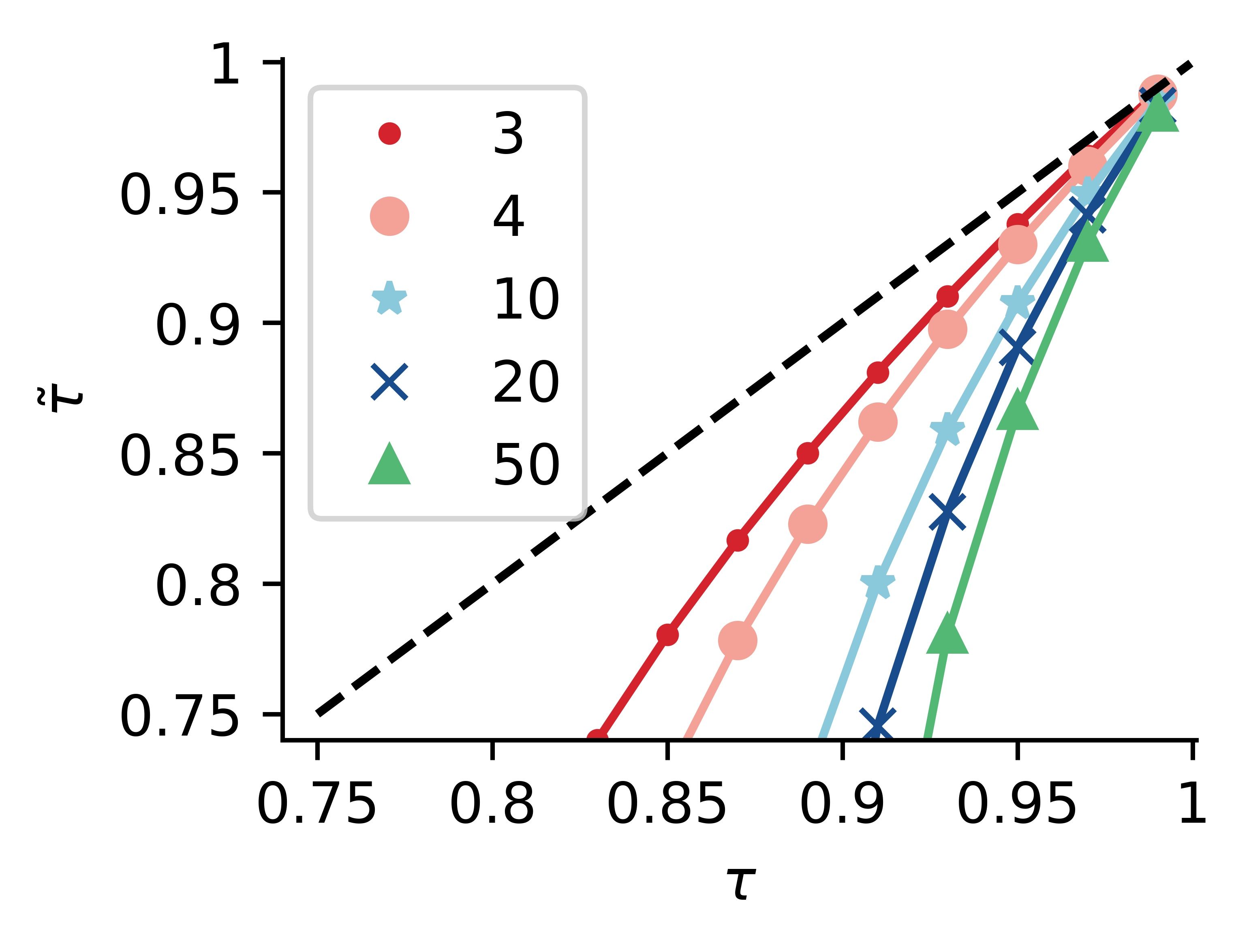}}
      \vspace{-0.4cm}
      \caption{Effect of number of possible categories (different color/line codes) on the decision boundary. (a) represents the difference between entropy values of the probabilities $v_n,w_n$ for a given confidence level $\tau$. (b) represents the $\tau$ and $\tilde{\tau}$ values where $H(w_n(\tilde{\tau})) = H(v_n(\tau))$.}
    \label{fig:effects_cardinality}
    \vspace{-0.3cm}
\end{figure}
\begin{prop}
    \label{prop:max_uncertainty}
    For a defined confidence level $\tau$;
    \begin{equation}
    \begin{split}
        &C_\tau = \lbrace p|\max_i p_i = \tau \rbrace \implies \max_{p\in C_\tau} H(p)= H(v_n(\tau))   \\
        &S_\tau = \lbrace p|H(p) = H(v_n(\tau)) \rbrace \implies \max_{ p \in S_\tau}\max_{i} p_i=\tau   \\
    \end{split}
    \end{equation}
\end{prop}
\begin{obs}
    \label{obs:intersection_entropy_conf}
    \emph{Following descriptions in Proposition \ref{prop:max_uncertainty}, confidence line for $\tau$, $C_{\tau}$,  intersects with the equi-entropy contours, $S_\tau$, only at $v_n(\tau)$ points.}
\end{obs}
Following this observation, we state that the set between the corners of the simplex $\Delta_n$ and $C_\tau$ is a subset of the set between the corners and $S_\tau$ as presented in the following observation:
\begin{obs}
    \label{obs:increasing_feasible_set}
    \emph{
    Define $S_1=\lbrace p \lvert \ \max_i p_i \geq \tau \rbrace$ and $S_2= \lbrace p \ \lvert \ H(p)\leq\tau' \rbrace$,
    $\tau'=H(v_n(\tau))$, $\forall \tau \in [1/n,1]$ then $S_1\subset S_2$. Therefore the $(S)$ region designed by entropy $S_1$ is larger than the region designed with the confidence threshold $S_2$.}
\end{obs}
It is apparent that the stopping region defined by entropy (M3) is larger than the stopping region defined through confidence level (M1) and hence enlarging the $\mathcal{R}_S$ in \eqref{eq:convRBE}. This increase in the region, on the other hand, decreases confidence. To analyze it we need to find the minimum confidence in an equi-entropy contour. It is shown that the minimum confidence is attained at $w_n(\tilde{\tau})$ that satisfies the following;
\begin{obs}\emph{\cite{ho2010interplay}}
\label{obs:weakest_entropy_guarantee}
    \begin{equation}
        \begin{split}
        &H(v_n(\tau))< 1, S_\tau = \lbrace p|H(p) = H(v_n(\tau))\rbrace\\ 
        &\hspace{0.2cm} \text{ then } \min_{ p \in S_\tau}\max_i p_i=\tilde{\tau}\\
        &\hspace{0.2cm}\text{ s.t. } -\tilde{\tau}\log_2(\tilde{\tau}) -(1-\tilde{\tau})\log_2(1-\tilde{\tau})= H(v_n(\tau))\\
        \end{split}
    \end{equation}
\end{obs}
This observation states that the entropy based $(S)$ boundary attains the minimum required max-probability value $\tilde{\tau}$ at $w_n(\tilde{\tau})$. $\tilde{\tau}$ vs $\tau$ and entropy values difference between $w_n(\tau)$ and $w_n(\tilde{\tau})$ are presented in Fig.~\ref{fig:effects_cardinality} for changing number of categories in the classification. Observe from Fig.~\ref{fig:effects_cardinality} that entropy is fragile with respect to the number of classes and the difference between $\tilde{\tau}$ and $\tau$  increases by decreasing values of $\tau$. Also observe that $H(w_n(\tilde{\tau}))- H(w_n(\tau))$ decreases linearly with $\tau$, and  increases exponentially (dotted-line) with $n$ as presented in Fig.~\ref{fig:effects_cardinality}-(a).

By nature, uncertainty based $(S)$ is capable of returning a class label when the confidence of the class is low as shown in Observation~\ref{obs:weakest_entropy_guarantee}. If there exists no $\tilde{\tau}$ that satisfies Observation~\ref{obs:weakest_entropy_guarantee}, equi-entropy contours do not intersect with the simplex boundary. This condition might result in immediate stopping if one of the classes is already unfavorable as shown in Fig.~\ref{fig:prob_motion}-(a) with the example trajectories close to $[a, c]$ edge. 

The above Proposition 2, and Observations 1-3 show that under certain conditions, the regions defined by uncertainty criteria may significantly diverge from the region formed by the confidence level threshold that was defined over the posterior distribution. To avoid such drawbacks, the system is usually set to a high confidence level (e.g. $\tau\approx .95$), such termination already mandates system continue with redundant recursions of evidence collection to achieve high confidence. Accordingly, despite their differences, uncertainty based stopping (M3) and constant confidence threshold  (M1) yield similar performances in recursive classification tasks. To analyze these similarities in the classification performances, in the next section, we investigate the behavior of the posterior motion over the probability simplex. Through such an analysis, we then gain an insight into designing stopping criterion that will avoid high confidence levels and corresponding redundant recursions/evidence collections. 

\subsection{Posterior Motion}
\label{subsec:posterior_motion}
Given a prior point $p(\sigma)$, the posterior evolution with given pairs of queries and evidences $(\Phi,\beps)$, 
\begin{equation*}
\begin{split}
    p(\sigma\lvert \Phi_{0:s},\beps_{0:s}) = p(\sigma) \oplus p(\varepsilon_1 \lvert \sigma,\phi_1)\oplus \cdots \oplus p(\varepsilon_s \lvert \sigma,\phi_s) \\
    %= p(\sigma)\oplus \sum_{\oplus,i} p(\varepsilon_i\lvert\sigma,\phi_i)
    \end{split}
\end{equation*}
{\color{rev_color}Posterior trajectory is the list of vectors that is obtained within RBC over sequences of evidence collections: $[p(\sigma),p(\sigma\lvert  \Phi_{0},\beps_{0}),\cdots, p(\sigma\lvert\Phi_{0:s},\beps_{0:s})]$. Posterior motion refers to the changes in RBC over sequences that resulted in the posterior trajectory.}
{\color{rev_color}In this section we aim to show that the trajectory of the posterior distribution follows a central path (Proposition~\ref{prop:getting_closer_mid}) as it gets closer to the corner of interest as illustrated in Fig.~\ref{fig:prob_motion}. This will further allow us to clarify the similar early stopping behavior of (M1) and (M3) geometrically. By understanding motion we will design a stopping condition that promotes early stopping while preventing accuracy loss for a motion on edges.} 

{\color{rev_color}The motion is determined by the evidence likelihood $p(\varepsilon_s | \sigma, \phi_s)$ at each step $s$ and the evidence collected through query $\phi_s$. In this section we assume two noisy information channels for true class and incorrect class respectively. Once the true class is queried, evidence is sampled from the "positive" distribution and evidence is sampled from "negative" distribution otherwise and we assume the posterior probability of the true state/class increases on average \cite{tsiligkaridis2014collaborative}. }

\begin{figure}[t]
      \centering
      \subfigure[]{\includegraphics[width=.48\columnwidth]{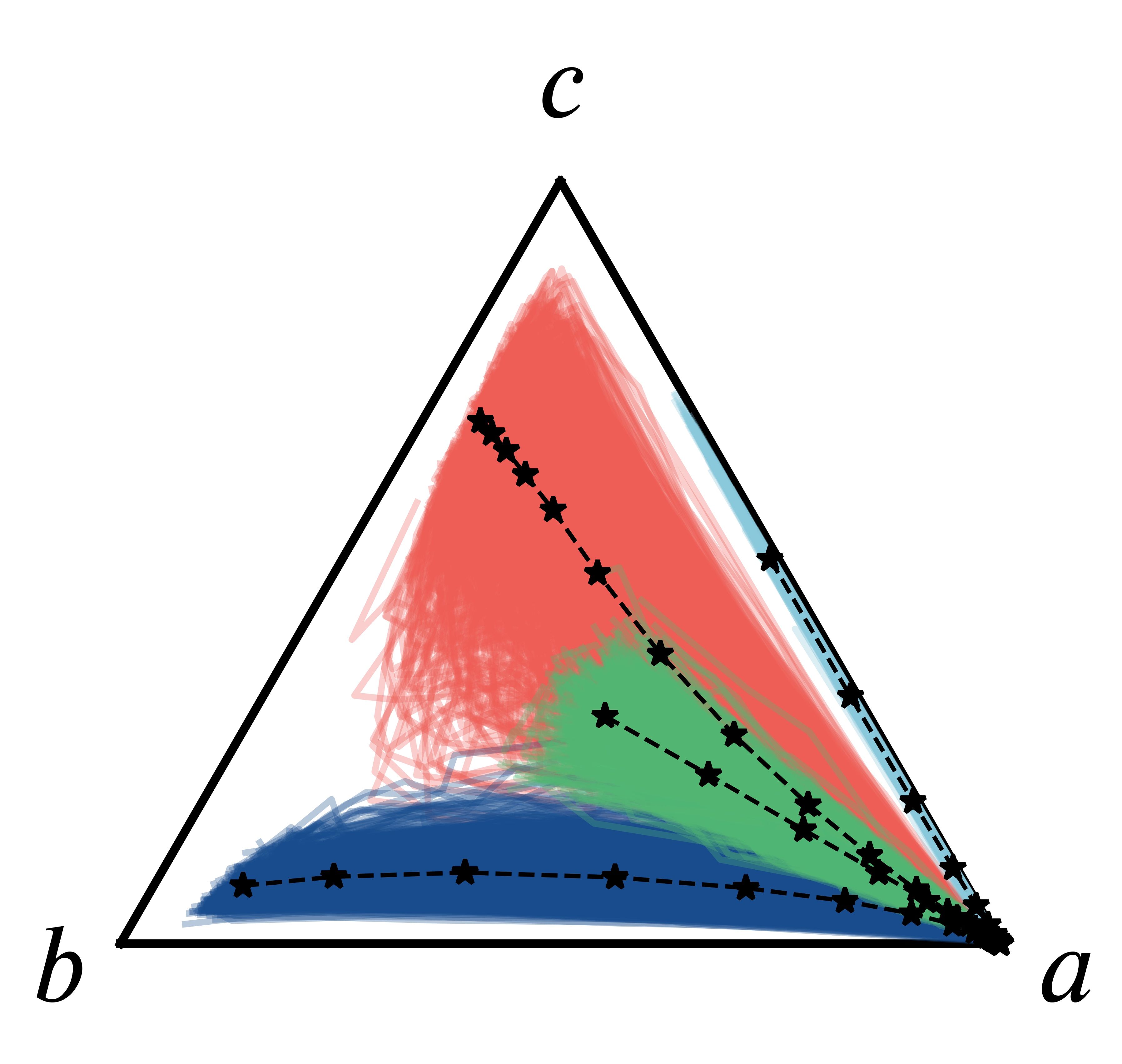}}
       \subfigure[]{\includegraphics[width=.48\columnwidth]{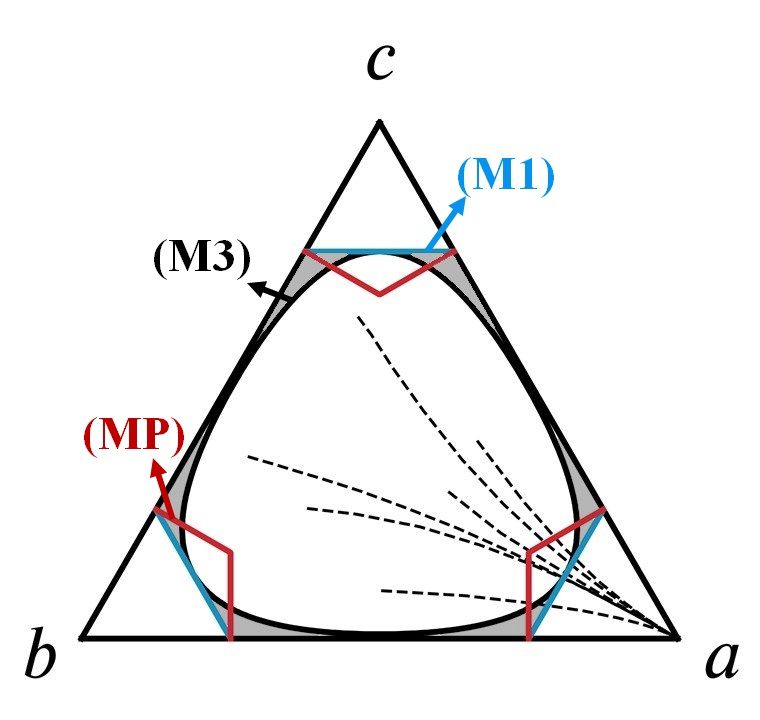}}
       \vspace{-0.4cm}
      \caption{Monte Carlo simulated trajectories for probability evolution in 3D simplex. To simulate the trajectories we sample evidence from lognormal distributions. (a) Each color represent 100 simulated examples from a different starting point and dashed lines representing the means of the trajectories. Observe the case where one of the classes is already disfavored (one between a, c), if decision boundary does not intersect with the edge, that results in immediate termination. (b) Three different perspectives for $\mathcal{S}_R$ design. We observe that the behavior of entropy (black) and confidence threshold (blue) are similar if trajectories are considered. In the method section we propose a region design using an equi-distance curves wrt. to the corners (red).}
    \label{fig:prob_motion}
    \vspace{-0.4cm}
\end{figure}

\begin{lemma}
\label{lemma:collinearity_of_post}
    Let $a\in\mathcal{A}$ and let $\phi(a)$ denote the query related with state $a$ then
    $p(\sigma)= [p_1=p(a),p_2,\cdots,p_n]$. $p(\sigma)$, $[1,0,\cdots,0]$ and $p(\sigma)\oplus p(\varepsilon|\sigma,\phi(a)) \ \forall \varepsilon,\forall p(\sigma)$ are collinear. 
\end{lemma}
Hence, once the system queries the environment, posterior probability for classification takes a step on the line that passes through the current position and the corner addressed by query. {\color{rev_color}Moreover}, on average, if the query addresses true state, the posterior moves towards the respective corner, if the query addresses an incorrect estimate the posterior moves away from the respective corner. To show that in general posterior gets closer to a central position, we need the projection of a posterior point to the line that passes through the center ($u_n$) and one of the corners.
\begin{lemma}
\label{lemma:projection_to_center}
    Given $p(\sigma)=p=[p_1,p_2,\cdots p_n]\in\Delta_n$ and given the line $\ell c_{n,i=1} = \lbrace [\tau,\frac{1-\tau}{n-1},\cdots ,\frac{1-\tau}{n-1}] | \forall \tau\in[0,1] \rbrace$ then $\ell_2$ norm projection of point $p(\sigma)$ onto line $\ell c_{n,i=1}$ is the following:
    \begin{equation}
    \begin{split}
        \text{\emph{proj}}_{\ell c_{n,i=1}}(p(\sigma))= \arg\min_{p\in\ell c_{n,i=1}} \| p(\sigma) - p \|_2 \\
        = \left[p_1,\frac{1-p_1}{n-1},\cdots ,\frac{1-p_1}{n-1}\right]
    \end{split}
    \end{equation}
\end{lemma}
To give an example, line $\ell c _{n=3, i=1}$ is visualized in Fig.~\ref{fig:simplex_repr}-(d) where $1^\text{st}$ location is $a$ . With this projection operation we show that $\ell_2$-norm distance between the projection and the actual point $d_s$ at sequence $s$ decreases quadratically with respect  to the posterior as $s$ increases. This statement is given in the following proposition:
\begin{prop}
\label{prop:getting_closer_mid}
Following Lemma \ref{lemma:projection_to_center}, given $p(\sigma)\in\Delta_n$;
\begin{equation*}
\begin{split}
    \| p(\sigma)  - \text{\emph{proj}}_{\ell c_{n,i=1}}p(\sigma)\|_2^2\propto (1-p_1)^2 \\
\end{split}
\end{equation*}
Following this, one defines the reduction between two sequences,
\begin{equation*}
\begin{split}
\exists \hat{s} \text{ \emph{s.t.} } \|p(\sigma\lvert\mathcal{H}_s) -  \text{\emph{proj}}_{\ell c_{n,i=1}}p(\sigma\lvert\mathcal{H}_s)\|_2 =d_s \\
\text{\emph{where }} \tilde{p}(\sigma\lvert\mathcal{H}_s+1) = p(\sigma\lvert\mathcal{H}_s) \oplus E_\varepsilon(p(\varepsilon\lvert \sigma, \phi)) \  \forall \phi \\
\|\tilde{p}(\sigma\lvert\mathcal{H}_{s+1})- \text{\emph{proj}}_{\ell_{c,n,i=1}}\tilde{p}(\sigma\lvert\mathcal{H}_s+1)\|_2=d_{s+1}\\
\text{Using Lemma \ref{lemma:collinearity_of_post}} \implies d_s > d_{s+1}\ \forall s\geq\hat{s}
\end{split}
\end{equation*}
\end{prop}
In recursive classification, on average, probability of the true state/class increases sequentially and following the proposition, the posterior probability gets closer to the line that passes through the respective corner and the center. We visualize examples of average trajectories in Fig.~\ref{fig:prob_motion}-(a).
{\color{rev_color} Additionally we presented in Fig.~\ref{fig:prob_motion}-(b) that unlike (M1) $\mathcal{S}_C$ curves (blue), (M3)-curves (black) allow additional $\mathcal{S_R}$ (shaded areas) that is concentrated closer to edges and they promote early stopping. Considering both Fig.~\ref{fig:prob_motion}-(a,b) together, if posterior follows a central path (M3) behaves as (M1) and (M3) cannot promote early stopping. Next, we propose a new stopping criterion perspective to extend the stopping region centrally towards the uniform distribution (e.g. red boundary in Fig.~\ref{fig:prob_motion}-(b)). This approach will address our observations in Sec.~\ref{subsec:frail_confidence} and Sec.~\ref{subsec:posterior_motion} that the new approach will be robust to number of categories while promoting early stopping with marginal accuracy loss.}

\section{Proposed Perspective}
% Method Section
\label{sec:Method}
%In this section we propose another insight for $\mathcal{R}_S$ design. 
In the previous sections we argued equi-entropy contours formed for (M3) are centered around $u_n$. Moreover (M3) is sensitive to number of categories and stagger in cases where some of the classes are already unfavorable. 
{\color{rev_color}Additionally we propose to bend (M1) from the center towards $u_n$ e.g. the red boundary in Fig.~\ref{fig:prob_motion}-(b).} 
Trivially, it is possible to form this by equi-distant points to the respective corners. However, by definition, edges and corners are $\notin\Delta_n$ (represents $\infty$ \cite{aitchison1982statistical}) which prevents measuring the distance with conventional information theoretic approaches.
% %
To avoid this, exploiting $\Delta_n\subset \mathbb{R}^n$, one can use a distance measure $\delta$ defined over $\mathbb{R}^n$ (e.g. $\ell_p$ norms) and intersect the $\bar{\tau}$ ball around $p$ $B^{\bar{\tau}}_\delta(p)= \lbrace x | \delta(p,x)<\bar{\tau}$ with $\Delta_n$ to obtain $\tilde{B} = B\cap \Delta_n$ centered around a corner and $\in\Delta_n$ $\tilde{B}^{\bar{\tau}}_\delta(p)= \lbrace x | \delta(p,x)<\bar{\tau}, x\in\Delta_n \rbrace$. Let $c^k\in\Delta_n $ be the $k^\text{th}$ corner (e.g. $c^1 = [1,0,\cdots,0]$), decision region with $\delta$ and $\bar{\tau}$ is defined as;
\begin{equation}
    \label{eq:conv_decision_region}
    \mathcal{S}_R: \ p\in \bigcup_{k\in\lbrace 1,2,\cdots,n\rbrace} \tilde{B}^{\bar{\tau}}_\delta(c^k)
\end{equation}
We use a distance measure influenced by Kittler's work \cite{kittler2018delta}: the \textit{delta} divergence.
We use the definition of \textit{delta}- divergence and modify it to define our novel decision region.
This divergence determines an interest set (e.g. indices of the most probable elements in respective distributions) and compare the probability mass in these elements. We limit the set to two elements in our paper. The corresponding distance between $p,q\in\Delta_n$ is denoted as $\delta_{\textit{MP}}$ (MP indicating 'method proposed') defined as the following;
\begin{equation}
\label{eq:mod_delta_divergence}
    \begin{split}
    \delta_{\textit{MP}}(p,q) = \sum_{i\in I} | p_i - q_i |, \  I = \lbrace j_1,j_2,k_1,k_2 \rbrace \\
    j_1 = \arg\max_i p_i, j_2 = \arg\max_{i\neq j_1} p_i\\
    k_1 = \arg\max_i q_i, k_2 = \arg\max_{i\neq k_1} q_i
    \end{split}
\end{equation}
{Observe the measure satisfies non-negativity, identity of indiscernibles and symmetry.} Using this distance to obtain balls, we can state the following proposition about the stopping region and criterion, $\mathcal{S}_R$ and $\mathcal{S}_C$ respectively;
\begin{prop}[Proposed $\mathcal{S}_R$ (MP)]
\label{prop:delta_made_good}
$c^k$ being ${k}^\text{th}$ corner,$p\in\Delta_n, {\delta} = \delta_{\textit{MP}}, \bar{\tau}\in [1/n,1]$ then;
\begin{equation}
\begin{split}
    &\mathcal{S}_R:  p\in \bigcup_{k\in\lbrace 1,2,\cdots,n\rbrace} \tilde{B}^{\bar{\tau}}_\delta(c^k) \\
    &\equiv \mathcal{S}_C: p_{j_1} - p_{j_2} > 1 -\bar{\tau}\\
    \text{where }  &j_1 = \arg\max_i p_i \ \text{and } j_2  = \arg\max_{i\neq j_1} p_i
\end{split}
\end{equation}
\end{prop}
We denote this method by (MP). We visualize an example decision boundary in Fig.~\ref{fig:prob_motion}-(b). As can be observed from the Fig.~\ref{fig:prob_motion}-(b), respective decision boundary is inline with the motion of the probability distribution. To have a decision boundary bending the boundary at $\tau$ for (M1) from its center the following condition is required;
\begin{obs}
\label{obs:taubar_reqs}
 \emph{Given $\bar{\tau} = 2 - 2\tau$ and WLOG for $p\in\Delta_n , \ \arg\max_i p_i = 1$. Define (MP) and (M1) decision boundaries;
 \begin{equation}
    \begin{split}
       & C_\tau = \lbrace p | p_1 = \tau
        \rbrace\\
        &B_{\bar{\tau}} = \lbrace p| p_1 - p_m=1-\bar{\tau}, m=\arg\max_{i\neq 1}p_i\rbrace
    \end{split}
\end{equation}
$\implies$  $C_\tau\cap B_{\bar{\tau}}$= $w_n(\tau)$ and $\max_{p\in B_{\bar{\tau}}}\max_i p_i = \tau$}
 \end{obs}
\begin{obs}
\label{obs:taubar_reqs2}
\emph{
\begin{equation}
\begin{split}
    &B_{\bar{\tau}} = \lbrace p| p_1 - p_m=1-\bar{\tau}, m=\arg\max_{i\neq 1}p_i\rbrace\\
    &\min_{p\in B_{\bar{\tau}}}\max_i p_i = \frac{1+(n-1)(1-\bar{\tau})}{n} = \psi, p =v_n(\psi)  
    \end{split}
\end{equation}}
\end{obs}
The closest point of the decision boundary to $u_n$ is $v_n(1+\bar{\tau}({(1-n)}/{n})$ and hence on the line $\ell c_{n,i}$ for each respective corner $i$. This implies, unlike uncertainty methods, proposed boundary does not interfere with the inference, $(c)$ function. By definition \eqref{eq:mod_delta_divergence} is robust to number of categories and the cases where one of the classes is already unfavorable. We omit derivations and refer the reader to \cite{kittler2018delta}.
\begin{prop}
\label{prop:perf_guarantees}
Given $\tau$, $p = [p_1,p_2,\cdots,p_n]\in\Delta_n$ s.t. $2 =\arg\max_{i\neq 1}p_i$ and $p_2 \geq (1-p_1)\tau$ and evidence $\varepsilon = [\varepsilon_+,1,\cdots,1]$ where $\varepsilon_+ \sim \text{\emph{lognorm}}(\mu,c^2)$, we define the posterior at sequence $s$ $p^s = p^0 \oplus \varepsilon \oplus \cdots \oplus \varepsilon$. With $\bar{\tau} = 2- 2\tau$ and $\tilde{\tau}=((2\tau-1)(n-1)+1)/n$ define;
\begin{equation*}
\begin{split}
    &S_{r\text{\emph{(M1)}}} = \left\lbrace p^s | p_1^s \geq \tau \right\rbrace\\
     &{s}'_{R\text{\emph{(M1)}}} = \left\lbrace p | p_2 \geq \tau \right\rbrace\\
    &S_{R\text{\emph{(MP)}}} = \lbrace p^s | p_1^s-p^s_2 \geq 1-\bar{\tau}\rbrace\\
    &S'_{R\text{\emph{(MP)}}} = \lbrace p^s | p^s_2-p^s_j \geq 1-\bar{\tau}, j = \arg\max_{k\notin \lbrace 1 ,2\rbrace} p^s_k \rbrace
    \implies\\
    &p(p^s\in S_{R(\text{\emph{MP}})}(p_1,p_2,\bar{\tau},s)) \geq 
      p(p^s\in S_{R\text{\emph{(M1)}}}(p_1,\tau,s)) \  \text{\emph{\&}} \\
    &p(p^s\in S'_{R(\text{\emph{M1}})}(p_1,p_2,\tilde{\tau},s)) \geq 
      p(p^s\in S'_{R\text{\emph{(MP)}}}(p_1,p_2,\bar{\tau},s))\\
      &\hspace{0.5cm} \geq p(p^s\in S'_{R(\text{\emph{M1}})}(p_1,p_2,{\tau},s))
      \ \forall s\in \lbrace 1,2,\cdots \rbrace
\end{split}
\end{equation*}

\end{prop}
{Proposition 5 demonstrates that compared to a stopping criterion based on posterior distribution thresholding (M1), the proposed method (MP) always has higher probability of entering the stopping region with correct decision, while the probability of entering an incorrect region resulting in incorrect decision is constrained within certain probability values. 
Therefore, a system can be designed to achieve a desired true positive probability while limiting false alarm probability by using the proposed stopping criterion. 
Supporting numerical examples are in the appendix  Sec.\ref{sec:tp_fa_guarantee}.}

\section{Experiments and Results}
% Experiments section of the paper
\label{sec:Experiments}

In this section we run experiments to support our findings. Throughout this section we denote the \textit{proposed method} introduced in Sec.~\ref{sec:Method} by (MP) and compare it with the methods presented in Table \ref{tab:competitor_methods_table}. Specifically, we designate a confidence level $\tau$ for (M1). We select respective $c$s for (M2:4) such that (M1:4) intersect at $v_n(\tau)$ {\color{rev_color}(See Observation~\ref{obs:intersection_entropy_conf})}. We select $\bar{\tau}$ for (MP) following Observation~\ref{obs:taubar_reqs}. For (M5) $c_{\delta_{\text{KL}}}$  is selected as $10^{-2}$ {\color{rev_color}(see Table~\ref{tab:competitor_methods_table} for competing method definitions)}. We demonstrate results for the two following cases: (i) Synthetic experiments (ii) A letter decision for electroencephalography (EEG)-based brain computer interface (BCI) typing system.

\subsection{Synthetic Experiments}
{
{\color{rev_color}In this section we aim to demonstrate the following items: 

\begin{enumerate}
    \itemsep-0.5em 
    \item In the presence of disfavored classes, where the recursive classification is happening in a lower dimension than the cardinality of the class space (e.g. at least one of the class probabilities $\approx 0$), uncertainty based methods (especially Renyi Entropy with orders $\alpha\leq 1$) suffer from immediate or rushed stopping (Sec.~\ref{subsec:frail_confidence}) whereas (MP) is robust to this effect and present similar performance as confidence based methods. (Details in Table~\ref{tab:disfavored_stuff}, Table~\ref{tab:disfav_2})
    \item In case of a central posterior trajectory, confidence and uncertainty based methods perform similarly whereas (MP) stops earlier with marginal accuracy loss (Sec.~\ref{subsec:posterior_motion}). It is possible to early stop by lowering the confidence level. However, doing so is penalized in accuracy for cases in 1. (MP) provides an early stopping as if confidence level is relaxed for 2 and provides high accuracy performance for 1. (Details in Table~\ref{tab:disfav_3})
\end{enumerate}
}

{\color{rev_color}In our experiments we present scenarios with pre-defined prior information; for Tables \ref{tab:disfavored_stuff},\ref{tab:disfav_2} the priors are set such that some of the classes have $\sim0$ probability mass therefore this prior is close to the simplex edge and the posterior is expected to move in a lower dimensional space (case 1). For Table \ref{tab:disfav_3} the prior is set such that the target probability is set to $.1$ where the rest of the class options share the remaining mass equally. This puts the prior close to the center of the opposing edge of the respective corner, therefore the posterior will follow a central path (case 2). In our experiments evidence (likelihood at each sequence) employs the vector form: $\varepsilon = [\varepsilon_+, \varepsilon_-, \cdots, \varepsilon_-]$ where  $1^\text{st}$ candidate is the true class WLOG. In each sequence,  In each sequence, following the assumptions of Fig.~\ref{fig:graphicalModel}, $\varepsilon_+$ and all $\varepsilon_-$s each are sampled from their respective lognormal distributions independently (specifications are listed under each Table \ref{tab:disfavored_stuff}, \ref{tab:disfav_2}, \ref{tab:disfav_3}). The likelihoods are further merged with the latest posterior with $\oplus$ \eqref{eq:simplex_add}.}
For our experiments, in addition to the methods presented in {\color{rev_color}Table~\ref{tab:competitor_methods_table}}, we also use a confidence lower bound with a confidence level $\tilde{\tau}$ that is derived in Observation \ref{obs:taubar_reqs2} (e.g., $\tilde{\tau}$ for which the minimum accuracy can be achieved as (MP)) and we represent this method with $\bar{\text{(M1)}}$. To be precise, for (M1) and (MP) presented in Fig.~\ref{fig:prob_motion}-(b), $\bar{\text{(M1)}}$ will be the confidence lines each intersect with (MP) at the respective peak points. By design, we expect to show that (MP) is as robust as (M1) and (MP) reaches the speed of $\bar{\text{(M1)}}$ with marginal accuracy decrease.  In these tables, we report two measures for each method given number of sequences such that $p_\text{stop}$ represents the probability of stopping, $p_{\text{true}|\text{stop}}$ represents the accuracy among all terminations (e.g. $p_{\text{true}|\text{stop}}=0.5, p_\text{stop}=0.6$ means the system correctly selected 300 out of 600 terminated in 1000 Monte Carlo (MC) simulations). 
{\color{rev_color}In our results, we highlight $p_\text{stop}\geq .5$ indicating a reasonable termination operational point of (MP). E.g. in Table \ref{tab:disfavored_stuff} (MP) at $5^\text{th}$ sequence terminated $59\%$ of all simulated classification tasks with $97\%$ accuracy. Observe that high probability of termination $p_\text{stop}$ is desired only if the accuracy $p_{\text{true}|\text{stop}}$ within the terminated recursions is high enough.}
%We highlight the first sequence that achieves $p_\text{stop}\geq .5$ indicating a reasonable stopping operational point of a method.
{
\begin{table}[!t]
\begin{center}
\begin{minipage}{0.16\columnwidth}
		\centering
		\includegraphics[width=17mm]{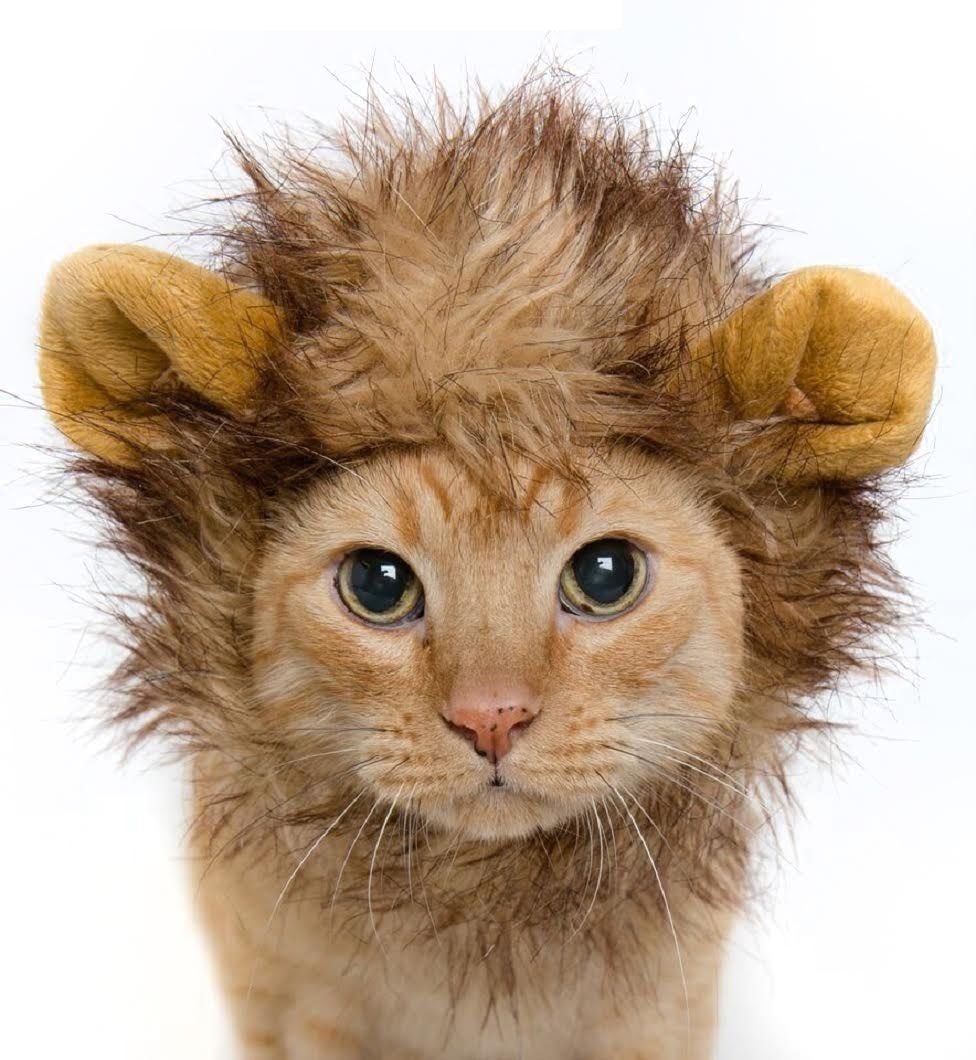}\\
		\tiny{\textsf{$p(\text{cat}) = 0.42$ \\ $p(\text{lion})= 0.55$}\\
		s.t. $|\mathcal{A}|= 3$ }
		%\\ $\mathcal{A} = \lbrace cat,lion,dog \rbrace$\\
		%{\tiny(imagesrc:https://www.amazon.com/Pet-Krewe-PK00101-Costume-Small/dp/B010E4TAKW)}}
\end{minipage}
\begin{minipage}{0.83\linewidth}
\footnotesize
\centering
\scalebox{0.77}{
\begin{tabular}{c c | c c c c a c c }
% ['MP', 'M1', 'M2', 'M3', 'M4', 'M5', 'M1L']
%\multicolumn{5}{c}{Stopping Region: $\mathcal{R}_S:= \lbrace p| \mathcal{S}_O(p) \rbrace$} \\
& &\multicolumn{7}{c}{Number of Sequences} \\
Method & &1 &2 &3 &4 &5 &6 &7 \\
\hline
\multirow{2}{*}{MP} 
&$p_\text{stop}$  & 0.00 & 0.06 & 0.22 & 0.43 & \textbf{0.59} & 0.67 & 0.77\\
&$p_{\text{true}|\text{stop}}$  & 0.00 & 0.67 & 0.86 & 0.96 & \textbf{0.97} & 0.98 & 0.99\\
% &$p_\text{false}$  & 0.00 & 0.02 & 0.02 & 0.03 & 0.02 & 0.01 & 0.01\\
\hline
\multirow{2}{*}{M1} 
&$p_\text{stop}$  & 0.00 & 0.06 & 0.22 & 0.43 & \textbf{0.59} & 0.67 & 0.77\\
&$p_{\text{true}|\text{stop}}$  & 0.00 & 0.67 & 0.86 & 0.96 & \textbf{0.97} & 0.98 & 0.99\\
% &$p_\text{false}$  & 0.00 & 0.02 & 0.02 & 0.03 & 0.02 & 0.01 & 0.01\\
\hline
\multirow{2}{*}{M2} 
&$p_\text{stop}$ & 0.00 & 0.08 & 0.27 & 0.47 & \textbf{0.62} & 0.70 & 0.79\\
&$p_{\text{true}|\text{stop}}$  & 0.00 & 0.67 & 0.85 & {0.95} & \textbf{0.95} & 0.97 & 0.98\\
% &$p_\text{false}$  & 0.00 & 0.03 & 0.02 & 0.03 & 0.02 & 0.01 & 0.01\\
\hline
\multirow{2}{*}{M3} 
&$p_\text{stop}$  & 0.00 & {0.34} & \textbf{0.56} & 0.70 & 0.80 & 0.85 & 0.90\\
&$p_{\text{true}|\text{stop}}$ & 0.00 & {0.20} & \textbf{0.48} & 0.63 & 0.74 & 0.80 & 0.87\\
% &$p_\text{false}$  & 0.00 & 0.14 & 0.08 & 0.09 & 0.05 & 0.04 & 0.02\\
\hline
\multirow{2}{*}{M4} 
&$p_\text{stop}$  & \textbf{1.00} & {1.00} & 1.00 & 1.00 & 1.00 & 1.00 & 1.00\\
&$p_{\text{true}|\text{stop}}$  & \textbf{0.00} & {0.55} & 0.72 & 0.82 & 0.87 & 0.90 & 0.94\\
% &$p_\text{false}$  & 1.00 & 0.45 & 0.24 & 0.20 & 0.13 & 0.10 & 0.07\\
\hline
\multirow{2}{*}{M5} 
&$p_\text{stop}$  & 0.00 & 0.00 & 0.37 & 0.41 & 0.46 & \textbf{0.55} & 0.66\\
&$p_{\text{true}|\text{stop}}$  & 0.00 & 0.00 & 0.58 & 0.76 & 0.87 & \textbf{0.91} & 0.95\\
% &$p_\text{false}$  & 0.00 & 0.00 & 0.14 & 0.11 & 0.07 & 0.05 & 0.03\\
\hline
\multirow{2}{*}{$\bar{\text{M1}}$}
&$p_\text{stop}$  & 0.00 & 0.16 & 0.37 & \textbf{0.56} & 0.71 & 0.77 & 0.84\\
&$p_{\text{true}|\text{stop}}$  & 0.00 & 0.64 & 0.84 & \textbf{0.93} & 0.95 & 0.97 & 0.98\\
% &$p_\text{false}$  & 0.00 & 0.06 & 0.04 & 0.05 & 0.03 & 0.03 & 0.02\\
\end{tabular}
}
\end{minipage}
\end{center}
\caption{A toy example with 3 class recursive classification. The set contains "cat, lion, dog" whereas the prior information for the given image (left) is $p = [0.42, 0.55, 0.03]$ with $\tau = .8$. %Observe this case directly corresponds to disfavoring a particular class. 
{\color{rev_color}Observer that the prior is located on the edge (See Sec.~\ref{subsec:frail_confidence}).}
We proceed in the recursive classification with evidence $\varepsilon = [\varepsilon_+, \varepsilon_-, \varepsilon_-]$ where $\varepsilon_+\sim \text{lognorm}(0.6,0.5^2)$ and $\varepsilon_- \sim \text{lognorm}(0,0.5^2)$. We perform 5000 MC simulations to report each number. On the table (right), we present probability of stopping with the given criteria and accuracy among times the system stopped. As expected entropy (M3) staggers by stopping $34\%$ and with only $20\%$ accuracy within second sequence. More importantly, as we reasoned before, Renyi entropies with order $\alpha <1$ (M4), $\alpha = 0.2$ in this case, create ambiguous decision regions which leads to stopping all the time as presented.{\tiny(imagesrc:https://www.amazon.com/Pet-Krewe-PK00101-Costume-Small/dp/B010E4TAKW)}
\vspace{-7mm}}
\label{tab:disfavored_stuff}
\end{table}}
{
\begin{table}[!t]
\begin{center}
\begin{minipage}{0.16\columnwidth}
		\centering
		\includegraphics[width=18mm]{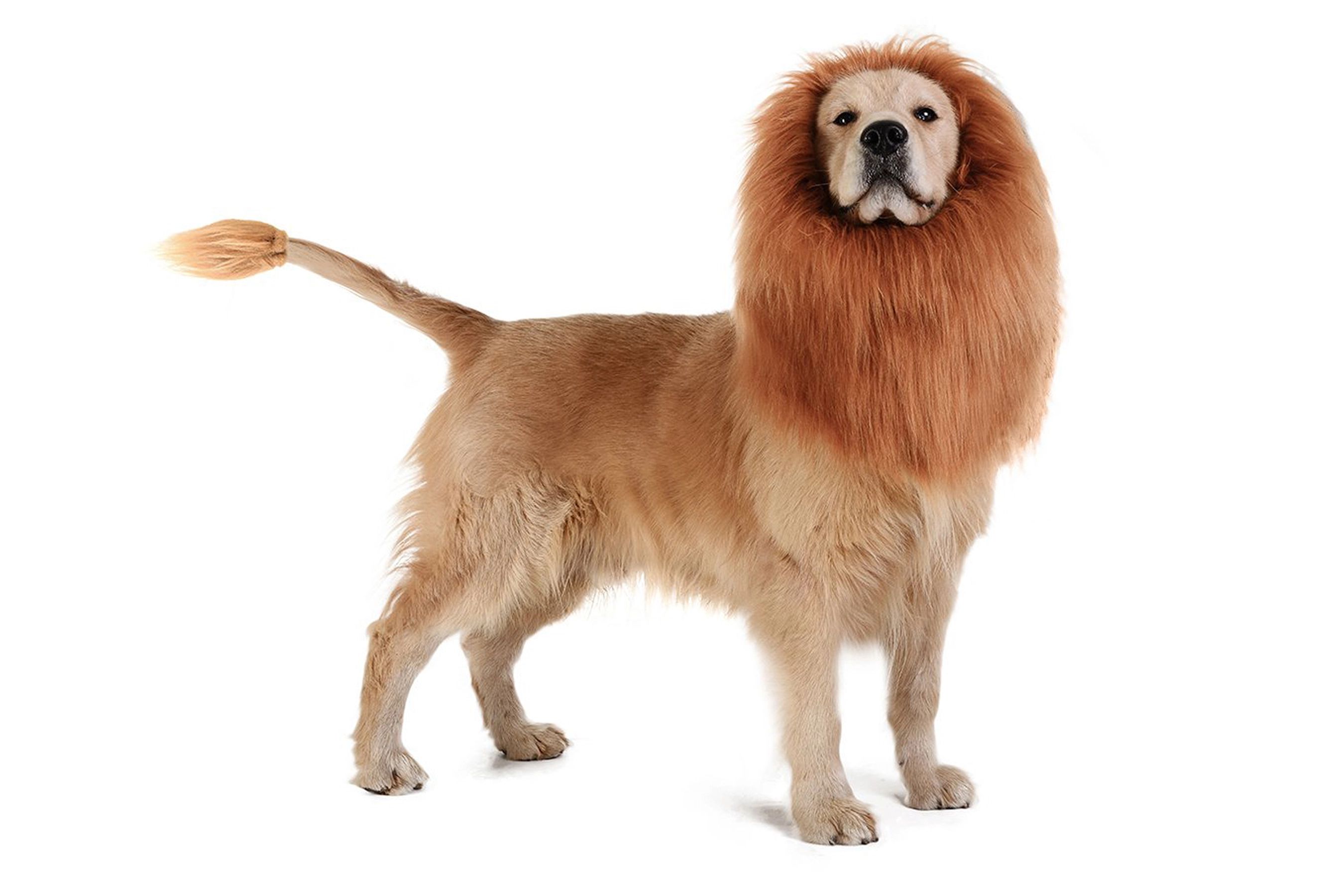}\\
		\tiny{\textsf{$p(\text{cat}) = 0.13$ \\ $p(\text{lion})= 0.52$ \\ $p(\text{dog})= 0.30$
		%\\ $\mathcal{A} = \lbrace cat,lion,dog, \cdots \rbrace$
		} \\ 
		s.t. $|\mathcal{A}|= 10$ }
		%\\ {\tiny(imagesrc:https://people.com/pets/tomsenn-lion-mane-dog-costume-amazon/)}
\end{minipage}
\begin{minipage}{0.83\linewidth}
\footnotesize
\scalebox{0.66}{
\begin{tabular}{c c | c c c c a c c c c  }
% ['MP', 'M1', 'M2', 'M3', 'M4', 'M5', 'M1L']
%\multicolumn{5}{c}{Stopping Region: $\mathcal{R}_S:= \lbrace p| \mathcal{S}_O(p) \rbrace$} \\
& &\multicolumn{9}{c}{Number of Sequences} \\
Method & &1 &2 &3 &4 &5 &6 &7 &8 &9\\
\hline
\multirow{2}{*}{MP} 
&$p_\text{stop}$  & 0.00 & 0.08 & 0.26 & 0.42 & \textbf{0.61} & 0.70 & 0.78 & 0.82 & 0.88 \\
&$p_{\text{true}|\text{stop}}$ & 0.00 & 0.36 & 0.78 & 0.89 & \textbf{0.94} & 0.96 & 0.97 & 0.98 & 0.98 \\
% &$p_\text{false}$  & 0.00 & 0.02 & 0.02 & 0.03 & 0.02 & 0.01 & 0.01\\
\hline
\multirow{2}{*}{M1} 
&$p_\text{stop}$ & 0.00 & 0.03 & 0.18 & 0.35 & \textbf{0.54} & 0.66 & 0.76 & 0.81 & 0.86 \\
&$p_{\text{true}|\text{stop}}$ & 0.00 & 0.29 & 0.83 & 0.90 & \textbf{0.94} & 0.97 & 0.97 & 0.98 & 0.98 \\
% &$p_\text{false}$  & 0.00 & 0.02 & 0.02 & 0.03 & 0.02 & 0.01 & 0.01\\
\hline
\multirow{2}{*}{M2} 
&$p_\text{stop}$ & 0.00 & 0.05 & 0.23 & 0.41 & \textbf{0.61} & 0.72 & 0.81 & 0.85 & 0.88 \\
&$p_{\text{true}|\text{stop}}$  & 0.00 & 0.32 & 0.80 & 0.89 &\textbf{0.94} & 0.96 & 0.97 & 0.98 & 0.98 \\
% &$p_\text{false}$  & 0.00 & 0.03 & 0.02 & 0.03 & 0.02 & 0.01 & 0.01\\
\hline
\multirow{2}{*}{M3} 
&$p_\text{stop}$  & \textbf{1.00} & 1.00 & 1.00 & 1.00 & 1.00 & 1.00 & 1.00 & 1.00 & 1.00 \\
&$p_{\text{true}|\text{stop}}$ & \textbf{0.00} & 0.43 & 0.63 & 0.74 & 0.84 & 0.87 & 0.91 & 0.93 & 0.94 \\
% &$p_\text{false}$  & 0.00 & 0.14 & 0.08 & 0.09 & 0.05 & 0.04 & 0.02\\
\hline
\multirow{2}{*}{M4} 
&$p_\text{stop}$  & \textbf{1.00} & 1.00 & 1.00 & 1.00 & 1.00 & 1.00 & 1.00 & 1.00 & 1.00\\
&$p_{\text{true}|\text{stop}}$  & \textbf{0.00} & 0.43 & 0.63 & 0.74 & 0.84 & 0.87 & 0.91 & 0.93 & 0.94\\
% &$p_\text{false}$  & 1.00 & 0.45 & 0.24 & 0.20 & 0.13 & 0.10 & 0.07\\
\hline
\multirow{2}{*}{M5} 
&$p_\text{stop}$ & 0.00 & 0.00 & 0.35 & 0.40 & 0.47 & \textbf{0.55} & 0.62 & 0.69 & 0.76 \\
&$p_{\text{true}|\text{stop}}$ & 0.00 & 0.00 & 0.46 & 0.65 & 0.82 & \textbf{0.88} & 0.92 & 0.95 & 0.96 \\
% &$p_\text{false}$  & 0.00 & 0.00 & 0.14 & 0.11 & 0.07 & 0.05 & 0.03\\
\hline
\multirow{2}{*}{$\bar{\text{M1}}$}
&$p_\text{stop}$  & 0.00 & 0.48 & \textbf{0.67} & 0.79 & 0.88 & 0.92 & 0.94 & 0.95 & 0.97 \\
&$p_{\text{true}|\text{stop}}$  & 0.00 & 0.38 & \textbf{0.70} & 0.80 & 0.87 & 0.90 & 0.93 & 0.95 & 0.96 \\
% &$p_\text{false}$  & 0.00 & 0.06 & 0.04 & 0.05 & 0.03 & 0.03 & 0.02\\
\end{tabular}
}
\end{minipage}
\end{center}
\caption{Toy example employing the properties listed in Table\ref{tab:disfavored_stuff}. Here we are in a $10$ class recursive classification scenario. To highlight the difference we set $\tau = .75$ and $\varepsilon_+\sim \text{lognorm}(0.8,0.5^2)$ and $\varepsilon_- \sim \text{lognorm}(-0.3,0.5^2)$. We perform 5000 MC simulations to report each number. Here we start the recursive classification where 3 candidates share the probability weight and the remaining. Share the remainder almost uniformly. Observe that as class space expands (M3) and (M4) behave similarly. Most importantly (MP) is robust and behave similar to (M1) where the lower bound $\bar{\text{(M1)}}$ results a deficit in accuracy.{\tiny(imagesrc:https://people.com/pets/tomsenn-lion-mane-dog-costume-amazon/)}
\vspace{-7mm}}
\label{tab:disfav_2}
\end{table}}
{
\begin{table}[!t]
\begin{center}
\scalebox{0.77}{
\begin{tabular}{c c |  c c c c a c c c }
% ['MP', 'M1', 'M2', 'M3', 'M4', 'M5', 'M1L']
%\multicolumn{5}{c}{Stopping Region: $\mathcal{R}_S:= \lbrace p| \mathcal{S}_O(p) \rbrace$} \\
& &\multicolumn{7}{c}{Number of Sequences} \\
Method &  &10 &11 &12 &13 &14 &15 &16 &17\\
\hline
\multirow{2}{*}{MP} 
&$p_\text{stop}$   & 0.09 & 0.18 & 0.31 & 0.45 & \textbf{0.60} & 0.70 & 0.77 & 0.86 \\
&$p_{\text{true}|\text{stop}}$  & 0.95 & 0.98 & 0.98 & 0.99 & \textbf{0.99} & 0.99 & 1.00 & 1.00 \\
\hline
\multirow{2}{*}{M1} 
&$p_\text{stop}$  & 0.05 & 0.11 & 0.21 & 0.33 & 0.48 & \textbf{0.60} & 0.70 & 0.81 \\
&$p_{\text{true}|\text{stop}}$  & 1.00 & 0.99 & 0.99 & 1.00 & 1.00 & \textbf{1.00} & 1.00 & 1.00 \\
\hline
\multirow{2}{*}{M2} 
&$p_\text{stop}$  & 0.05 & 0.11 & 0.21 & 0.33 & 0.48 & \textbf{0.61} & 0.70 & 0.81 \\
&$p_{\text{true}|\text{stop}}$   & 1.00 & 0.99 & 0.99 & 1.00 & 1.00 & \textbf{1.00} & 1.00 & 1.00 \\
\hline
\multirow{2}{*}{M3} 
&$p_\text{stop}$ & 0.05 & 0.12 & 0.22 & 0.34 & 0.49 & \textbf{0.61} & 0.71 & 0.81 \\
&$p_{\text{true}|\text{stop}}$ & 1.00 & 0.99 & 0.99 & 1.00 & 1.00 & \textbf{1.00} & 1.00 & 1.00 \\
\hline
\multirow{2}{*}{M4} 
&$p_\text{stop}$ & 0.05 & 0.13 & 0.23 & 0.35 & {0.49} & \textbf{0.62} & 0.72 & 0.82\\
&$p_{\text{true}|\text{stop}}$ & 0.98 & 0.99 & {0.99} & 1.00 & {1.00} & \textbf{1.00} & 1.00 & 1.00 \\
\hline
\multirow{2}{*}{M5} 
&$p_\text{stop}$ & 0.00 & 0.00 & 0.00 & 0.00 & 0.00 & 0.01 & 0.03 & 0.08 \\
&$p_{\text{true}|\text{stop}}$ & 0.00 & 0.00 & 0.00 & 0.00 & 0.00 & 1.00 & 1.00 & 1.00 \\
\hline
\multirow{2}{*}{$\bar{\text{M1}}$}
&$p_\text{stop}$ & 0.10 & 0.20 & 0.33 & 0.47 & \textbf{0.61} & 0.72 & 0.79 & 0.88 \\
&$p_{\text{true}|\text{stop}}$ & 0.94 & 0.98 & 0.98 & 0.99 & \textbf{0.99} & 0.99 & 0.99 & 1.00 \\
\end{tabular}

}
\end{center}
\caption{Toy example employing the properties listed in Table\ref{tab:disfavored_stuff}. Prior information in MC simulations has a fixed prior probability for the true class that is set to $0.1$ and the remainder probability is distributed using a normalization to 0.9 sum value after random assignment at each simulation. 
{\color{rev_color}Therefore the prior is in a central position away from the corner of true class and will follow a central path (See Sec.\ref{subsec:posterior_motion}).}
We set $\tau = .85$ and $\varepsilon_+\sim \text{lognorm}(0.8,0.5^2)$ and $\varepsilon_- \sim \text{lognorm}(-0.3,0.5^2)$. Observe that (MP) provides early stopping achieving the speed of $\bar{\text{(M1)}}$ by losing marginal accuracy compared to (M1). As the posterior is following a central path, uncertainty based methods and confidence methods behave similarly as explained in Sec.\ref{subsec:posterior_motion}.
\vspace{-7mm}}
\label{tab:disfav_3}
\end{table}} 
{\color{rev_color}Table \ref{tab:disfavored_stuff} and Table \ref{tab:disfav_2} collectively explain items (1)-(2) such that, in these tables $1/3$ and $7/10$ classes are disfavored already in the beginning of estimation. Priors with disfavored classes lead to a trajectory that follows a path closer to sides of the simplex and hence result in rushed or immediate stopping yielding probability $p(\text{stop})\approx 1$ within few sequences. However, for both (M3)(M4) accuracy $p(\text{true}|\text{stop})$ is very low and as we increase dimensionality $3\rightarrow 10$ in Table \ref{tab:disfav_2} (M3) characteristics gets closer to (M4) (as also discussed in the manuscript Fig.~\ref{fig:effects_cardinality}-(a)). In such cases a method should collect evidence to increase its accuracy to a desired level and as promised, (MP) performs comparably to confidence based methods reaching $p(\text{true}|\text{stop})\geq 94$ in exchange for 5 sequences. Observe lowering confidence threshold with $\bar{\text{(M1)}}$ results in a significant $27\%$ accuracy loss (Table \ref{tab:disfav_2}) for which (MP) still achieves high accuracy.}
{\color{rev_color}Table \ref{tab:disfav_3} demonstrates (3) where the prior is designed such that the probability of the true class is set to $0.1$ where the remaining probabilities are almost equally distributed. This places the prior probability on a centric position and hence the posterior follows a more central path. Differently in this experiment all methods require $\geq 10$ sequences to terminate. Following Sec.~\ref{subsec:posterior_motion}, uncertainty based methods and confidence behave similar. In such scenarios (MP) achieves similar accuracy  $p(\text{true}|\text{stop})$ 1 sequence earlier at $14^\text{th}$ sequence with $p(\text{stop})\geq .5$ compared to other methods (i.e., (MP) is faster with the same accuracy compared to (M1-5)). Observe that, (MP) stops earlier at a lower confidence threshold similar to $\bar{\text{(M1)}}$ but this is without the accuracy expense in previous scenario where $\bar{\text{(M1)}}$ sacrifices $24\%$ accuracy (see Table~\ref{tab:disfav_2} $\bar{\text{(M1)}}$ column-3).}

To summarize the findings, we visualize time accuracy trade-off for each method for the cases presented in Tables \ref{tab:disfav_2} and \ref{tab:disfav_3}. These results are presented in Fig.~\ref{fig:Final_figure_performance}. {\color{rev_color}Top figure in Fig.~\ref{fig:Final_figure_performance}, we observe that (MP) allows us to select (M1)'s sub-curve with high accuracy unlike (M1L=$\bar{\text{M1}}$, M2,M3). From bottom illustration, where the posterior trajectory follows a central path, we observe that (MP) provides a better speed-accuracy operation curve compared to all other methods.}

%
% TODO: make the aspect ratio equal for the figures.
\begin{figure}[t]
\centering
		\centering
	   \subfigure{\includegraphics[width=.85\columnwidth]{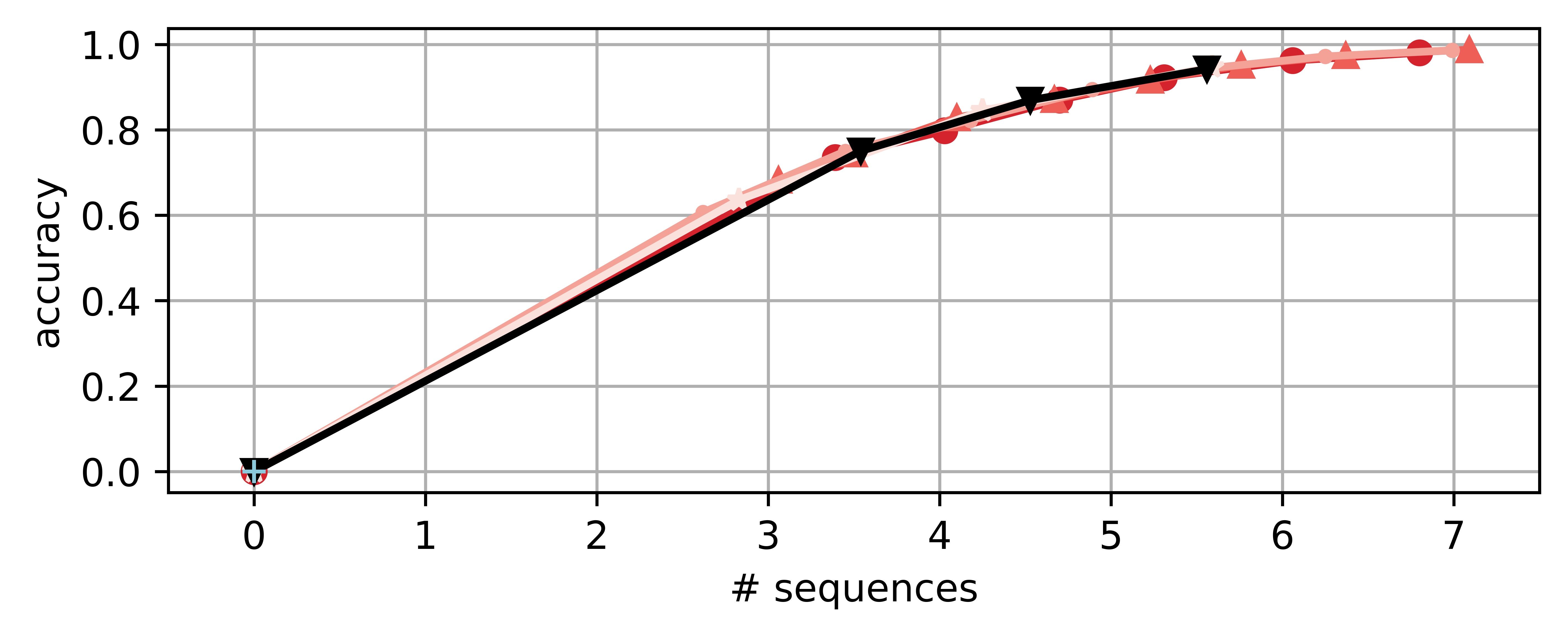}}
	   \subfigure{\includegraphics[width=.85\columnwidth]{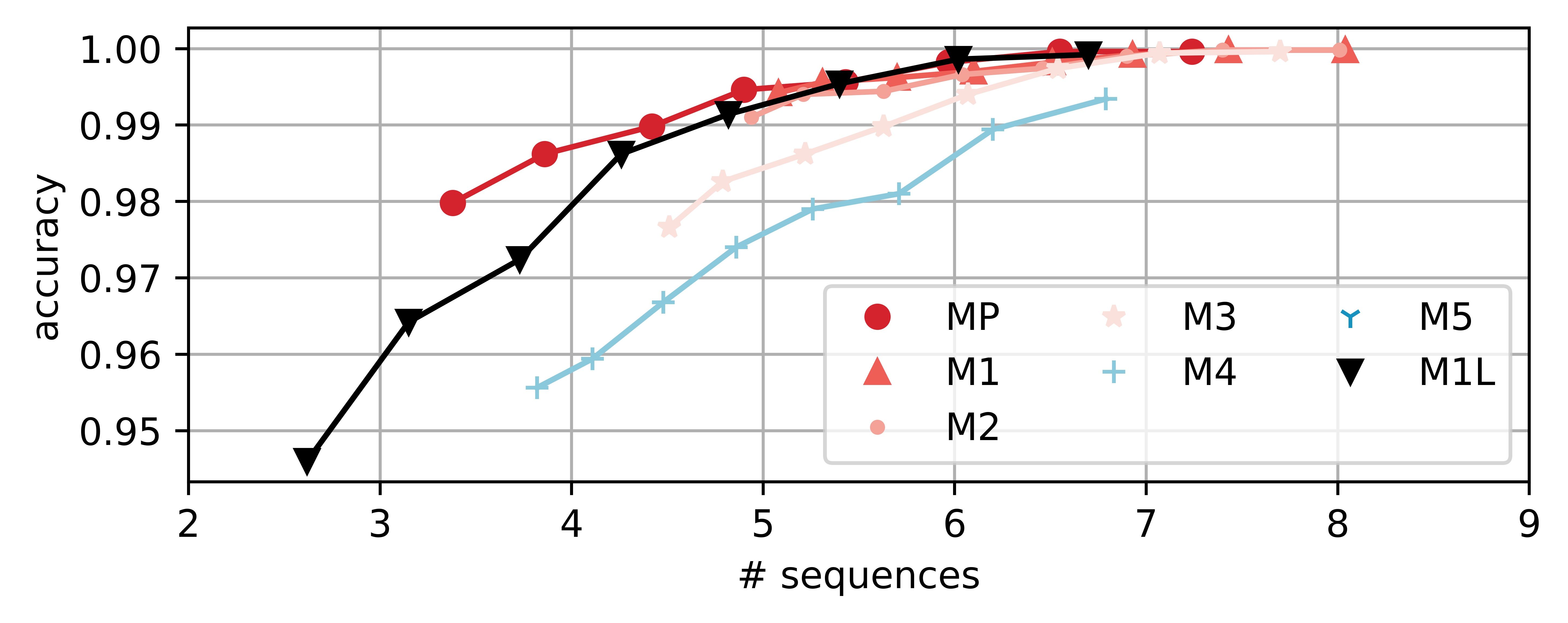}}\\
      \caption{Summary of operation characteristics of the methods using the starting conditions for recursive classification presented in Table\ref{tab:disfav_2} and Table\ref{tab:disfav_3} respectively on top and bottom. For our experiments we use $\varepsilon_+ \sim \text{lognorm}(0.6,0.5^2)$ and $\varepsilon_- \sim \text{lognorm}(0,0.5^2)$. We run 5000 recursive classification simulations and report average accuracies and average sequences for different $\tau$ values. Each line in the figure are drawn with a collection of (accuracy, sequence) points where each are computed for $\tau = [0.65,0.69,0.72,0.76,0.79,0.83,0.86,0.9]$ from left to right. We omit (M5) as it spends way more sequences. Observe that (M4) staggers, especially for the disfavored case (left). (MP) allows us to select an operation point that favors accuracy in the disfavored and gains speed by losing marginal accuracy where the posterior is following a central path (right).}
    \label{fig:Final_figure_performance}
    \vspace{-0.3cm}
\end{figure}
}

\subsection{Real Data Experiments}
\subsubsection{Experimental Details}
In our experiments we use a BCI typing system called RSVP Keyboard presented in Orhan's work \cite{orhan2012rsvp}
{and the implementation BCIPy (https://github.com/CAMBI-tech/BciPy) \cite{memmott2020bcipy}. The system is visualized for the stimulus screen and an actual healthy participant performing a task in Fig.~\ref{fig:RSVPKeyboardFig}.}
{\color{rev_color}System is tasked to correctly classify user's intended symbol in mind ($\sigma$). The system utilizes EEG as evidence ($\varepsilon$) and a language model that serves as a prior ($p(\sigma|\mathcal{H}_0)$). A sequence in this application refers to a series of rapid letter presentations to the user and collecting respective evidence. At the beginning of each sequence the system selects a predetermined number of most likely letters ($\Phi_s$) with respect to latest posterior distribution $p(\sigma|\mathcal{H}_{s-1})$. Due to noisy observations ($\varepsilon_s$) at each sequence, the system relies on recursive evidence collection. Collected EEG evidence and prior information is fused in a Bayesian fashion ($p(\sigma | \mathcal{H}_1) = p(\sigma\lvert\mathcal{H}_0) \oplus p(\beps_0\lvert \sigma, \Phi_0)$ see Appendix~\ref{sec:how_algebra_works}). In the signal p300 evoked potential presence corresponds to a positive response and absence corresponds to a negative response. Once intended symbol appears on the computer screen ($\sigma\in\Phi_s$), subject's recorded brain signal evokes a distinguishable response \cite{woodman2010brief}. The system terminates evidence collection once a stopping condition is met and in this section we show that (MP) provides early stopping with marginal accuracy penalty.}

{\textbf{Data Collection:} Ten healthy participants (six females), 20-35 years old were recruited under IRB-130107 protocol approved by Northeastern University. A DSI-24 Wearable Sensing EEG Headset was used for data acquisition, at a sampling rate of 300 Hz with active dry electrodes.All participants performed the calibration session containing 100 sequences; each sequence includes 5 trials; and one trial in each sequence is the target symbol which is displayed on the screen prior to each sequence (RSVP paradigm).}
A sequence contains randomly ordered ten symbols with a pre-defined target symbol. EEG is acquired from 16 channels using the International 10–20 configuration (Fp1, Fp2,F3, F4, Fz, Fc1, Fc2, Cz, P1, P2,C1, C2, Cp3, Cp4, P5, P6). %The system details are presented in the Supplementary Material {\color{red} will add soon}. 
Recorded EEG are used to learn class conditional EEG evidence distributions. 

\begin{figure}[t!]
    \centering
      \subfigure{\includegraphics[width=\linewidth]{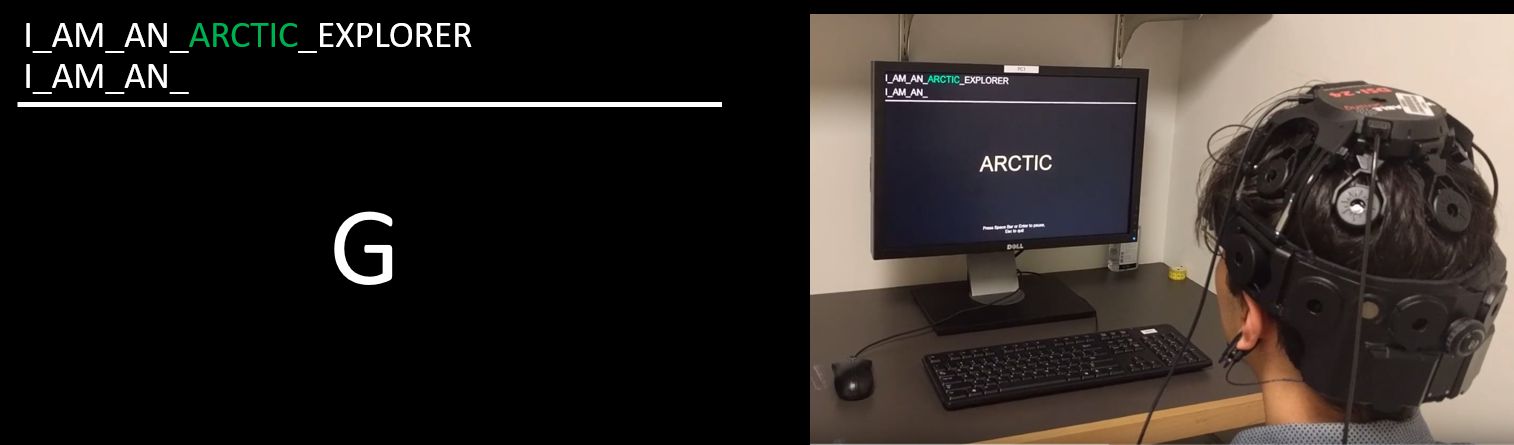}}
      \textsf{(a) \hspace{38mm} (b)}
      \caption{{EEG driven Rapid Serial Visual Presentation (RSVP) keyboard typing interface. (a) The stimuli is flashed in the middle of the screen while the user is informed with the text above. (b) The user is conducting copying the phrase task (multiple copy letter tasks). The user is informed about the required phrase. EEG is collected on top of the scalp non-invasively.}}
      \label{fig:RSVPKeyboardFig}
      \vspace{-0.2cm}
\end{figure}

{\textbf{Pre-Processing:} In EEG-BCIs the primary interest of filtering is to extract the signal of interest components~\cite{mur14,moghadam15}. We first filter the EEG signal to remove drifts and artifact-related high frequency components with a band pass filter for [1,50]Hz. After filtering, EEG is windowed to extract the respective evidence at each channel for stimuli presentations. Time-windowed data from different EEG channels is usually concatenated to obtain the EEG feature vector that has a high dimension because of using a multi-channel measurement. Therefore, dimensionality reduction using ICA or PCA is also needed~\cite{mur14}. The system relies on reducing EEG time series into one dimensional feature vector. Filtered multi-channel EEG data {time windows are} passed through channel-wise principal component analysis where the outputs are concatenated to an intermediate feature vector. We assume in each class, feature vectors are drawn from a multivariate Gaussian distribution and hence Regularized discriminant analysis (regularized quadratic discriminant analysis \cite{friedman1989regularized}) is a plausible choice that results in one dimensional representation of the signal. Each positive and negative sample in the calibration EEG data is reduced to a single dimensional feature and positive and negative feature distributions are learned accordingly.}

{\textbf{Experimental Task:}
For our experiments we use these distributions in Copy-Letter task such that each user's data is used to type a target letter within a pre-determined phrase for multiple phrases~\cite{orhan2016probabilistic}.
More specifically, for each letter typing scenario the user is tasked to respond to the system, and the decision is made when the cumulative evidence matches the correct letter after multiple recursions. Therefore, in such a setting, the decision chance level is 0.03\%. }
To make a decision the system recursively queries the user with multiple letter flashes. We designate the number of queries to be presented at each sequence to be  $N\in\lbrace15,10,5 \rbrace$. We present the results for two conditions. First one is a typing scenario with no language model (uniform prior information). And in the second scenario there exists a language model for the requested typing and the candidate letter is in top 16. 
{ The choice of two scenarios illustrated in the numerical results is to represent the following; (i) Class priors are uniform (in the 28-vertex simplex, RBC starts from the center of the simplex labeled as $u$ in Fig.~\ref{fig:simplex_repr}-(d). (ii) The prior probability for the correct class label (desired letter of the user) is selected to be significantly lower than in the first case (in the 28-vertex simplex, RBC starts further away from the target vertex). These cases represent typical situations with uninformative prior and adversarial prior,  additionally these are the challenging cases for RBC.
}
\subsubsection{Results}
%
% TODO: make the aspect ratio equal for the figures.
\begin{figure}[!t]
      \subfigure{\includegraphics[width=.8\columnwidth]{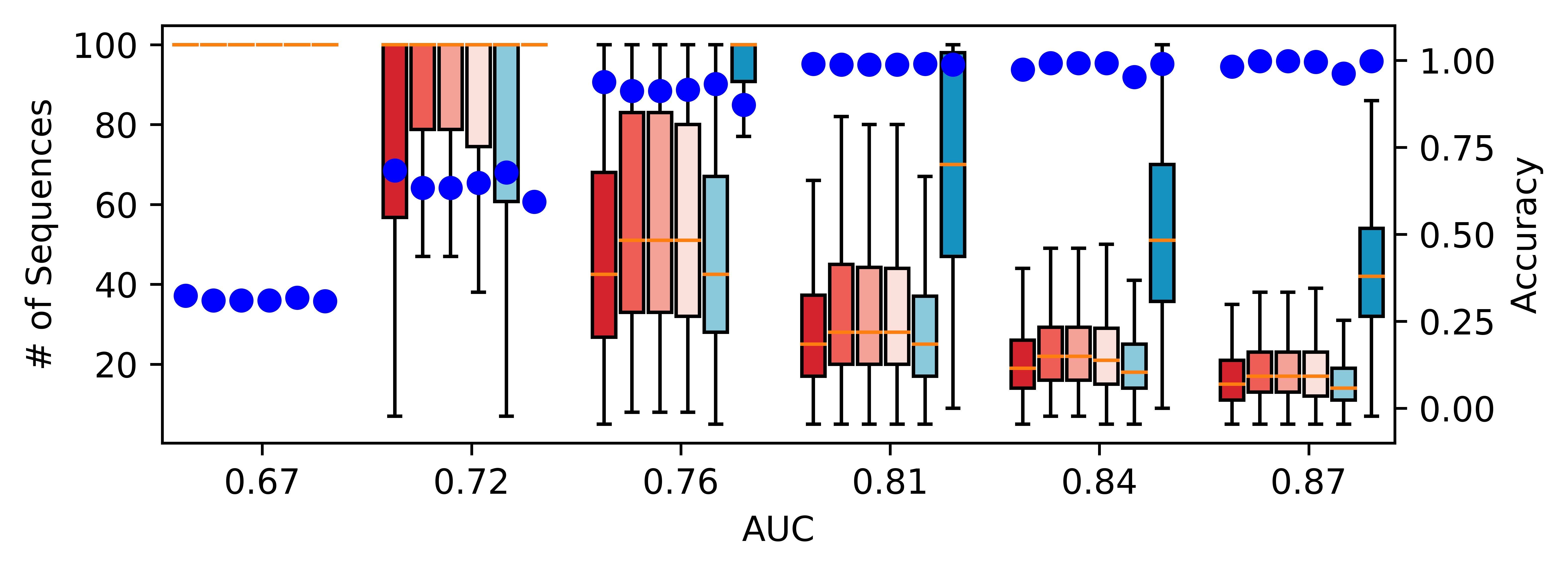}}
      \subfigure{\includegraphics[width=.8\columnwidth]{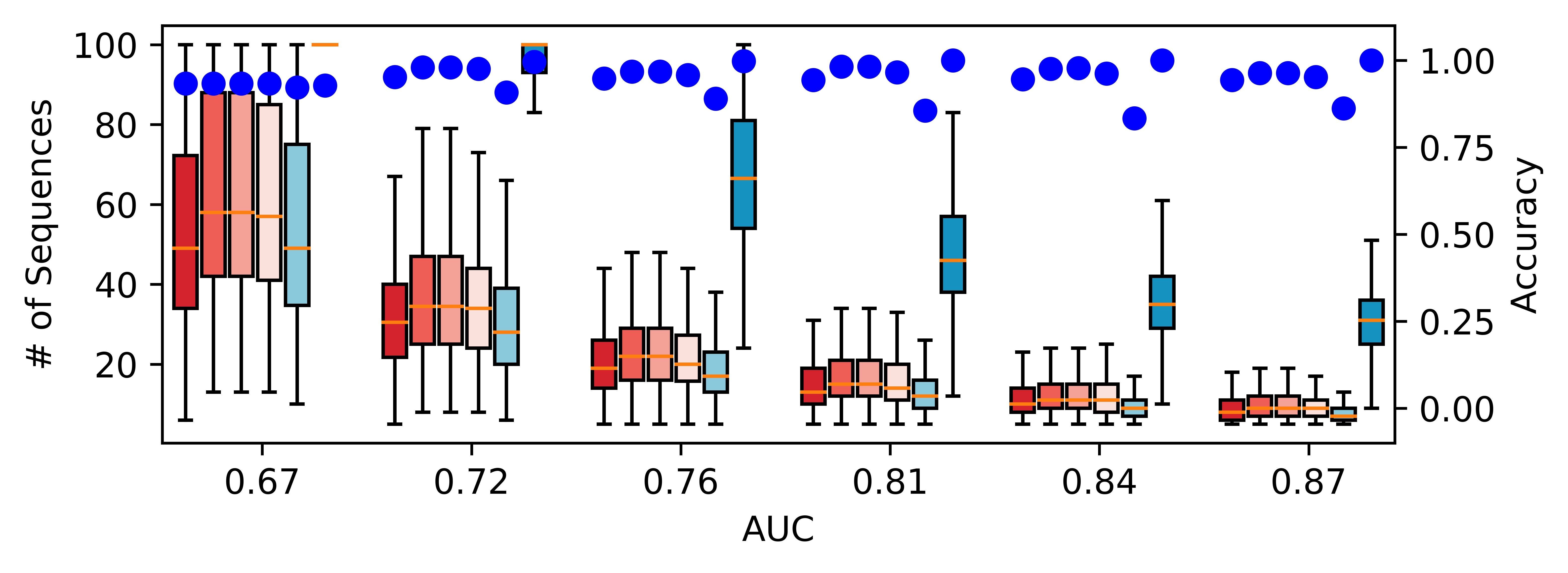}}
      \subfigure{\includegraphics[width=.80\columnwidth]{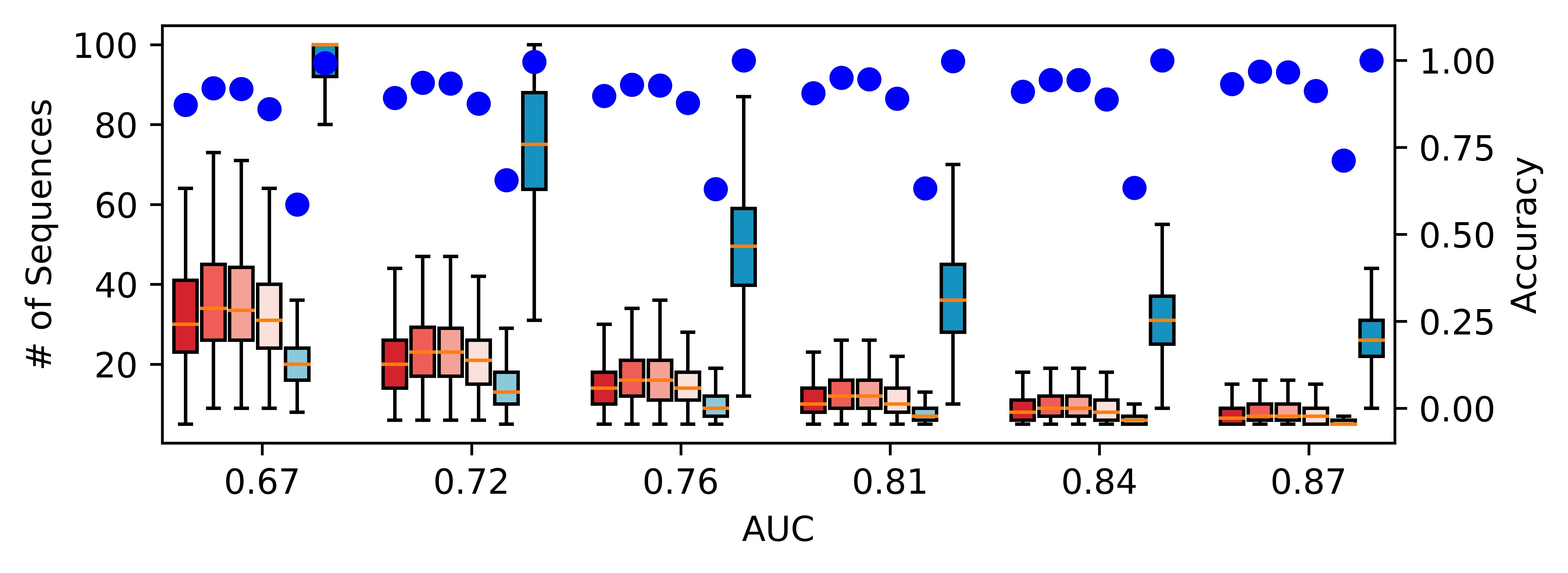}}
      \subfigure{\includegraphics[width=.15\columnwidth]{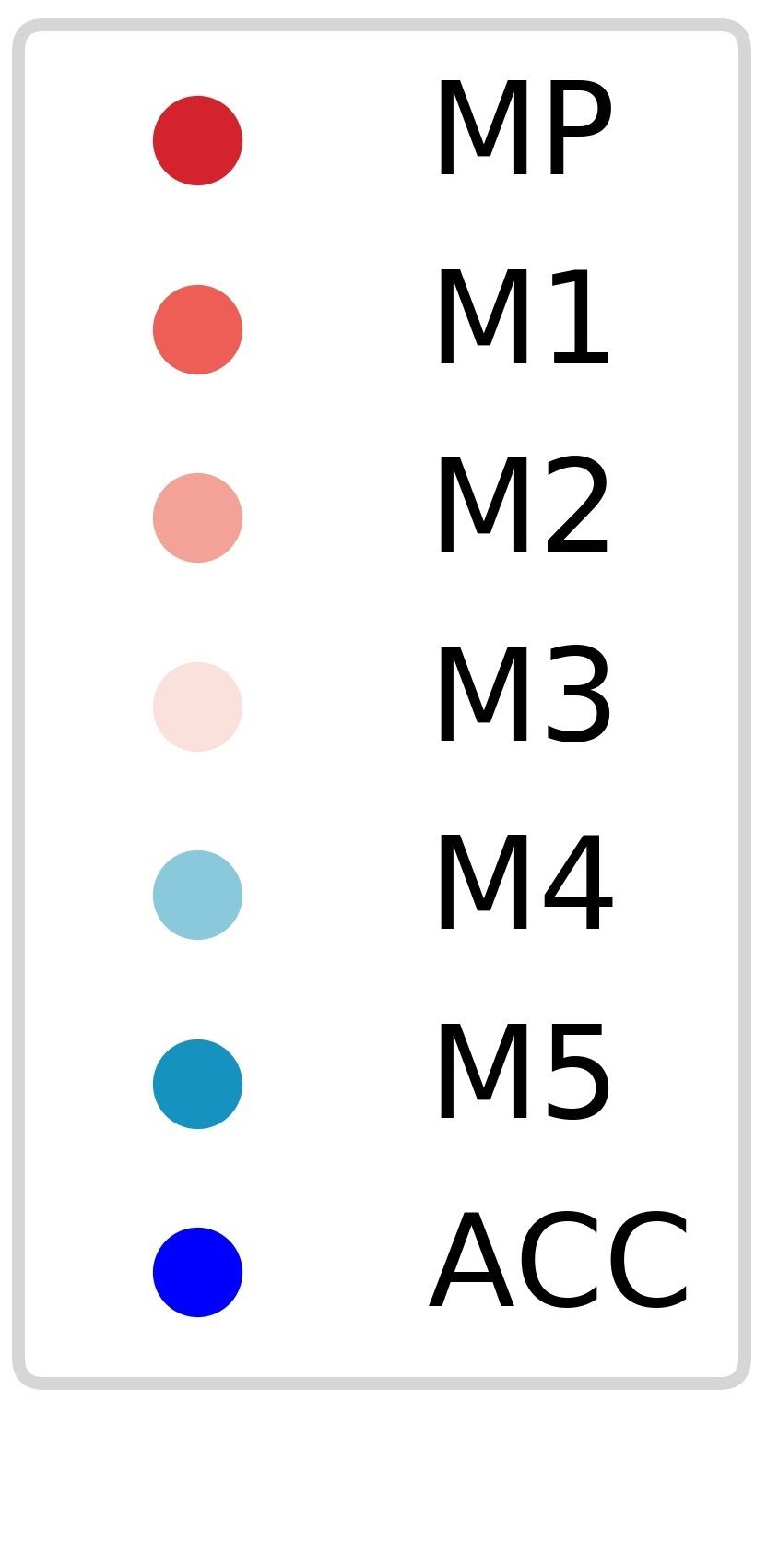}}
      \vspace{-0.4cm}
      \caption{Number of recursion spent and accuracy plots for recursive classification in BCI typing system. Each scenario is generated using human-in-the-loop calibration data trained generative models. In each figure results are presented in ascending order of performance measures (area under receiver operation characteristics curve (AUC)). Number of queries in each recursion from top to bottom, 5, 10, 15 respectively. Legend covers methods from left to right and dots on the figures represent respective accuracy values. The users tried to type "A" without any language model (uniform prior information). Top to bottom legend order is from left to right for each block.}
    \label{fig:factor_plots_synth_1}
    \vspace{-0.3cm}
\end{figure}
In our experiments we categorize the users based on their calibration performances. The measure of performance is the area under receiver operating characteristics curve (AUC) of classification based on the features extracted during calibration. We specifically selected user data with AUC performances $\lbrace0.67,0.72,0.76,0.81,0.84,0.87\rbrace$ to have a spectrum of ranging performances. 

We visualize our findings in Fig.~\ref{fig:factor_plots_synth_1} and \ref{fig:factor_plots_synth_2} respectively. {It is observed that (MP) switches  the operation point to a location such that faster results are obtained with the cost of small amount of decreases in accuracy. To show the significance of the (MP), we present here a scenario that includes correctly typing 100 letters on a computer screen. 
{We refer the reader to appendix Sec.\ref{sec:bci_supplementary} for the complete results that allowed us to generate visualizations.}

Comparing with the conventional method (M1), when uniform prior is used, (MP) outperforms (M1) in terms of speed, i.e., (MP) and (M1) complete the same task with 1735 and 1945 sequences respectively. Accordingly (MP) saves 210 sequences. This corresponds to saving $3(m)30(s)/32(m)$ lifetime during typing. Additionally, if a language model prior is used for the same task, (MP) still outperforms (M1), 1580 sequences vs 1728 sequences saving $2(m)24(s)/28(m)48(s)$ lifetime during typing. These reported amount of time are computed under the condition that at the end of the task 100 letters are completely correctly typed including corrections of the wrongly typed letters (i.e., email is completely correct).  Saving time is very crucial for practical BCI typing as these systems are designed for individuals with limited speech and physical abilities. Therefore, fatigue and discomfort caused by the BCI system are important factors that significantly affect the BCI typing performance, and limiting the time to complete the tasks accurately in practice. }

%
% TODO: make the aspect ratio equal for the figures.
\begin{figure}[!t]
      \subfigure{\includegraphics[width=.8\columnwidth]{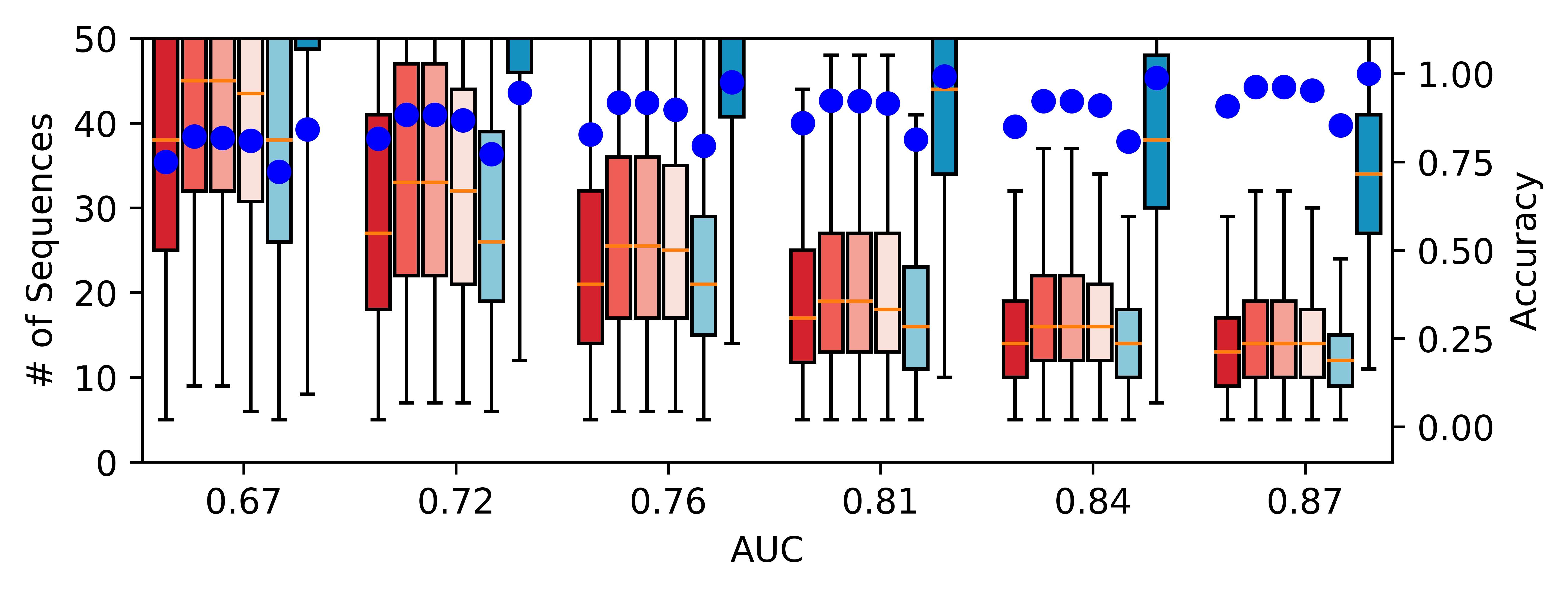}}
      \subfigure{\includegraphics[width=.8\columnwidth]{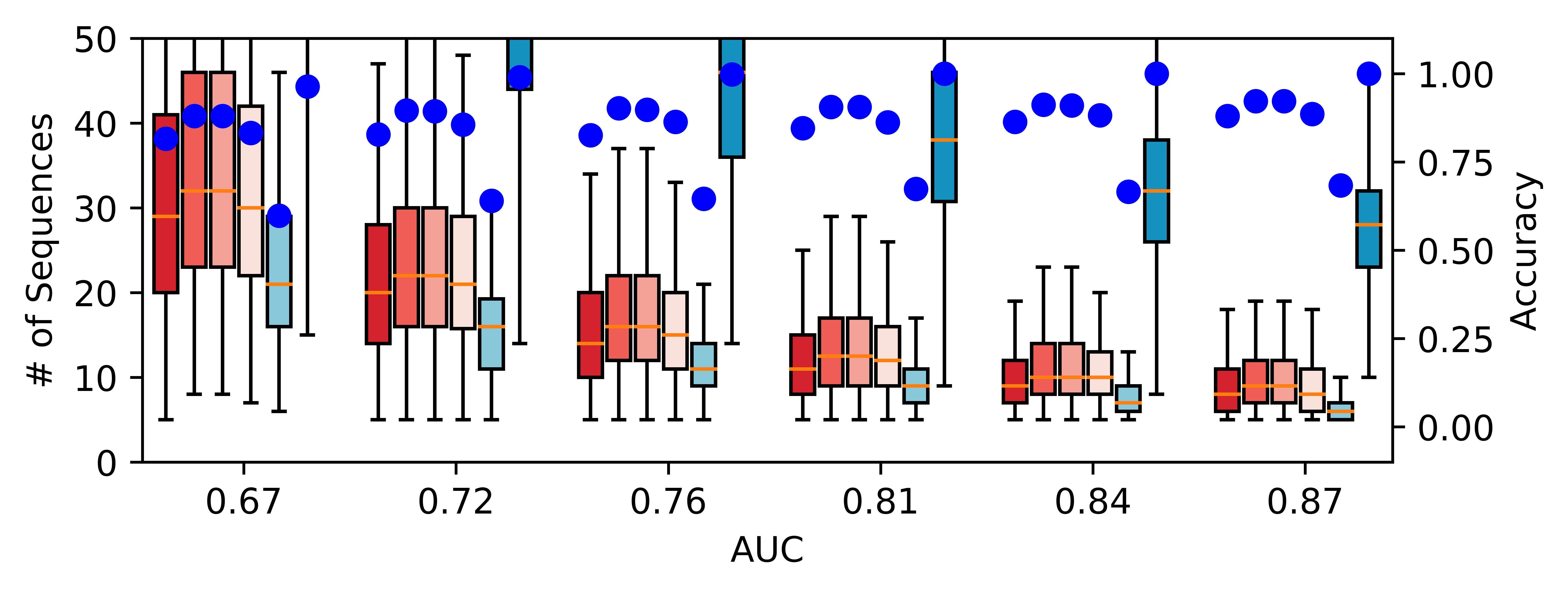}}
      \subfigure{\includegraphics[width=.8\columnwidth]{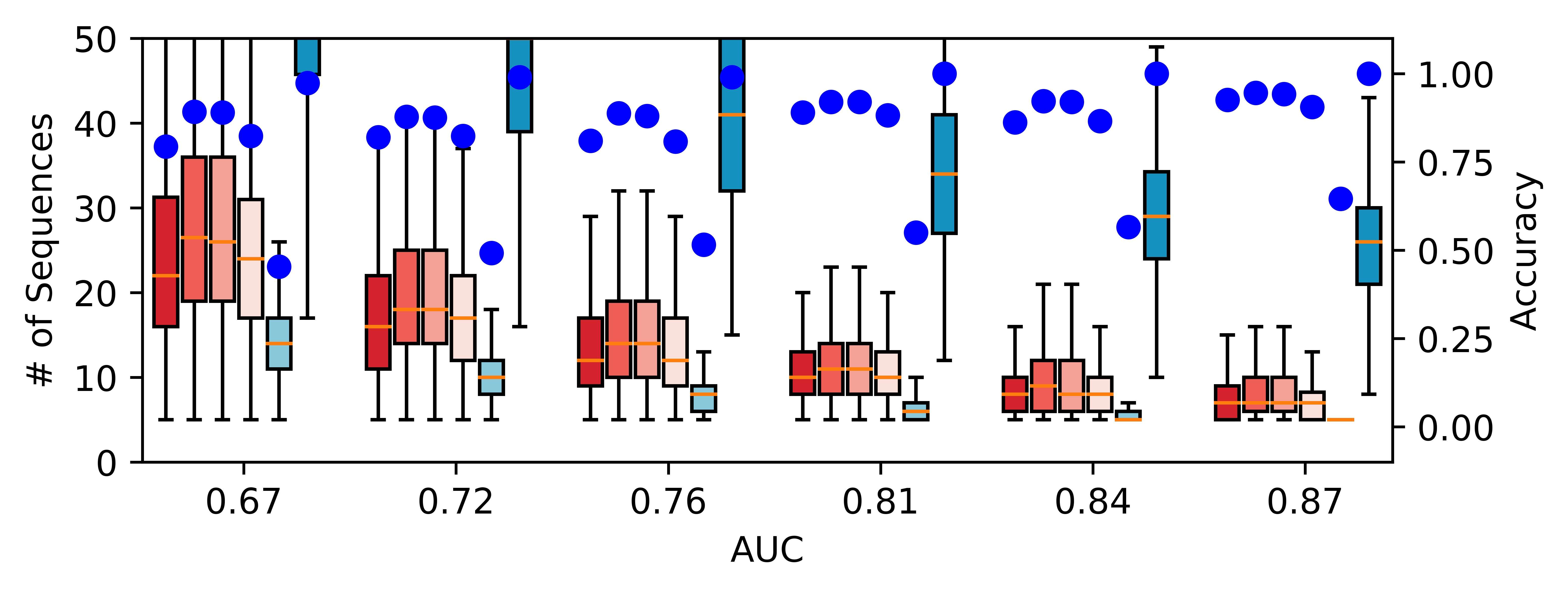}}
      \subfigure{\includegraphics[width=.15\columnwidth]{fig/factor_plots_synth/legend_.jpg}}
      \vspace{-0.4cm}
      \caption{The same implementation as described in Fig.\ref{fig:factor_plots_synth_1}. The difference here is the subjects were supported by a language model trying to type "IT\_O" given "IT\_" and hence even though "O" is not the top letter it is one of the likely letters. }
    \label{fig:factor_plots_synth_2}
    \vspace{-0.3cm}
\end{figure}

\section{Conclusion}
In this paper, our focus was on the analysis and design of stopping criterion for recursive Bayesian classification. 
{\color{rev_color}
Accordingly, we proposed a stopping condition design to overcome the sensitivity to number of classes of the existing uncertainty methods, while still providing an early stopping opportunity with marginal accuracy cost. Proposed method offers true positive probability rates above the true positive rates of confidence thresholding while limiting the false alarm rate which we have shown analytically. We experimentally validated that proposed approach sustains an accuracy value where uncertainty methods are suffering while still allow for early stopping where confidence based methods struggle. It is possible to stop earlier with a lower confidence threshold to stop as early as proposed method where posterior follows a central trajectory. Synthetic experiments show that, lowering the confidence threshold is penalized in accuracy for disfavored cases where the proposed approach still sustained a high accuracy level. This emphasizes the importance of utilizing RBC geometry not to experience unwanted accuracy loss in certain scenarios. Therefore, it is possible to say the proposed method is combining the strengths of conventional methods. We also validated the practical use case on a brain computer interfaced typing system. This paper focuses on proposing a perspective and the proposed piece wise linear approach in Sec.~\ref{sec:Method} is a simplistic yet powerful design. We aim to explore the limitations of the curve order in stopping criterion design and propose new objectives that result in better performance. Also in Sec.~\ref{sec:Experiments} we experiment with evidence models with varying performances. A future work can center around providing an analytical relationship between evidence model performance (e.g. variance, AUC etc.) and stopping region performance. 
}
This work can further be extended by considering true positive maximization and false negative minimization as a multi-objective optimization.

%%%%%%%% LIST SUPPLEMENTARY MATERIAL CITATIONS %%%%%%%%%%%%%
\nocite{barcelo2001mathematical}
\nocite{renyi1961measures}
{\tiny{\setstretch{0.8}
\bibliographystyle{IEEEtran}
%\scriptsize
\bibliography{ref}}
}
\begin{IEEEbiography}[{\includegraphics[width=1in,height=1.25in,clip,keepaspectratio]{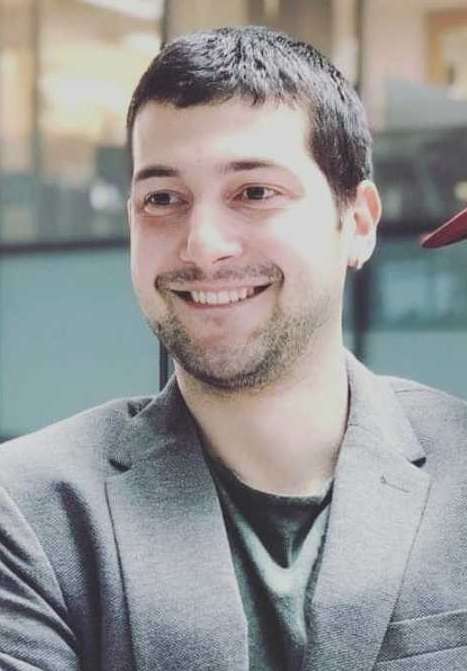}}]%
{Aziz Ko\c{c}anao\u{g}ullar\i}
received the B.S. degree in {\color{rev_color}electronics engineering and mathematics from Istanbul Technical University, Istanbul Turkey in 2014 and 2015 respectively. He received M.Sc. degree in telecommunications engineering in Istanbul Technical University in 2016. He received Ph.D. degree in electrical engineering in Northeastern University Boston, MA, USA in 2020. Since then he is a member of Computational Radiology Laboratory (CRL) in Boston Children's Hospital - Harvard Medical School Boston, MA, USA where he is currently a post doctoral research fellow. His main areas of research interest are active recursive inference in sequential decision making processes and active model learning.}
\vspace{-0.5cm}
\end{IEEEbiography}
\begin{IEEEbiography}[{\includegraphics[width=1in,height=1.25in,clip,keepaspectratio]{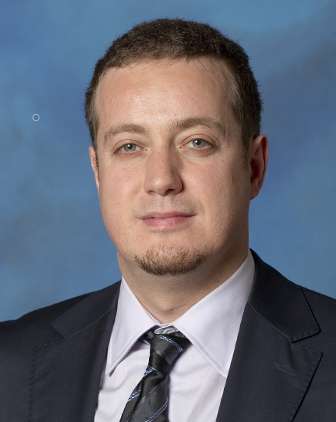}}]%
{Murat Akcakaya}
received the Ph.D. degree in electrical engineering from Washington University in St. Louis, MO, USA, in December 2010. He is an {\color{rev_color}Associate} Professor in the Electrical and Computer Engineering Department of the University of Pittsburgh. His research interests are in the areas of statistical signal processing and machine learning.
\vspace{-0.5cm}
\end{IEEEbiography}
\begin{IEEEbiography}[{\includegraphics[width=1in,height=1.25in,clip,keepaspectratio]{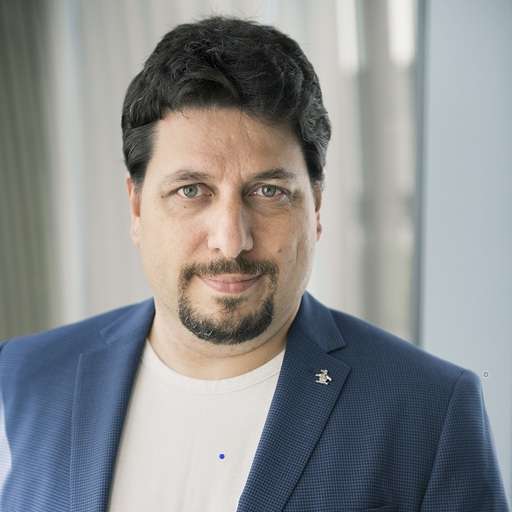}}]%
{Deniz Erdo\u{g}mu\c{s}}
received the B.S. degree in electrical engineering and mathematics, and the M.S. degree in electrical engineering from the Middle East Technical University, Ankara, Turkey, in 1997 and 1999, respectively, and the Ph.D. degree in electrical and computer engineering, in 2002, from the University of Florida, Gainesville, FL, USA, where he was a postdoc until 2004. He is currently a Research Professor at Northeastern University. His research focuses on statistical signal processing and machine learning with applications to biomedical signal/image processing and cyberhuman systems.
\vspace{-0.5cm}
\end{IEEEbiography}

\clearpage
 {%
  \section{Appendix}
  % TODO: Probably dump it into supplementary materials
\subsection{How algebra works in posterior updates}
\label{sec:how_algebra_works}
\textit{Addition:} Given $p,q\in\Delta_n$%
    \begin{equation*}
    \begin{split}
        p \oplus q = \frac{[p_1q_1,p_2q_2,\cdots,p_nq_n]}{\sum_i p_iq_i}
        \end{split}
    \end{equation*}
Given $p(\sigma)\in\Delta_{|\mathcal{A}|}$ where $\mathcal{A}$ is the state space and $p(\varepsilon | \sigma,\phi) = [p(\varepsilon| \sigma_1 ,\phi),p(\varepsilon| \sigma_2 ,\phi),\cdots, p(\varepsilon| \sigma_{|\mathcal{A}|} ,\phi)]$; $p(\varepsilon | \sigma_{|\mathcal{A}|},\phi)\notin \Delta_{|\mathcal{A}|}$ and hence one cannot rigorously define $p(\sigma) \oplus p(\varepsilon \lvert \sigma,\phi)$. To be able to do such $p(\varepsilon \lvert \sigma,\phi)$ should be unit $\ell_1$ norm wrt. $\sigma$ and hence we reqiore $p(\varepsilon \lvert \sigma,\phi)/(\sum_\sigma p(\varepsilon \lvert \sigma,\phi))$ which is not practical. However the closure operator (normalization as defined in \cite{aitchison1982statistical}) that maps an arbitrary point to $\Delta_n$ allows algebraically $p(\sigma) \oplus p(\varepsilon \lvert \sigma,\phi)$ as the following;
\begin{equation*}
    \begin{split}
        p(\sigma) \oplus \frac{p(\varepsilon \lvert \sigma,\phi)}{\sum_\sigma p(\varepsilon \lvert \sigma,\phi)} = \frac{[p(\sigma_i)p(\varepsilon \lvert \sigma_i,\phi)]_i / \sum_\sigma p(\varepsilon\lvert\sigma,\phi)}{ \sum_i( p(\sigma_i)p(\varepsilon \lvert \sigma_i,\phi)) / \sum_\sigma p(\varepsilon\lvert\sigma,\phi)}\\ = \frac{[p(\sigma_i)p(\varepsilon \lvert \sigma_i,\phi)]_i}{ \sum_i( p(\sigma_i)p(\varepsilon \lvert \sigma_i,\phi))} 
        \simeq p(\sigma) \oplus p(\varepsilon \lvert \sigma,\phi)
    \end{split}
\end{equation*}
Following this equation we represent a posterior with the following update; $p(\sigma | \varepsilon, \phi)=p(\sigma)\oplus p(\varepsilon| \sigma, \phi)$ which is analogous with posterior$=$prior$\oplus$likelihood.
\subsection{Uncertainty Decision Boundaries / The reason
behind analytically analyzing only (M3)}
\label{sec:why_entropy_is_enough}
\label{sec:reasonform3}
In this section we reason the decision on only analyzing Shannon's entropy in Section\ref{subsec:frail_confidence}. In information theory, Renyi entropy generalizes Shannon entropy \cite{renyi1961measures}. The definition for Renyi entropy, parameterized over $\alpha$ is the following;
\begin{equation}
    p\in\Delta_n, \ H_\alpha(p)= \frac{1}{1-\alpha} \log\sum_i p_i^\alpha
\end{equation}
Observe that the limit case $\lim_{\alpha\rightarrow 1} H_\alpha(p)= H(p)$ results in Shannon entropy measure. In this paper we only propose analytical derivations for Shannon entropy as a special case. However the findings can directly be applied to Renyi measures. The generalization can be analytically shown, but to have a neat presentation we omit the derivations here. However, in Fig.\ref{fig:renyi_shannon_case} we present decision boundaries for Renyi entropy and Shannon as a special case to demonstrate their similar decision geometry.
%
% TODO: make the aspect ratio equal for the figures.
\begin{figure}[!h]
\begin{center}
      \subfigure{\includegraphics[width=.46\columnwidth]{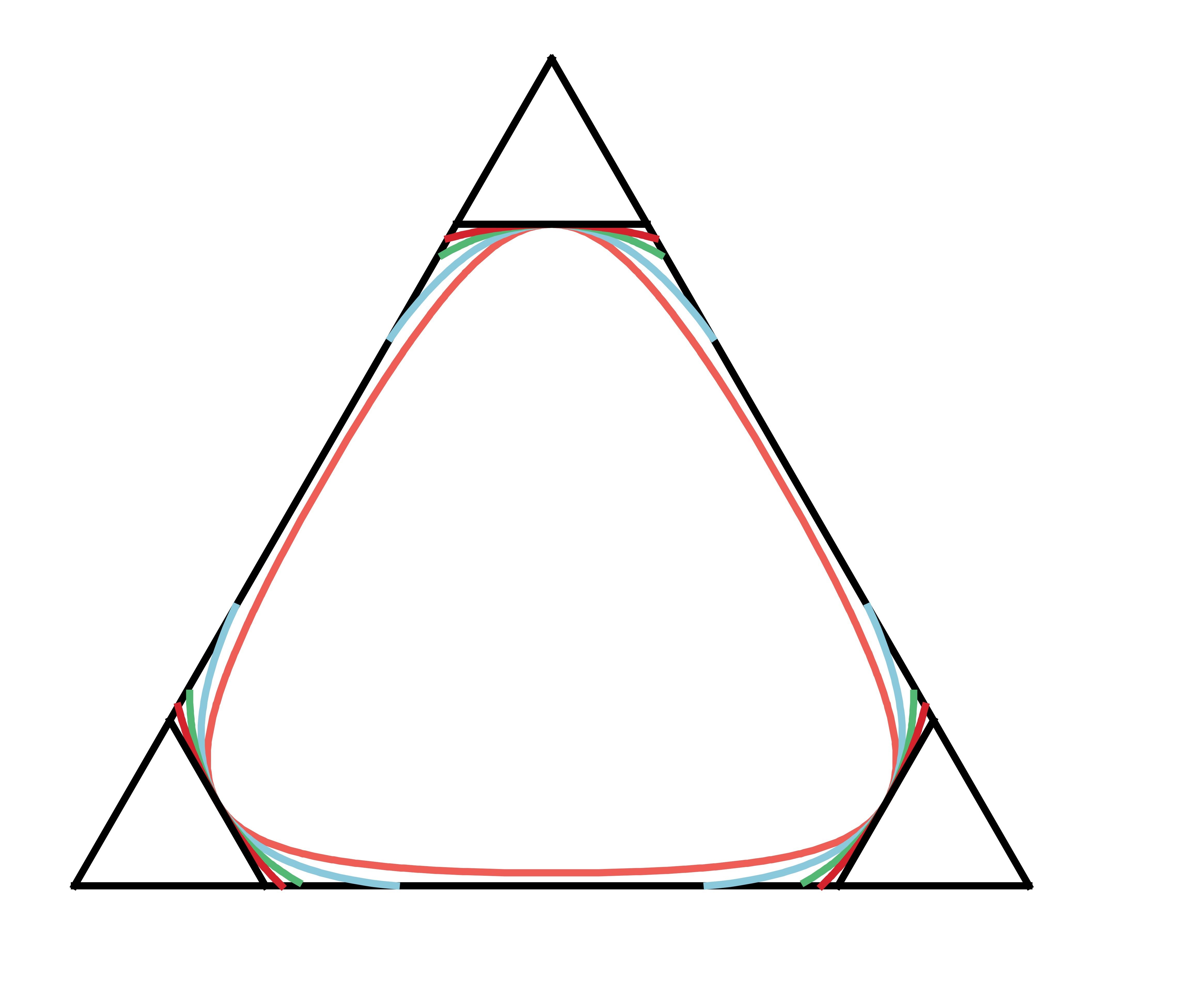}}
      \subfigure{\includegraphics[width=.46\columnwidth]{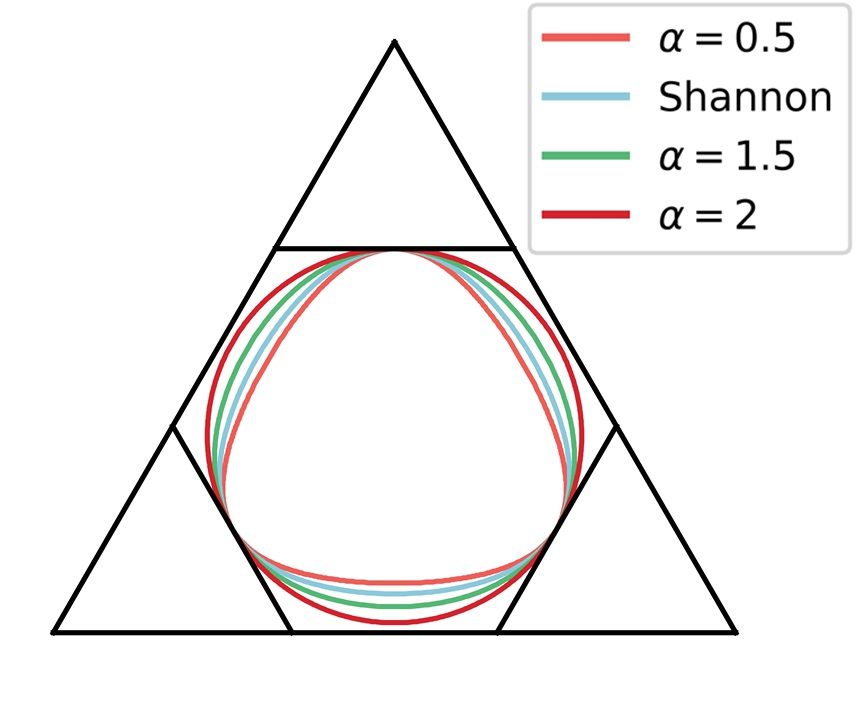}}
      \end{center}
    \label{fig:sim_risk_plot}
    \caption{Decision boundaries formed using $\mathcal{S}_O$ for Renyi using $\alpha\in\lbrace 2,1.5,0.5\rbrace$ and Shannon entropy ($2,1.5,Shannon,0.5$ in the outer to inner order in the figures). Confidence lines are plotted for reference. Corresponding values for confidence lines are $\tau = 0.8$ for left and $\tau = 0.65$ for right figure.
    \label{fig:renyi_shannon_case}}
\end{figure}

\subsection{Proposed Method Supplementary}
\label{sec:experiments_supplementary}

\subsubsection{True Positive - False Alarm Guarantee}
\label{sec:tp_fa_guarantee}
In this section we present true selection and error probabilities for simple examples that can be visualized on a three class simplex $\Delta_3$. We present these result to support our claims in Section \ref{sec:Method} of the manuscript. We pick several points on the simplex and with a predefined evidence distribution sampling from a lognormal, we visualize the bounds on error in Fig.\ref{fig:TP_FP_plots} below. We plot the analytical bounds derived in the proof of Proposition \ref{prop:perf_guarantees} using lognormal distribution assumptions. Lower bound is represented with (M1)(red) and the upper bound is represented with $\bar{\text{(M1)}}$ (black). Instead of plotting the analytic values of (MP) we plot the average probability values calculated over 5000 Monte Carlo simulations. We compute $p(p_s\in S_R)$ and $p(p_s\in S'_R)$ values for different starting points that are color coded in the Fig.\ref{fig:TP_FP_plots}. To generate the figures we use $\varepsilon = [\varepsilon_+, 1,\cdots,1]$ where $\varepsilon_+\sim\text{lognorm}(0.8,0.6^2)$ and the true class is $a$. Given this evidence model, all posteriors follow a straight path to corner $a$ (this behavior is discussed in Lemma\ref{lemma:collinearity_of_post} in the manuscript).

In this paper we do not provide a complete analysis on the effects of the evidence distribution. However, by nature, in recursive Bayesian classification tasks, the separation of the evidence from different classes increases both accuracy and speed. We pick the prior points demonstrated in Fig.\ref{fig:TP_FP_plots} and investigate the effects of the distribution on probability of reaching the correct region and probability of reaching an incorrect region. For our experiments, we assume the system only collects evidence from an assumed distribution $\varepsilon= [\varepsilon_+, 1,\cdots, 1]$ where $\varepsilon_+ \sim \text{lognorm}(\mu,c^2)$ for $s=5$ times. We visualize our findings for a range of parameters in Fig.\ref{fig:TP_FP_dist_effect}. As expected, as we increase the standard deviation ($c$) of the exponentiated Gaussian distribution, the system is more likely to make errors and less likely to decide correctly. On the other hand as mean ($\mu$) increases, the system yields more accurate decisions and less error.

Through the Fig.s \ref{fig:TP_FP_plots} and \ref{fig:TP_FP_dist_effect}, we summarize our claims that are presented in Section  \ref{sec:Method} especially about the bounds on error and correct selection probability given in Proposition \ref{prop:perf_guarantees}.

%
% TODO: make the aspect ratio equal for the figures.
\begin{figure*}[!t]
\centering
      \subfigure{\includegraphics[width=.45\columnwidth]{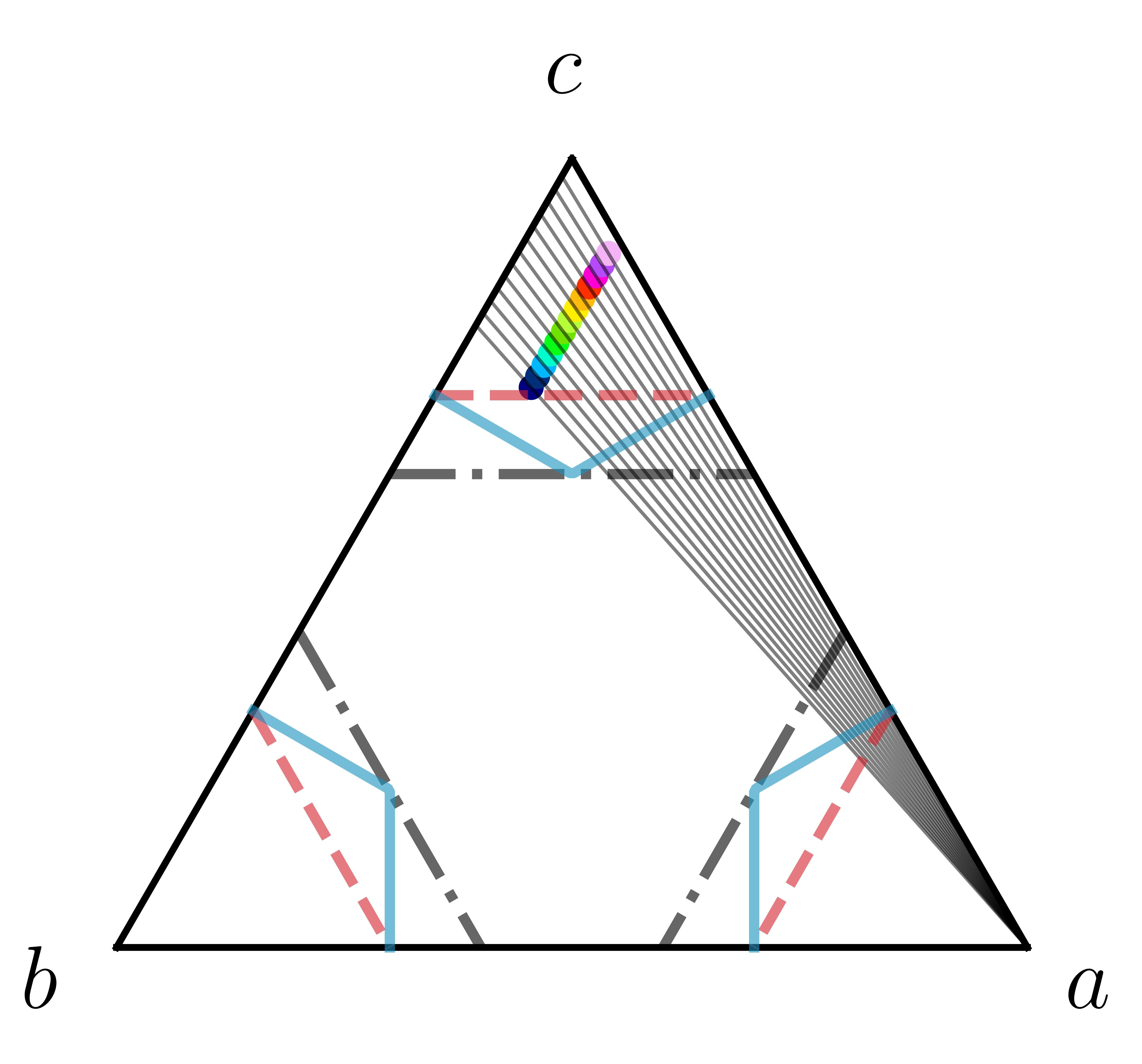}}
      \subfigure{\includegraphics[width=.45\columnwidth]{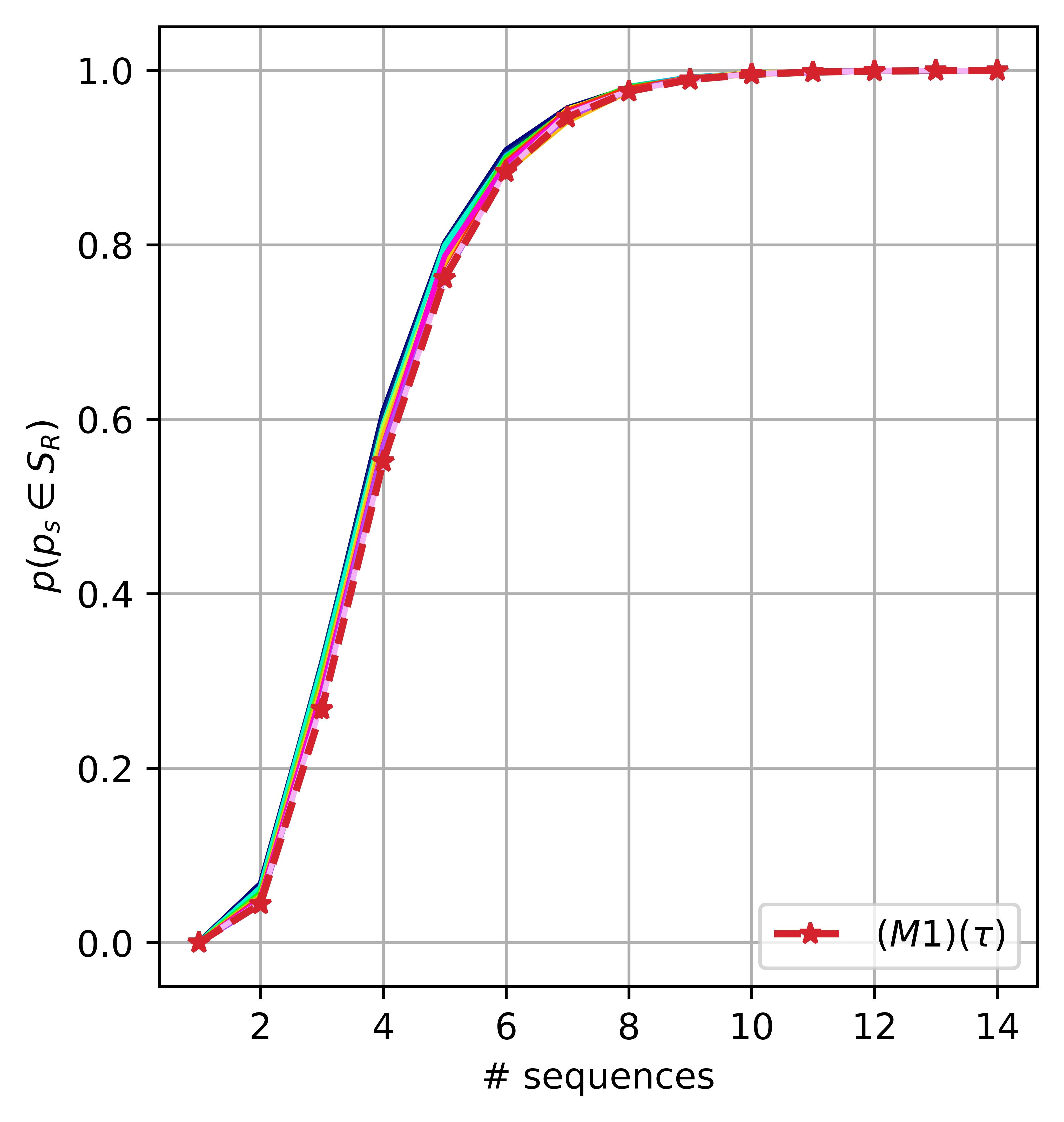}}
      \subfigure{\includegraphics[width=.45\columnwidth]{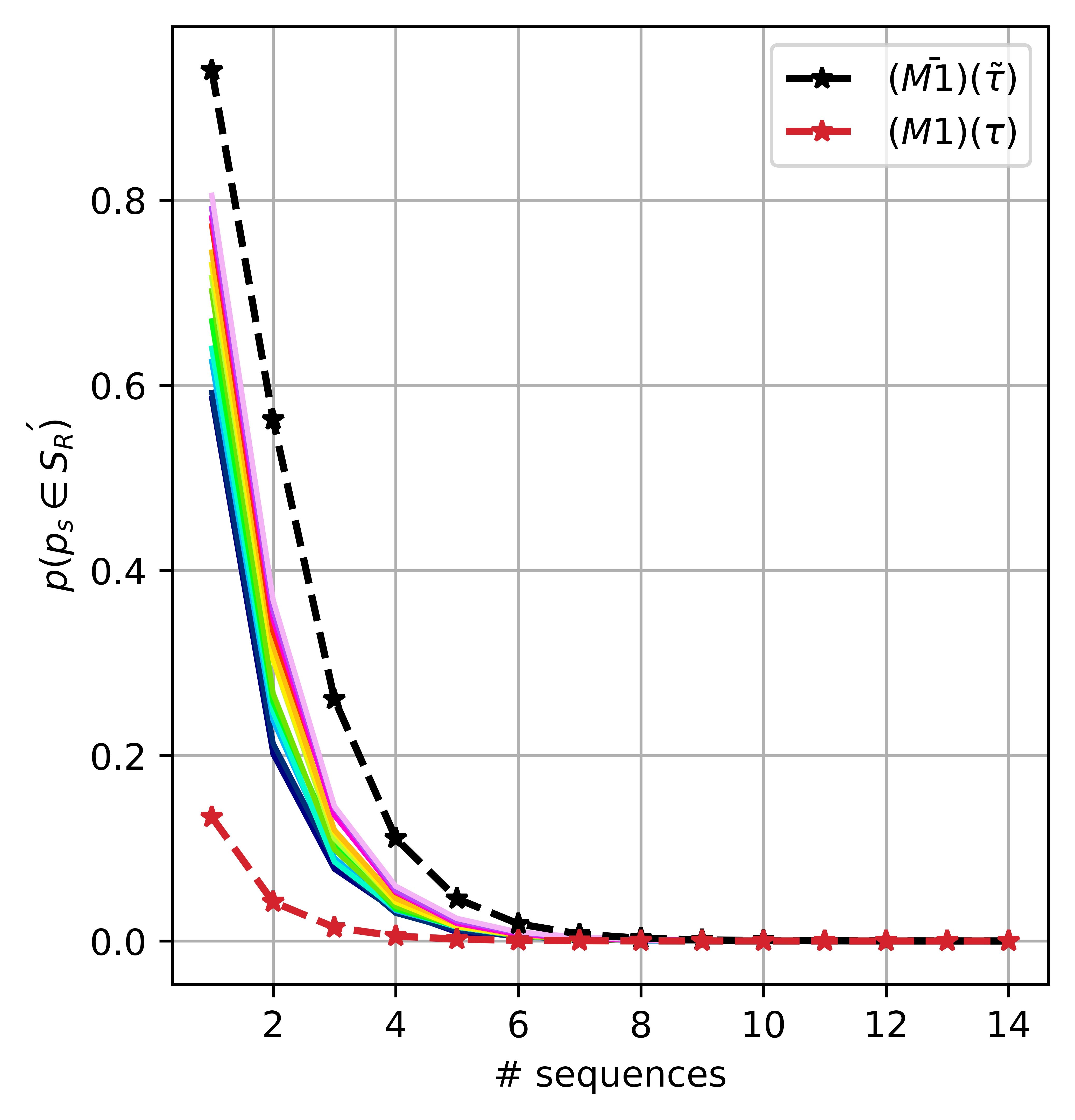}}\\
      \subfigure{\includegraphics[width=.45\columnwidth]{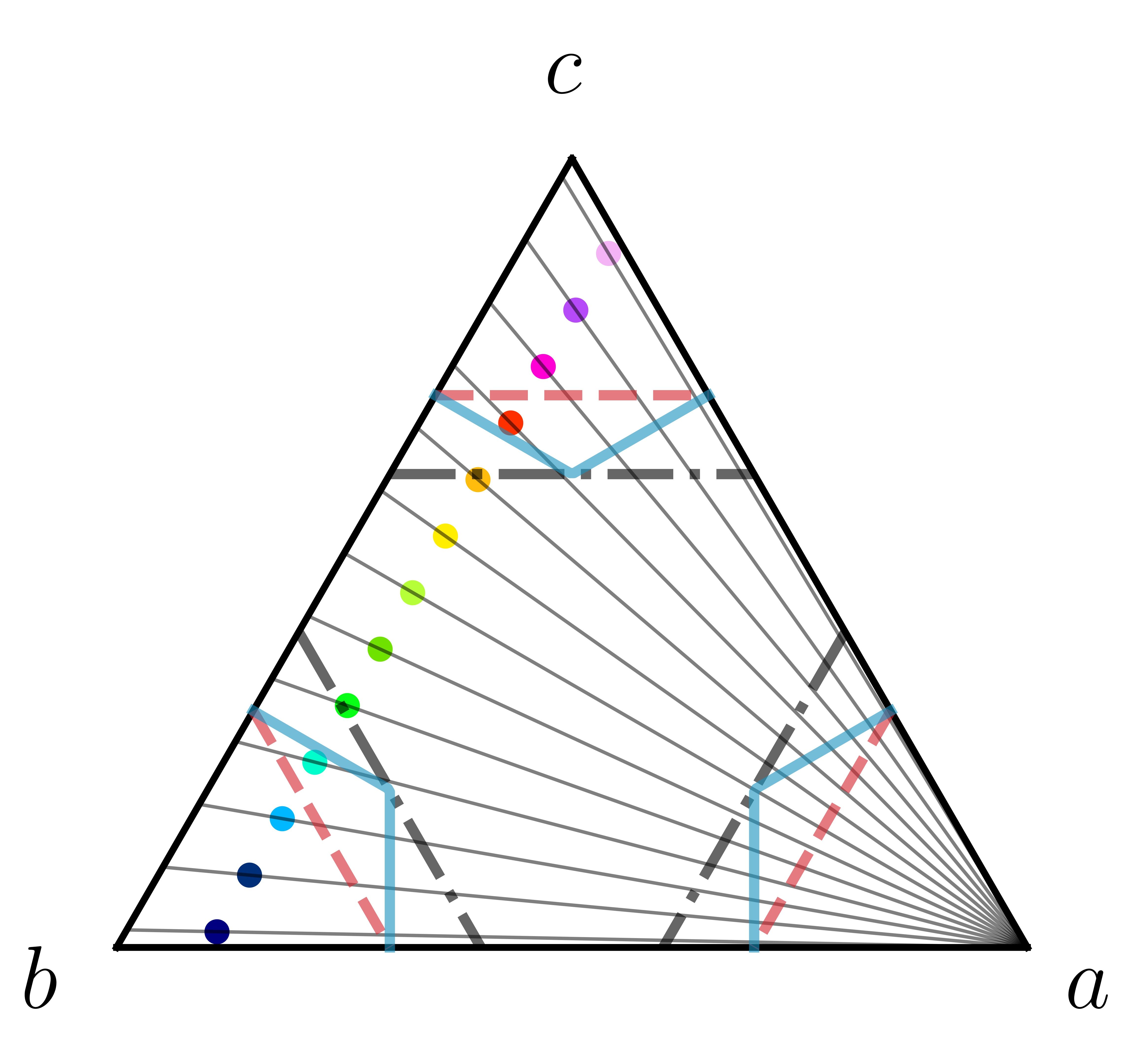}}
      \subfigure{\includegraphics[width=.45\columnwidth]{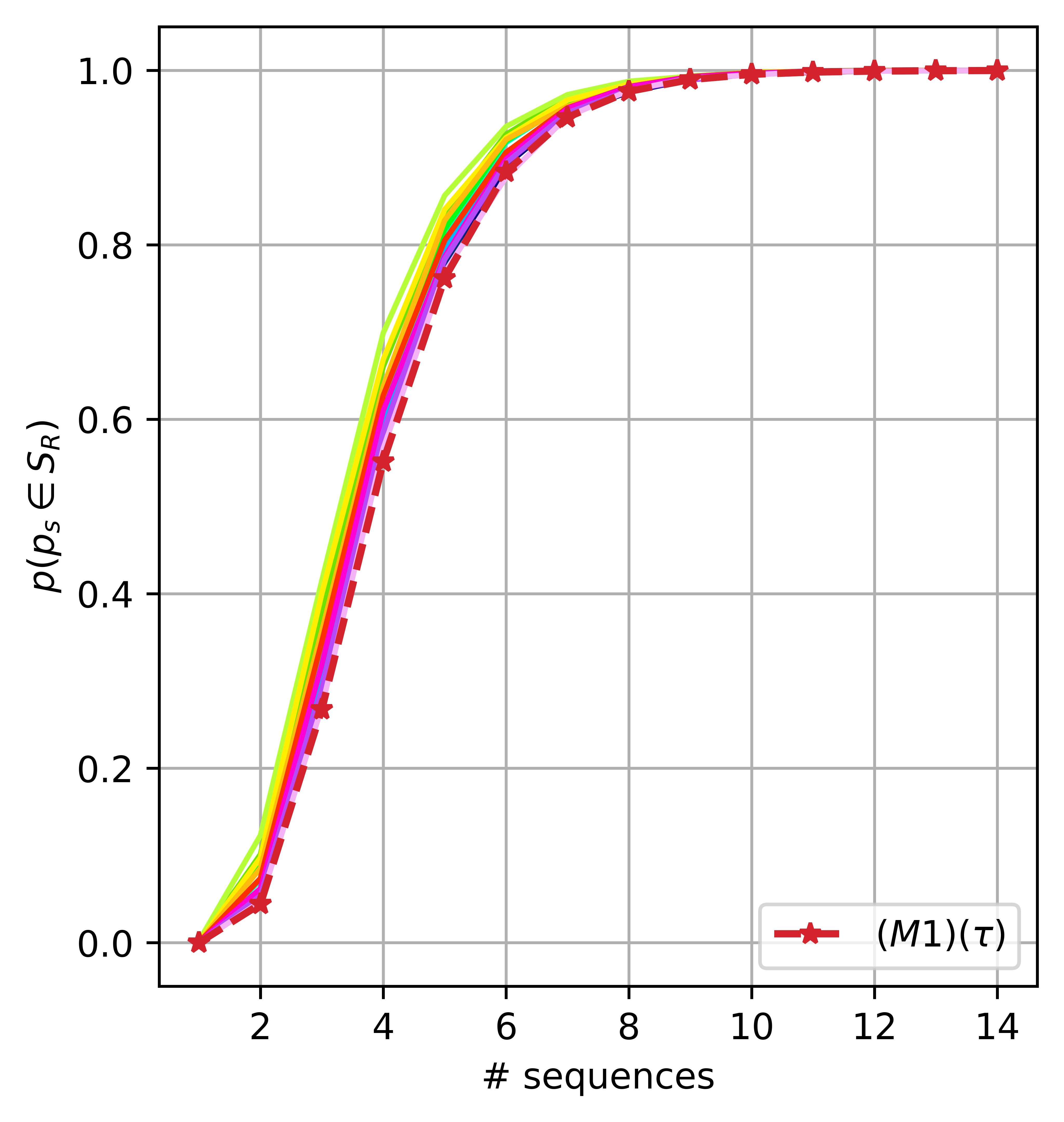}}
      \subfigure{\includegraphics[width=.45\columnwidth]{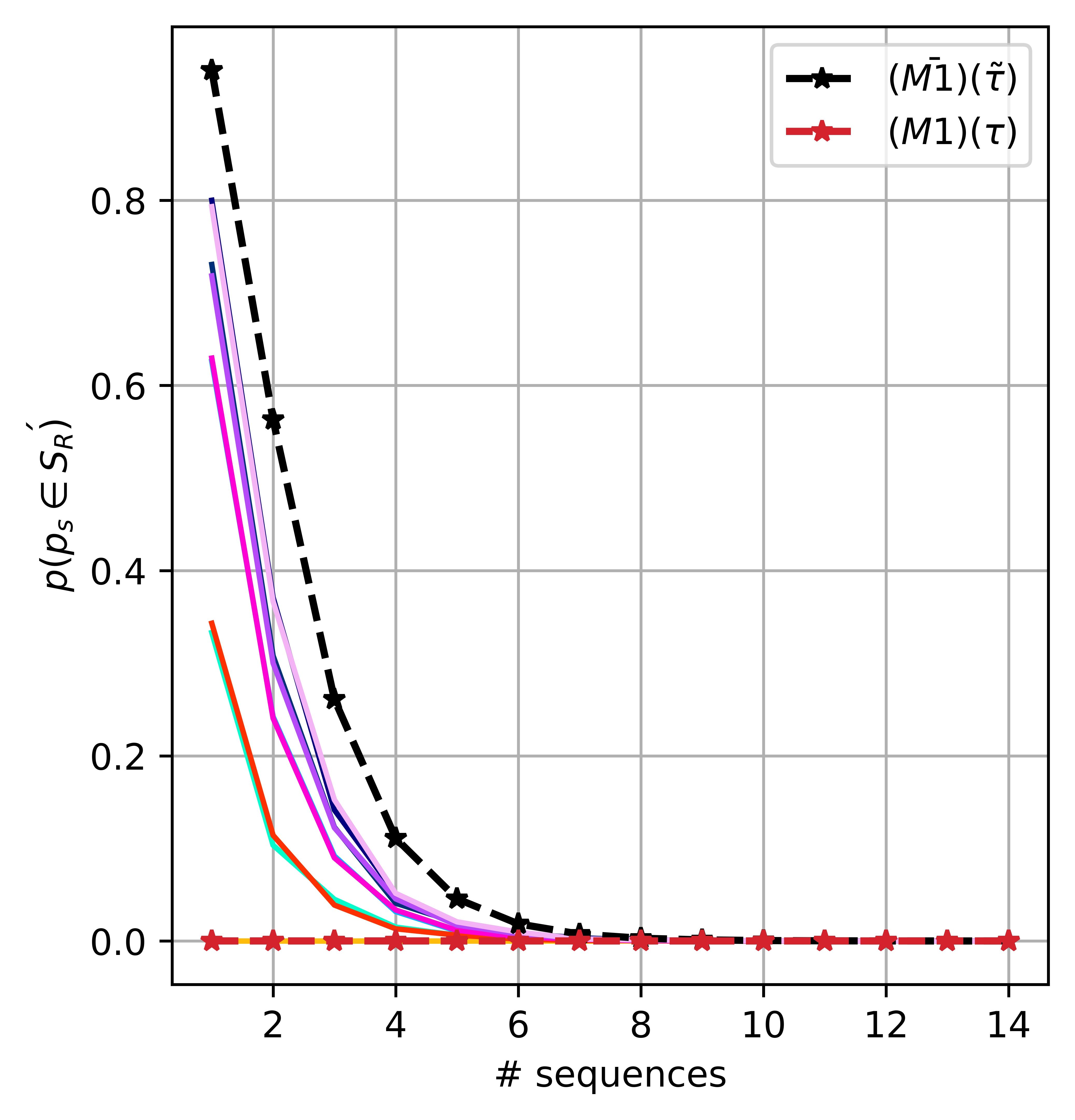}}\\
      \subfigure{\includegraphics[width=.45\columnwidth]{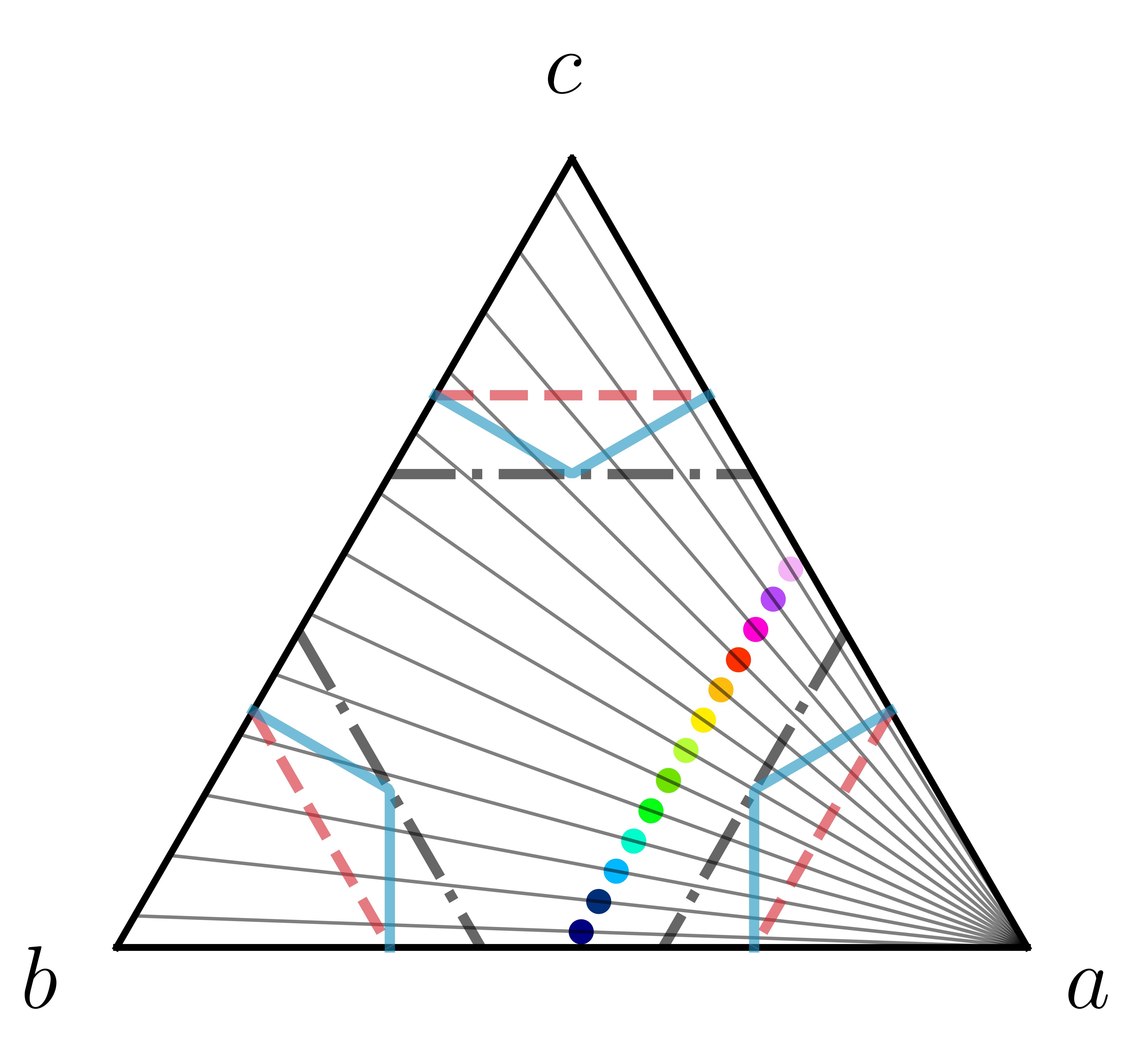}}
      \subfigure{\includegraphics[width=.45\columnwidth]{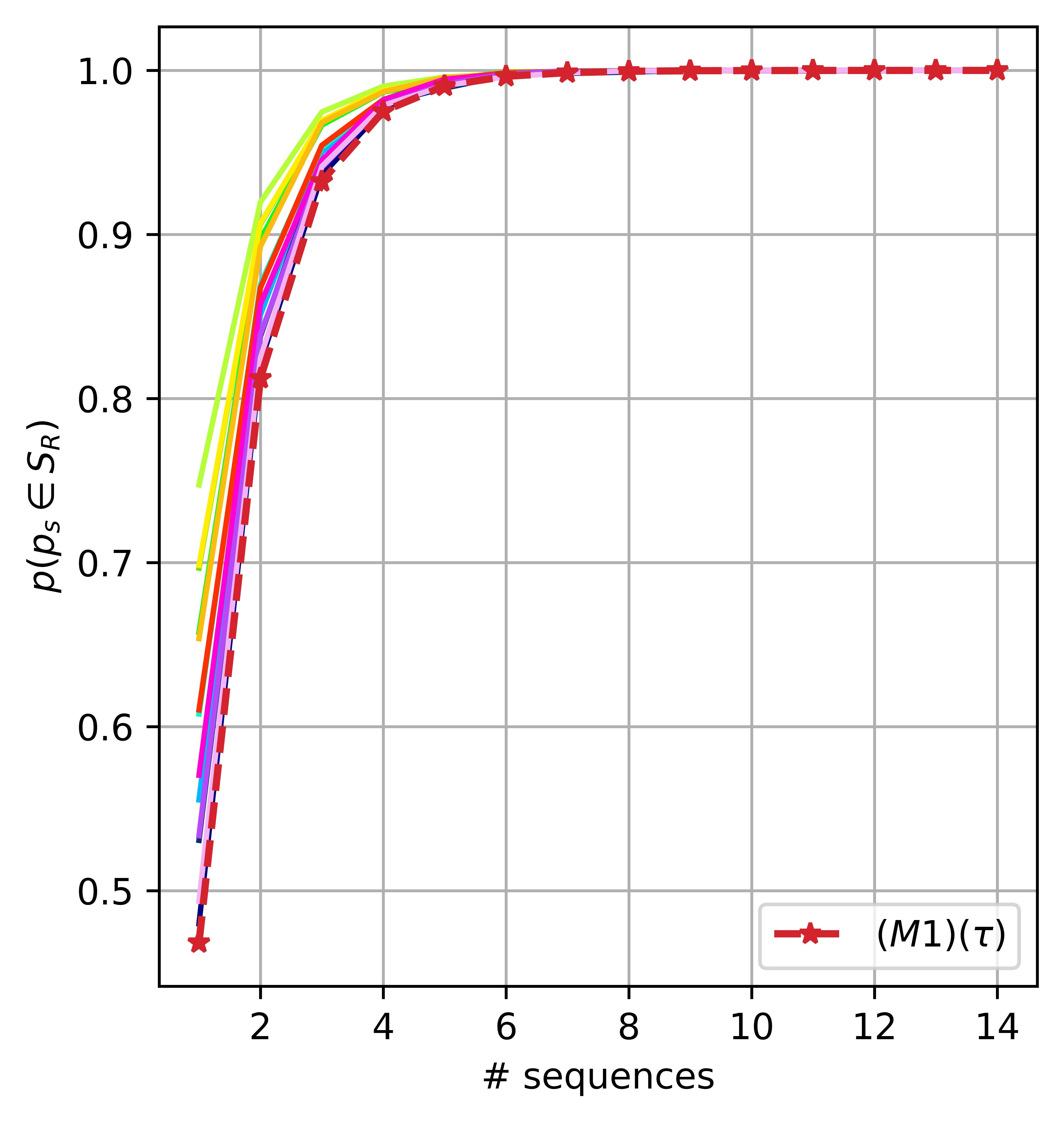}}
      \subfigure{\includegraphics[width=.45\columnwidth]{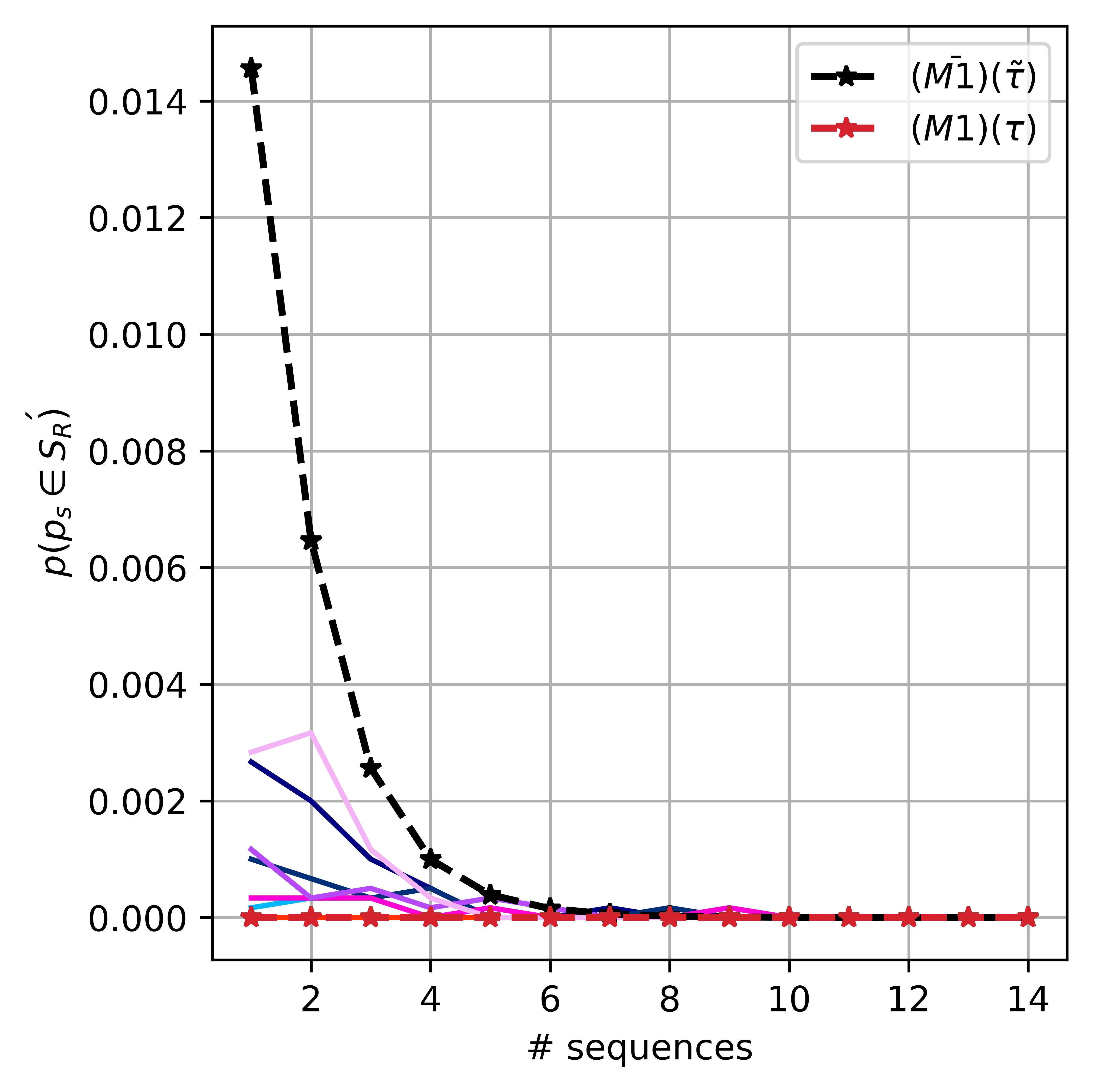}}\\
      \vspace{-0.4cm}
      \caption{In this figure we compare the probability of correct selection (reaching the correct stopping region) by \# of sequences and incorrect selection in $\Delta_3$ of (MP) (blue) with (M1) (red) and ($\bar{\text{M1}}$) (black). $\text{1}^\text{st}$ column represents the prior points of the recursive classification with different colors in $\Delta_3$. $\text{2}^\text{nd}$ column represents the probability of posterior at sequence $s$ lying on the correct region for stopping. $\text{3}^\text{rd}$ column represents the probability of incorrect stopping by sequences. Observe that as described in  Proposition \ref{prop:perf_guarantees} probability of correct decision is always above (M1) and probability of error is sandwiched with lower (M1) and upper ($\bar{\text{M1}}$) curves.}
    \label{fig:TP_FP_plots}
    \vspace{-0.3cm}
\end{figure*}
%
%
% TODO: make the aspect ratio equal for the figures.
\begin{figure*}[t]
\centering
\begin{minipage}{0.7\columnwidth}
		\centering
		\includegraphics[width=\columnwidth]{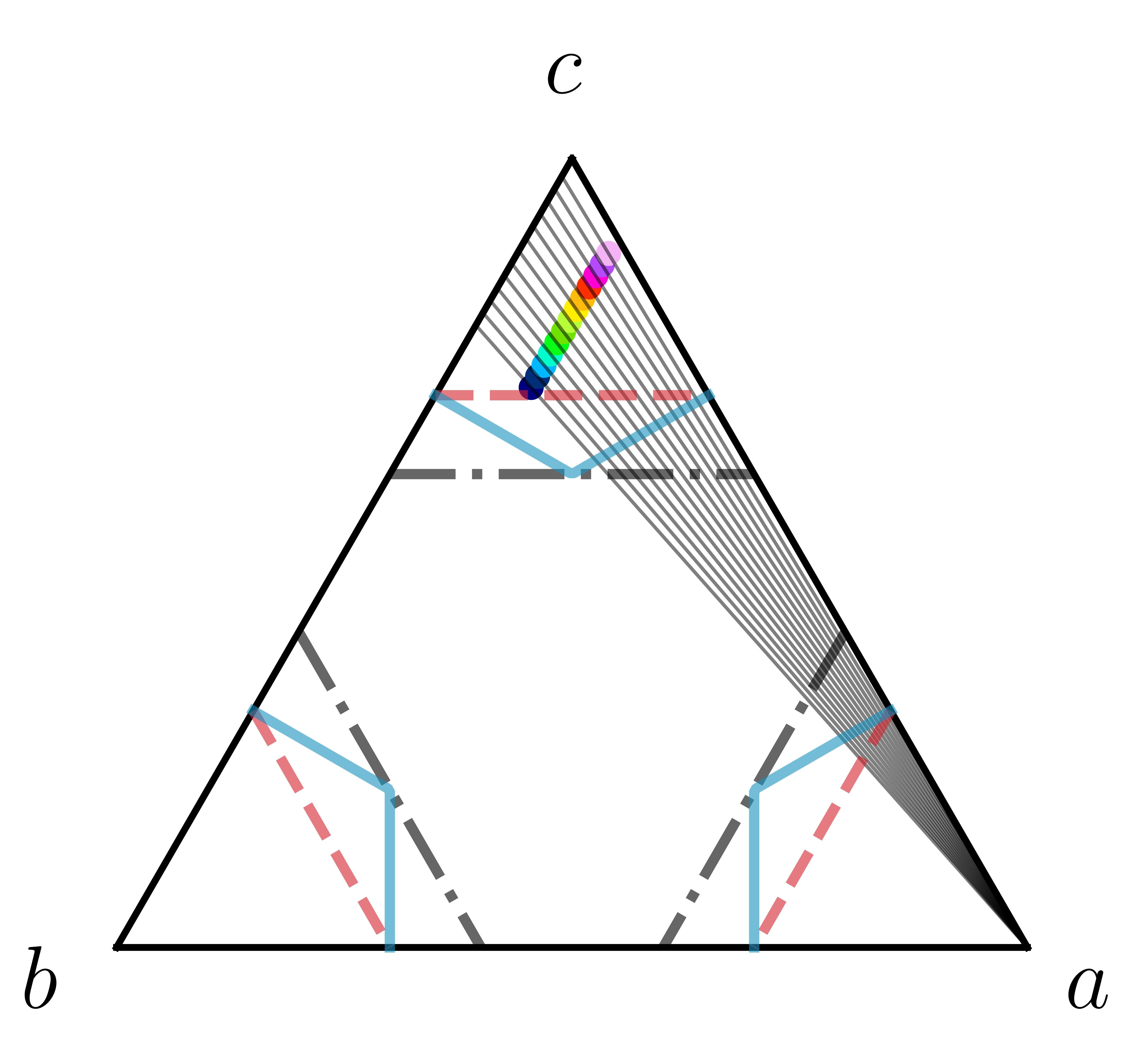}
\end{minipage}
\begin{minipage}{1.3\columnwidth}
		\centering
	   \subfigure{\includegraphics[width=.45\columnwidth]{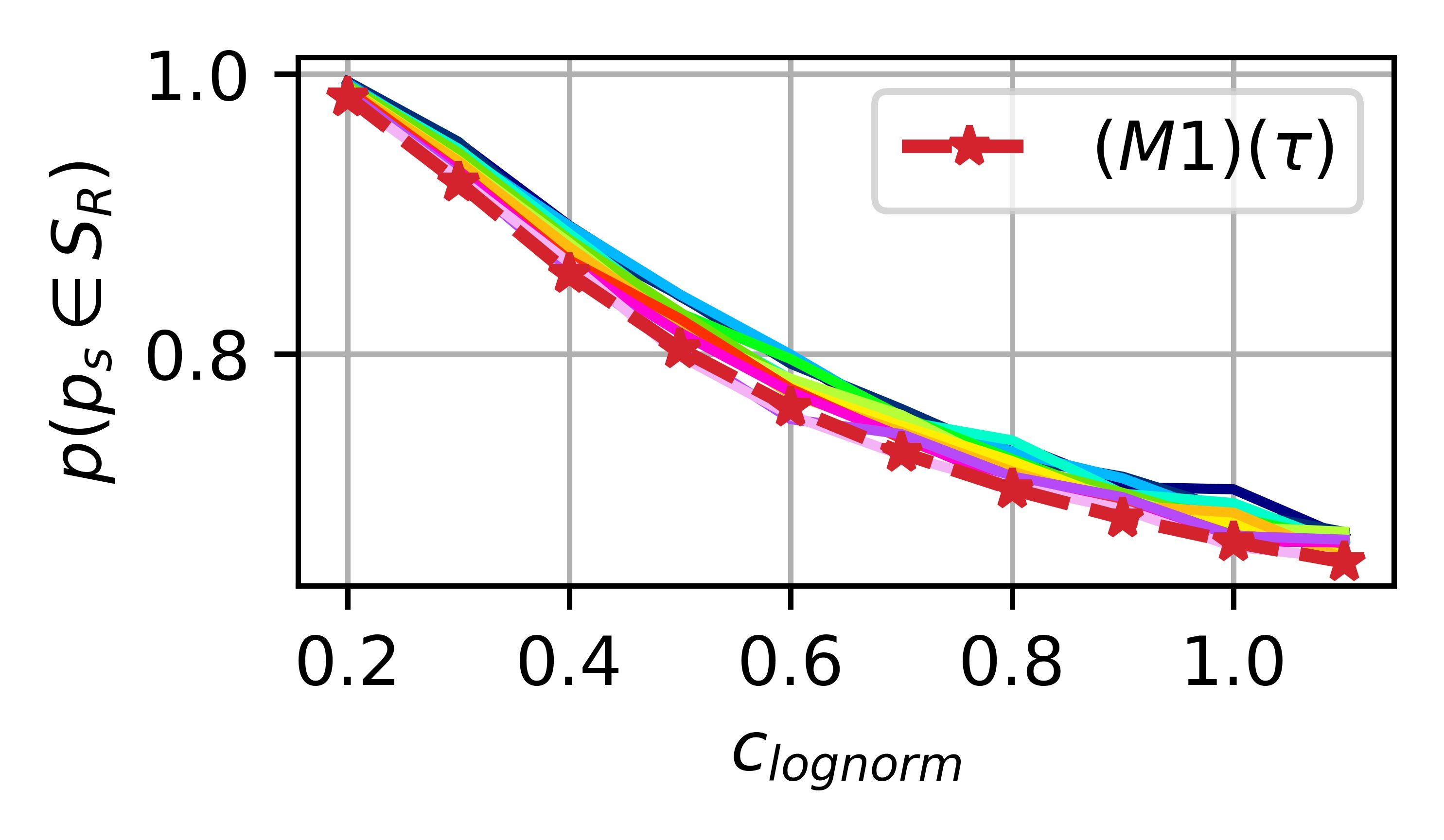}}
	   \subfigure{\includegraphics[width=.45\columnwidth]{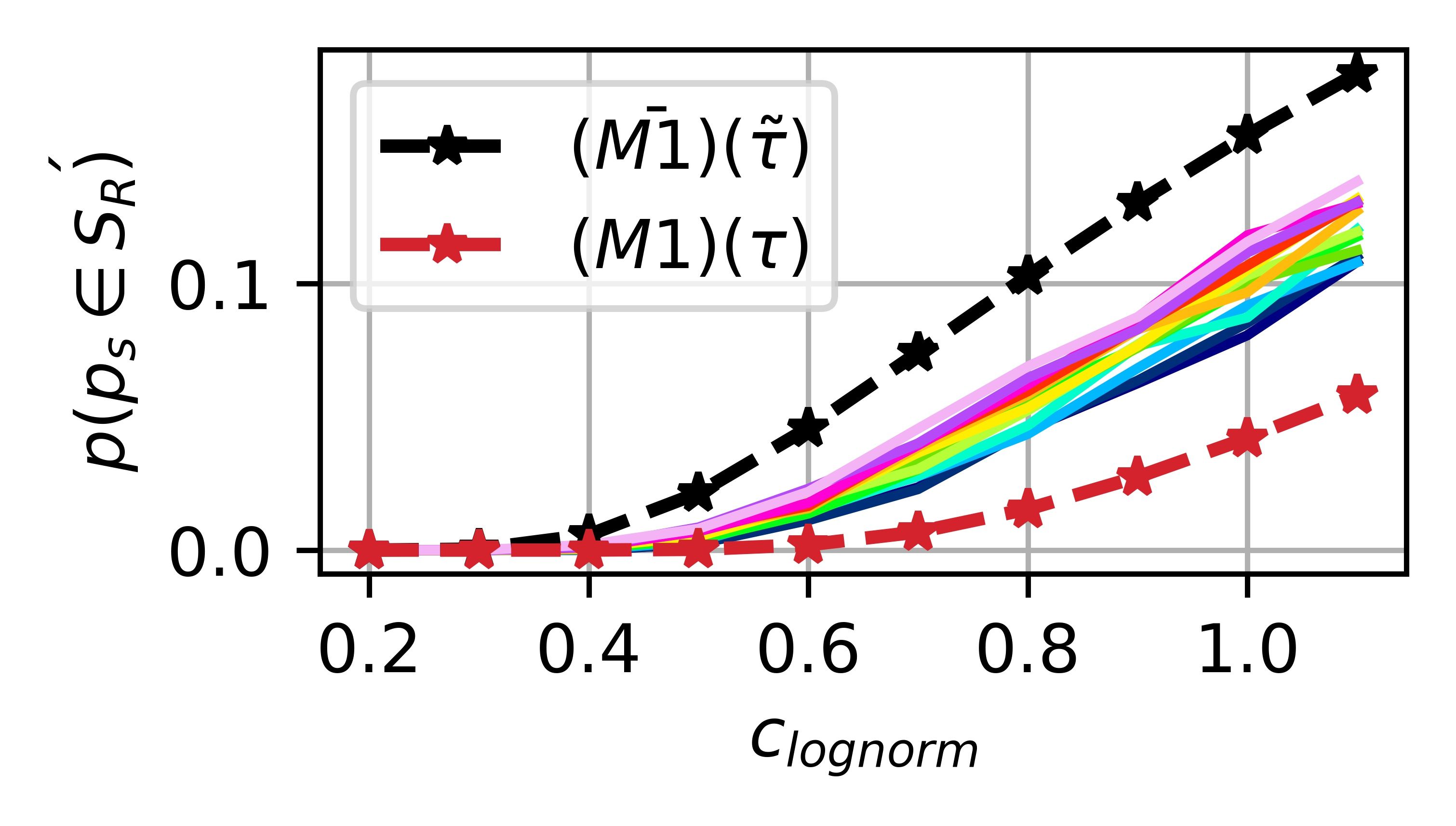}}\\
       \subfigure{\includegraphics[width=.45\columnwidth]{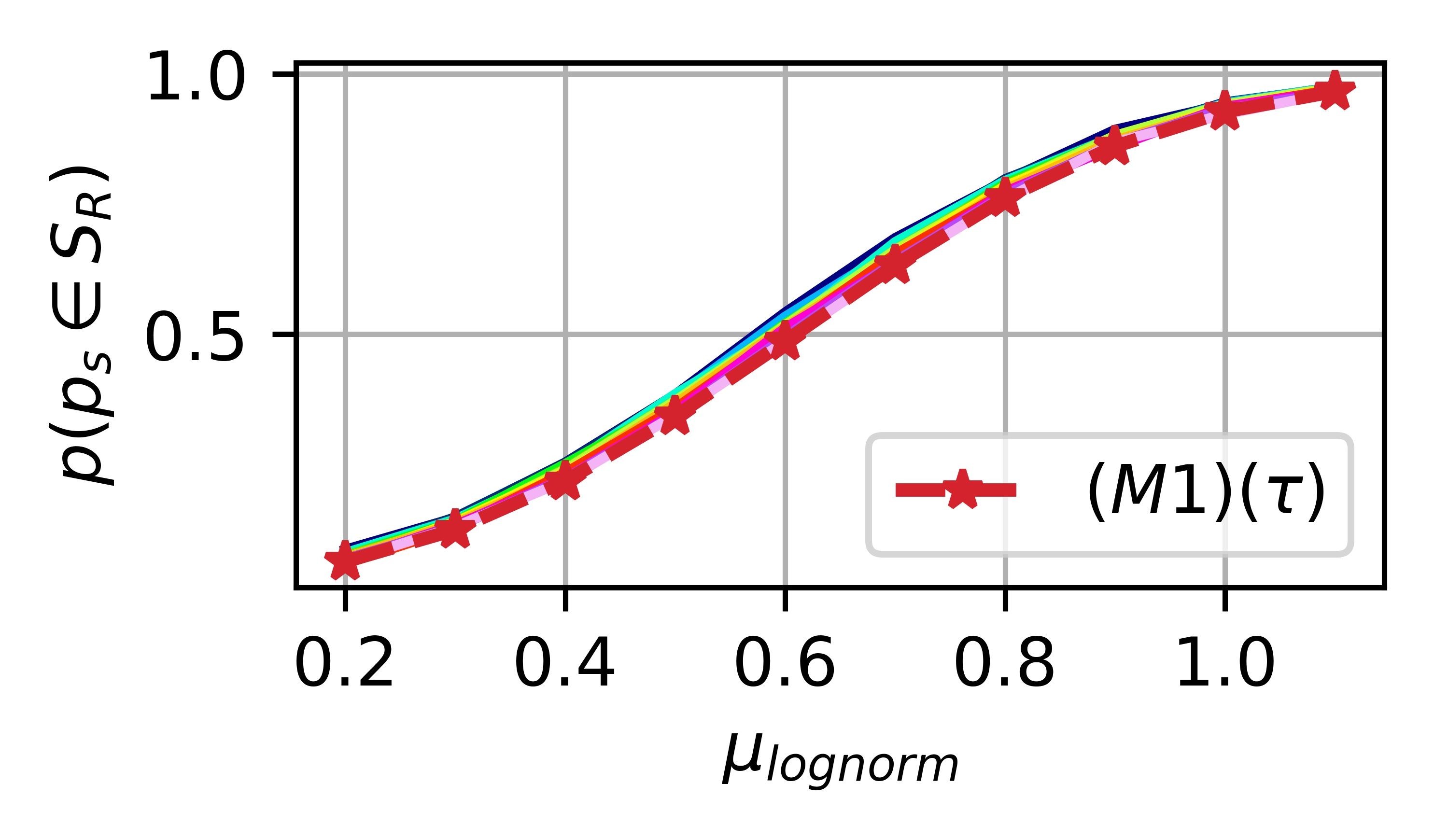}}
       \subfigure{\includegraphics[width=.45\columnwidth]{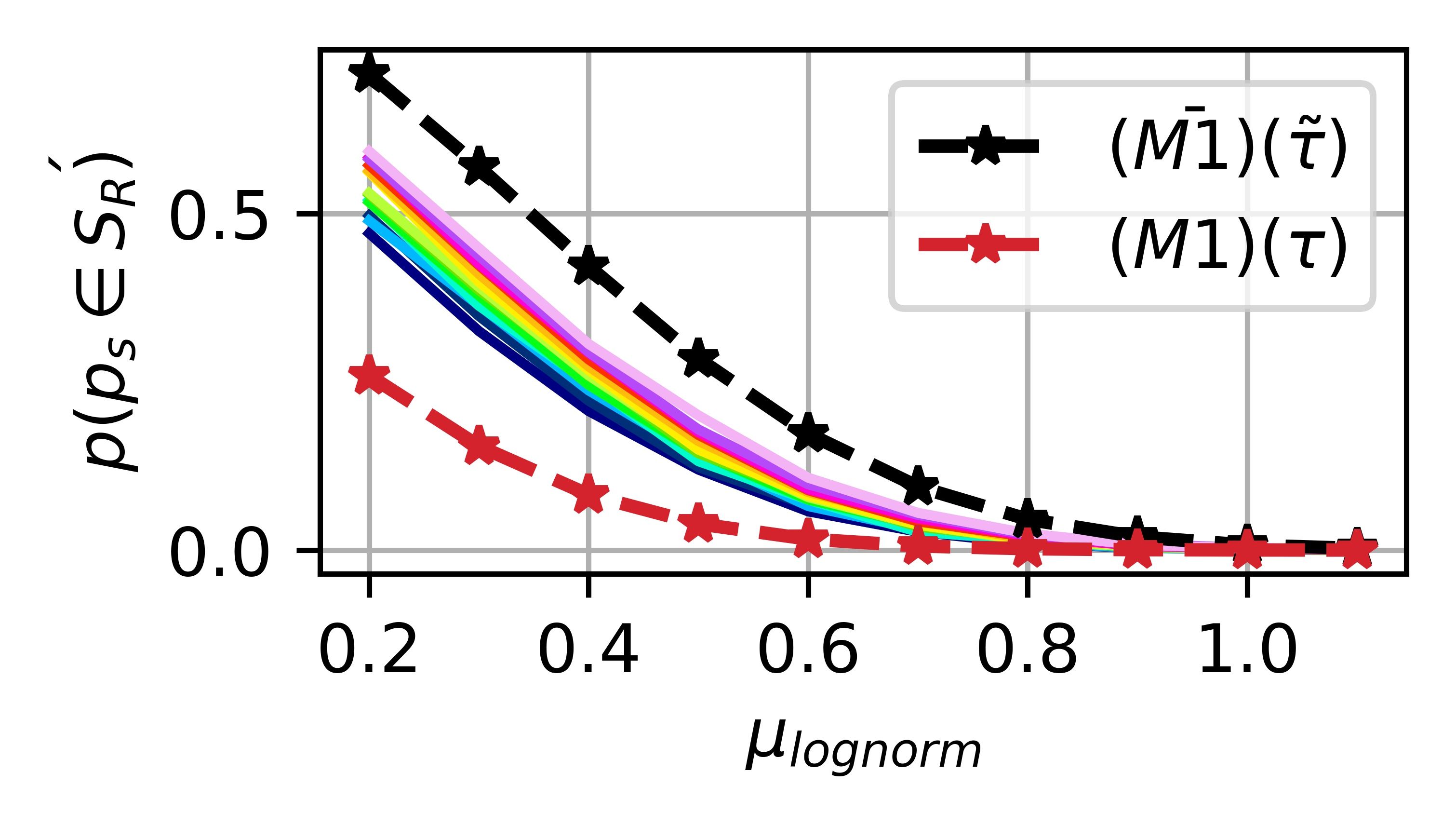}}\\
\end{minipage}
      \caption{In this figure we visualize the effects of the distribution variance and mean in recursive classification for the represented points in the figure. We consider $\varepsilon_+ \sim \text{lognorm}(\mu,c^2)$ and calculate the probabilities for a set sequence number $s=5$. We represent 4 figures to indicate the effects (on the right side of the $\Delta_3$ simplex with corners $a$, $b$ and $c$). Top row represents the effects of the standard deviation of the exponentiated Gaussian distribution on probability of correct selection and incorrect selection for left and right respectively where $\mu = 0.8$. Whereas bottom row visualizes the effects of the mean where $c = 0.6$. It is observed as mean increases (since the evidence gets higher values) probability of correct selection increases where error chance decreases. On the other hand results are the opposite as expected for the standard deviation increments.}
    \label{fig:TP_FP_dist_effect}
    \vspace{-0.3cm}
\end{figure*}
\clearpage

\clearpage

\newpage

\subsubsection{BCI Typing Supplementary}
\label{sec:bci_supplementary}
In experiments section (Sec.\ref{sec:Experiments}), we report time required for a subject to successfully complete a 100 letter typing scenario only for the conventional method (M1). We use the following computation to report required number of sequences to cpmplete the task;
\begin{equation*}
\begin{split}
    &\text{acc,seq from user data}; \ \text{rem} = 100, \ \text{\#seq} = 0 \\
    &\textbf{while} \text{ rem}>0 \textbf{ do:}\\ &\hspace{2mm}\text{rem}\gets\text{rem} - \text{ceil}(\text{rem}\times\text{acc}/100)\\
    &\hspace{2mm}\text{\#seq} \gets\text{\#seq} +\text{rem} \times \text{seq}\\
    &\textbf{end}\\
    &\textbf{return } \text{\#seq}
\end{split}
\end{equation*}
In Table\ref{tab:rsvp_time_table} we report the performance values for all the methods that are compared to (MP). We specifically report the number sequences required for each user to successfully complete the task in addition to the number of sequences averaged across all participants. Moreover, we report  average typing accuracy and sequence required to type 1 letter (correct or incorrect) for each method averaged across users. It is apparent (MP) allows faster typing as reported in the experiments by an accuracy loss of $5\%$ on average. (M4) seems to outperform (MP) in the uniform case. However, (M4) results in frequent incorrect decisions that require corrections which irritates the target population to use such BCI systems. Therefore, such methods similar to (M4) may overall perform worse due to fatigue effects. 

{
\begin{table}[!h]
\footnotesize
\begin{center}
\begin{tabular}{c c |c | c c c c c }
%\multicolumn{5}{c}{Stopping Region: $\mathcal{R}_S:= \lbrace p| \mathcal{S}_O(p) \rbrace$} \\
& &Seq &\multicolumn{5}{c}{Sequence Difference from \# (MP)}\\
LM &Perf &\# (MP) &(M1) &(M2) &(M3) &(M4) &(M5)\\
\hline
\multirow{6}{*}{\rotatebox[origin=c]{90}{Uniform}} 
&67 &3628  &+327   &+306   &+143  &-183  &+5687\\
&72 &2290  &+295   &+279   &+142  &-130  &+5200\\
&76 &1627  &+197   &+186   &+85   &-59   &+3391\\ 
&81 &1180  &+186   &+173   &+93   &+18   &+2524\\
&84 & 947  &+129   &+124   &+677  &+14   &+2173\\
&87 & 739  &+126   &+122   &+71   &+35   &+1921\\
\cline{2-8}
&\text{avg} &1735 &+210 &+198  &+101 &-50 &+3482 \\
\cline{2-8}
&\text{acc}  &90\% &94\%  &94\% &88\% &64\% &99\% \\
&\text{E(seq)}  &15.44 &18.07  &17.96 &16.05 &10.70 &51.95 \\
\hline
\multirow{6}{*}{\rotatebox[origin=c]{90}{Language Model}}
&67 &3012  &+194  &+202  &+11  &+110 &+3252\\
&72 &2053  &+233  &+224  &+78  &+21  &+3208\\
&76 &1570 &+147  &+141  &+62  &-43   &+2743\\
&81 &1171 &+119  &+115  &+43  &-22   &+2276\\
&84 &923 &+104  &+105  &+51  &+43  &+2018\\
&87 &754 &+95   &+89   &+49  &+58   &+1865\\
\cline{2-8}
&\text{avg} &1580 &+148 &+146  &+490 &+27 &+2560 \\
\cline{2-8}
&\text{acc}  &85\% &90\%  &90\% &85\% &56\% &98\% \\
&\text{E(seq)}  &13.08 &15.68  &15.57 &13.90 &8.46 &41.04 \\
\hline
\end{tabular}
\end{center}
\caption{Supplementary for BCI experiments Fig.\ref{fig:factor_plots_synth_1}, Fig.\ref{fig:factor_plots_synth_2}. For all experiments the confidence threshold is set to $\tau=.85$ for (M1) and the constants for the other objectives are calculated accordingly. We report number of sequences spent for (MP) for different performing users and corresponding differences with other methods. \textit{avg} represents the average sequence values for each method conditioned on LM respectively . \textit{acc} and \textit{\text{E}(seq)} represent accuracy in typing and average number of sequences spent for 1 letter respectively. For the language model case, based on prior location for \textit{"O"} after \textit{"IT\_"}, the difference between (M4) vanishes.}
\label{tab:rsvp_time_table}
\end{table}}

\subsection{Proofs}
\label{Sec:Proofs}
\begin{proof}[\textbf{Proof Proposition \ref{prop:SimplexVSpace}}]

\textit{Addition:} Given $p,q\in\Delta_n$ the addition operation is defined as the following;
    \begin{equation*}
    \begin{split}
        p \oplus q = \frac{[p_1q_1,p_2q_2,\cdots,p_nq_n]}{\sum_i p_iq_i}
        \end{split}
    \end{equation*}

\textit{Multiplication with scalar:} Given $p\in\Delta_n,\lambda\in\mathbb{R}$ the multiplication with scalar is defined as the following;
    \begin{equation*}
        \begin{split}
            p\otimes \lambda = \frac{[p_1^\lambda,p_2^\lambda,\cdots,p_n^\lambda]}{\sum_i p_i^\lambda}
        \end{split}
    \end{equation*}

Let $p,q,r \in \Delta_n$; $p\oplus q = [[p\oplus q]_k \lvert k\in\lbrace ,1 ,2,\cdots, n \rbrace] \in \mathbb{R}^n$ and $p\otimes q = [[p\otimes q]_k \lvert k\in \lbrace ,1 ,2,\cdots, n \rbrace] \in \mathbb{R}^n$

\textit{Additive Closure:} $p_k>0, q_k>0 \implies [p\oplus q]_k>0$, $\sum_k [p\oplus q]_k = \sum_k p_kq_k / \sum_i p_iq_i = 1 \implies p\oplus q \in \Delta_n$

\textit{Scalar Closure:} $p_k>0, \lambda\in\mathbb{R} \implies p_k^\lambda >0 \implies [p\oplus \lambda]_k>0$, $\sum_k [p\otimes q]_k = \sum_k p_k^\lambda / \sum_i p_i^\lambda = 1 \implies p\otimes \lambda \in \Delta_n$

\textit{Commutativity:} scalar multiplication is commutative $\implies p_kq_k = q_kp_k \implies p\oplus q = q\oplus p$

\textit{Additive Associativity:} scalar multiplication is associative $\implies (p_kq_k)r_k = p_k(q_kr_k)$. $[p\oplus q]_i = p_iq_i/\sum_kp_kq_k, [(p\oplus q)\oplus r]_i = p_iq_ir_i / \sum_k p_kq_kr_k \implies (p\oplus q)\oplus r = p \oplus(q \oplus r)$

\textit{Zero Element:} $u_n \in \Delta_n$ as shown above.

\textit{Multiplicative $1$:} $[p\otimes 1]_k = p_k / \sum_i p_k = p_k \implies p\otimes 1 = p$

The following properties can be shown using the definitions above easily and hence omitted; additive inverses, scalar multiplication associativity, distributivity across vector addition, distributivity across scalar addition.
\end{proof}
\hrule
\begin{proof}[\textbf{ Proof Proposition \ref{prop:max_uncertainty}}]

$S_\tau = \lbrace p|\max_i p_i = \tau \rbrace \implies \max_{p\in S_\tau} H(p)= H(v_n(\tau))$

Let $p\in S_\tau$. WLOG pick $p_1=\tau$ where $|p|=n$;
\begin{equation*}
    \begin{split}
        \max_p H(p) = \max_{p_i\in p}-p_1\log{\left(p_1\right)}-{\displaystyle{\sum_{i \in \lbrace 2,\cdots ,n \rbrace}} p_i\log{\left(p_i\right)}}\\ \text{ from principle of maximum entropy}\\
        = -\tau\log{\left(\tau\right)}-{\displaystyle{\sum_{i \in \lbrace 2,\cdots ,n \rbrace}} \displaystyle{\frac{1-\tau}{n-1}}\log{\left(\displaystyle{\frac{1-\tau}{n-1}}\right)}} \\      =-\tau\log{\left(\tau\right)}-(1-\tau) \log{\left(\displaystyle{\frac{1-\tau}{n-1}}\right)}\\ =H(v_n(\tau))
    \end{split}
\end{equation*}
$S_\tau = \lbrace p|H(p) = H(v_n(\tau)) \rbrace \implies \max_{i, p \in C_\tau} p_i=\tau$

Let $p\in C_\tau$. WLOG pick $i=1$,  $p\in C_\tau \implies p_1\log p_1 = -H(v_n(\tau)) - \sum_{i\in \lbrace 2,\cdots,n \rbrace} p_i\log p_i$;
\begin{equation*}
    \begin{split}
    \left.\frac{d \  xlog(x)}{d x}\right\lvert_{x\in(1/e,\infty)} > 0  \\
    \implies\max x = \arg\max_x xlog(x) \ \forall x\in(1/e,\infty) \\
    \implies \max_{p\in C_\tau} p_1 =\arg_{p_1}\max_{p\in C_\tau} p_1 \log p_1 \\
    = \arg_{p_1}\max_{p\in C_\tau} -H(v_n(\tau)) - \sum_{i\in \lbrace 2,\cdots,n \rbrace} p_i\log p_i \\ 
    =\arg_{p_1}\max_{p\in C_\tau} - \sum_{i\in \lbrace 2,\cdots,n \rbrace} p_i\log p_i \\
    \text{ where } \sum_{i\in \lbrace 2,\cdots,n \rbrace} p_i = 1-p_1
    \end{split}
\end{equation*}
Using the identity $\hat{x} = \arg\max_x -log(x)^T x = [1/e,\cdots, 1/e]^T$ then project to $(1-p_1)$-unit $\ell_1$ norm ball again from principle of maximum entropy we observe $p_{1:n} = [\frac{1-p_1}{n-1},\cdots, \frac{1-p_1}{n-1}]\implies p =[p_1, \frac{1-p_1}{n-1},\cdots, \frac{1-p_1}{n-1}]= v_n(p_1)$. Since $H(p)=H(v_n(p_1))=H(v_n(\tau))$, $p_1 = \tau$.
\end{proof}
\hrule
\begin{proof}[\textbf{Proof Observation\ref{obs:intersection_entropy_conf}}] Derivation is trivial given Proposition \ref{prop:max_uncertainty}. Omitted.
\end{proof}
\hrule
\begin{proof}[\textbf{Proof Observarion \ref{obs:increasing_feasible_set}}]

Let $p\in\Delta_n$, WLOG $p_1 = max_i p_i$;
\begin{equation}
\begin{split}
    H(p) =\\
    -\sum_i p_i\log(p_i) = -p_1\log(p_1) - \sum_{i\in\lbrace 2,\cdots,n\rbrace} p_i\log(p_i) \\
    \leq -p_1\log(p_1)-(1-p_1)\log{\left(\displaystyle{\frac{1-p_1}{n-1}}\right)} \\
    -H(p) \geq \ \ p_1\log(p_1)+(1-p_1)\log{\left(\displaystyle{\frac{1-p_1}{n-1}}\right)}
\end{split}
\end{equation}
Let $f(u,a)=u\log(u)+(1-u)\log{\left(\displaystyle{\frac{1-u}{a-1}}\right)}$ for $u\in[1/a,1]$ and $a \geq 1$; $f(u,a)$ is strictly convex and monotonically increasing for $\forall u\in[1/a,1]$ where $a \geq 1$. If $p_1 \geq \tau$ then $f(p_1,n) \geq f(\tau,n)$. This concludes that;
\begin{equation}
    -H(p) \geq \tau\log(\tau) + (1-\tau)\log{\left(\displaystyle{\frac{1-\tau}{n-1}}\right)} = -\tau'
\end{equation}
Therefore, $\forall p_1 \geq \tau\geq 1/n \implies H(p) \leq \tau'$. Hence $\forall p \in S_1 \implies p\in S_2$ which concludes $S_1 \subseteq S_2$.

\noindent
Moreover, given $n=5$ where $\tau' = 2.16$, therefore corresponding $\tau = 0.5$. Pick a distribution with probabilities $p = \lbrace .4,.2,.2,0,0\rbrace$. Obviously, $p\notin S_1$. However, $H(p) =1.5219< \tau'$ therefore $p\in S_2$ which shows $S_2\not\subseteq S_1$. Bu the counterexample  $S_1 \subseteq S_2$ reduces to $S_1 \subset S_2$.
\end{proof}
\hrule
\begin{proof}[\textbf{Proof Observarion \ref{obs:weakest_entropy_guarantee}}]
The edge case for equi-entropy contours intersecting with the edges of the probability domain is WLOG $p=[0.5,0.5,0,\cdots,0]$ hence $H(p)=1$. $H(v_n(\tau))<1\implies$ the equi entropy contour $\mathcal{S}_\tau$ intersects with the borders of the simplex. As explained over truncated distributions in \cite{ho2010interplay}, we use the confidence line definition $C_\tau = \lbrace p|\max_i p_i = \tau \rbrace$ and use the identity $\arg\max_{p\in C_\tau} H(p)= w_n(\tau)$. Equivalently, for $w_n(\tilde{\tau})$ to hold the entropy condition, one can write;
\begin{equation*}
    \begin{split}
    -\tilde{\tau}\log_2(\tilde{\tau}) -(1-\tilde{\tau})\log_2(1-\tilde{\tau}) = H(v_n(\tau))
    % \tilde{\tau} = f^{-1}(H(v_n(\tau)))\\
    % \text{ where } f(x) = -x\log_2(x) -(1-x)\log_2(1-x)
    \end{split}
\end{equation*}
\end{proof}
\hrule
\begin{proof}[\textbf{Proof Lemma \ref{lemma:collinearity_of_post}}]

Pick $a_1,a_2,a_n \in \mathcal{A}$ being $1^\text{st},2^\text{nd}, n^\text{th}$ class respectively and $a_1$ being queried $\phi(a_1)$ with initial probability distribution over the entire state space $p(\sigma) = [p_{a_1}= p_1, \cdots,p_n]$ yielding the evidence $\varepsilon$. Observe the following relation between the posterior and prior using the label assignment for the queried candidate $\ell$ where for a query $\phi(a_c)$, and a candidate $a_{c'}, a_{c'} = a_c \implies \ell = 1 $;
\begin{equation*}
    \begin{split}
    p(\sigma\lvert \varepsilon,\phi(a_1)) = p(\sigma)\oplus p(\varepsilon\lvert \sigma, \phi(a_1)) \\
    = p(\sigma)\oplus [p(\varepsilon\lvert a_1,\phi(a_1)), p(\varepsilon\lvert a_2,\phi(a_1)), \cdots, p(\varepsilon\lvert a_n,\phi(a_1))] \\
    = [p(\varepsilon\lvert \ell=1), p(\varepsilon\lvert  \ell=0, \cdots,p(\varepsilon\lvert  \ell=0)]
    \end{split}
\end{equation*}
The following vectors map to the same point on simplex;  $[p(\varepsilon\lvert \ell=1), p(\varepsilon\lvert \ell=0), \cdots, p(\varepsilon\lvert \ell=0)]\sim [\frac{p(\varepsilon\lvert \ell=1)}{p(\varepsilon\lvert \ell=0)}, 1, \cdots,1]$
\begin{equation*}
\begin{split}
    p(\sigma\lvert \varepsilon,\phi(a)) = p(\sigma)\oplus p(\varepsilon\lvert \sigma, \phi(a)) \\
    = p(\sigma)\oplus [\frac{p(\varepsilon\lvert \ell=1)}{p(\varepsilon\lvert \ell=0)}, 1, \cdots,1] \\ 
     = p(\sigma)\oplus [k, 1, \cdots,1] \ \text{where } k=\frac{p(\varepsilon\lvert \ell=1)}{p(\varepsilon\lvert \ell=0)}\in\mathbb{R}^+ \\
     \propto [p_1\times k , p_2, \cdots,p_n]
\end{split}
\end{equation*}
First we write the posterior using the previous equation;
\begin{equation*}
\begin{split}
    p(\sigma\lvert \varepsilon,\phi(a_1)) 
= [\frac{p_1\times k}{m},\frac{p_2}{m},\cdots,\frac{p_n}{m}]\in \Delta_N \\
    \text{where } m = p_1\times (k-1) + 1
    \end{split}
\end{equation*}
Trivially, $p(\sigma)$ and $[1,0,\cdots]$ form a line. To show collinearity, we show $p(\sigma\lvert \varepsilon,\phi(a_1))$ also lies on that line, in other the point should satisfy the line equation;
\begin{equation*}
(\text{line}): \frac{x_1 - 1}{p_1- 1} =\frac{x_2 }{p_2}= \cdots = \frac{x_n}{p_n}
\end{equation*}
If we insert $x_1 = p_1\times k /m$ and respective remaining probabilities;
\begin{equation*}
\frac{((p_1\times k)/m) - 1}{p_1- 1}\stackrel{?}{=}\frac{p_2/m}{p_2}= \cdots = \frac{p_n/m}{p_n} = \frac{1}{m}
\end{equation*}
We show collinearity with the following algebraic manipulation;
\begin{equation*}
\begin{split}
m = p_1(k-1) +1 \implies \frac{((p_1\times k)/m) - 1}{p_1- 1} \\ =\frac{p_1\times k -p_1\times k +p_1-1}{(p_1-1)\times m}
= \frac{1}{m}
\end{split}
\end{equation*}
\end{proof}
\hrule
\begin{proof}[\textbf{Proof Lemma \ref{lemma:projection_to_center}}]

Given $p(\sigma)=[p(a),p(b),\cdots p(z)]\in\Delta_n$ and WLOG for element $i=1$ denoting the special position the line $\ell c_{n,i=1} = \lbrace [\tau,\frac{1-\tau}{n-1},\cdots ,\frac{1-\tau}{n-1}] | \forall \tau\in[0,1] \rbrace$. Observe that $\ell c_{n,i=1}$ only includes $v_n(.)$ special distributions and hence we can rewrite the minimization as the following;
\begin{equation*}
    \begin{split}
    \text{proj}_{\ell c_{n,i=1}}(p(\sigma)) &= \arg\min_\tau \|  p(\sigma) -w_n(\tau)\|_2^2
    \end{split}
\end{equation*}
Equiting the first derivative to $0$;
\begin{equation*}
    \begin{split}
    \frac{\delta}{\delta_\tau}  \|p(\sigma) -v_n(\tau)\|_2^2 = m^T p(\sigma) - m^T  w_n(\tau) = 0 \\    \text{where } m = [1,\frac{-1}{n-1},\cdots\frac{-1}{n-1}]
    \end{split}
\end{equation*}
Calculating on an element basis;
\begin{equation*}
\begin{split}
    p(a) - \frac{p(b)}{n-1}- \cdots -\frac{p(z)}{n-1} = \tau + \frac{\tau -1}{(n-1)^2}(n-1)\\
    p(a) - \frac{1-p(a)}{n-1} = \tau + \frac{\tau -1}{(n-1)}\implies
    p(a) = \tau
\end{split}
\end{equation*}
Therefore, 
\begin{equation*}
\text{proj}_{\ell_{c,n,i=1}}(p(\sigma))=v_n(p(a))= [p(a), \frac{1-p(a)}{n-1},\cdots, \frac{1-p(a)}{n-1}]    
\end{equation*}
This proof can be extended to any distance metric defined over simplex following Barcelo's work \cite{barcelo2001mathematical} Section 3-4. The only requirement is to define invertible centering transform $H_D$ and using centering log-ratio transform. With these transforms one can find an approximate mapping between the simplex and $\ell_2$ and hence solve $\arg\min_\tau \|H_D \text{clr}(p(\sigma)) - H_D \text{clr}(w_n(\tau))\|_2^2$, the approximate solution is the same.
\end{proof}
\hrule
\begin{proof}[\textbf{Proof Proposition \ref{prop:getting_closer_mid}}]

Given $p(\sigma) = [p(a),p(b),\cdots, p(z)]$ and from Lemma \ref{lemma:projection_to_center} $\text{{proj}}_{\ell_{c,n,i=1}}p(\sigma) = w_n(p(a))$. Hence;
\begin{equation*}
\begin{split}
    \left\| p(\sigma)-w_n(p(a)) \right\|_2^2
    \\= \left\|\left[0,p(b)-\frac{1-p(a)}{n-1},\cdots,p(z)-\frac{1-p(a)}{n-1}\right]\right\|_2^2    \\
    = \sum_{c\in\mathcal{A}\setminus \lbrace a \rbrace}   \left(p(c)-\frac{1-p(a)}{n-1}\right)^2
    \leq \left(1-p(a)-\frac{1-p(a)}{n-1}\right)^2 \\
    = (1-p(a))^2  \frac{(n-2)^2}{(n-1)^2}\propto (1-p(a))^2 
\end{split}
\end{equation*}
%
% Similarly, Aitchison distance is defined as $D_a(x,y) = \| x\oplus (-1)\otimes y \|_a$, Hence;
% %
% \begin{equation*}
% \begin{split}
%     D_a(p(\sigma),w_n(p(a))) = \left\|\left[ \frac{1}{\mathcal{X}}, \frac{(n-1)p(b)}{(1-p(a))\mathcal{X}},\cdots, \frac{(n-1)p(x)}{(1-p(a))\mathcal{X}} \right]\right\|_a \ \text{where } \mathcal{X} \text{ is the normalization constant.}\\ 
%     \text{exp}(D_a(p(\sigma),w_n(p(a)))) = \exp{\mathcal{Y}}
% \end{split}
% \end{equation*}
In Lemma \ref{lemma:collinearity_of_post}, following the assumption $p(\sigma|\mathcal{H}_s)$ increases on average, using the averaged evidence for each state, query $(\sigma,\phi)$ brings posterior towards $\ell c_{n,i}$ closer each $s$.
\end{proof}
\hrule
\begin{obs}
    \label{obs:delta_also_bad}
    \begin{equation*}
    \begin{split}
    p,q\in\Delta_n, j_1 = \arg\max_i p_i, j_2 = \arg\max_{i\neq j_1} p_i\\
    k_1 = \arg\max_i q_i, k_2 = \arg\max_{i\neq k_1} q_i, \\
    I = \lbrace j_1,j_2,k_1,k_2 \rbrace, 
    p' = 1-\sum_{i\notin I} p_i, q' =1-\sum_{i\notin I} q_i\\
    \delta_{\textit{delta}^2}(p,q) = \frac{1}{2} \left[ \sum_{i\in I} | p_i - q_i | + |p' - q' | \right]
    \end{split}
\end{equation*}
    \emph{ $\bigcup_{k\in\lbrace 1,2,\cdots,n\rbrace} \tilde{B}^{\bar{\tau}}_\delta(c^k)$ with $\delta = \delta_{delta^2}$ yields confidence thresholding.}
\end{obs}
\begin{proof}
{Observation \ref{obs:delta_also_bad}}

WLOG $1^\text{st}$ location is the point of interest, hence the respective boundary includes probabilities with $1^\text{st}$ element being the maximum valued; $q,p\in\Delta_n$ where $q=[1,0,\cdots,0]$ and $p=[p_1,p_2,\cdots,p_n]$ with $1= \arg\max_i p_i$. Observe that;
\begin{equation*}
    \begin{split}
        &j_1 =\arg\max_i p_i = 1 \ \ \
        j_2 =\arg\max_{i\neq 1} p_i\\
        &k_1 = \arg\max_i q_i = 1 \ \ \ 
        k_2 = \arg\max_{i\neq 1} q_i = \hat{i}, \forall \hat{i}\in 2,3,\cdots,n \\
    \end{split}
\end{equation*}
WLOG pick $k_2 = j_2$. Hence $I= \lbrace 1,j_2 \rbrace \implies p' = 1- p_{j_1} - p_{j_2}, \ q' = 0$. Inserting these into the following delta divergence equation;
\begin{equation*}
    \begin{split}
     \delta_{\textit{delta}^2}(p,q) = \frac{1}{2} \left[ \sum_{i\in I} | p_i - q_i | + |p' - q' | \right] \\
     =  \frac{1}{2}(1-p_1+p_2 + 1-p_1-p_2) = 1-p_1
     \end{split}
\end{equation*}
Hence the decision condition $ \delta_{\textit{delta}^2}(p,q)<\bar{\tau}\implies p_{j_1}> 1-\bar{\tau} \implies \max_i p_i >1-\bar{\tau}$. Confidence thresholding and delta divergence are equivalent.
\end{proof}
\hrule
\begin{proof}[\textbf{Proof Proposition \ref{prop:delta_made_good}}]
Following Observation~\ref{obs:delta_also_bad};

WLOG $1^\text{st}$ location is the point of interest, hence the respective boundary includes probabilities with $1^\text{st}$ element being the maximum valued; $q,p\in\Delta_n$ where $q=[1,0,\cdots,0]$ and $p=[p_1,p_2,\cdots,p_n]$ with $1= \arg\max_i p_i$. Observe that;
\begin{equation*}
    \begin{split}
        &j_1 =\arg\max_i p_i = 1 \ \ \
        j_2 =\arg\max_{i\neq 1} p_i\\
        &k_1 = \arg\max_i q_i = 1 \ \ \ 
        k_2 = \arg\max_{i\neq 1} q_i = \hat{i}, \forall \hat{i}\in 2,3,\cdots,n \\
    \end{split}
\end{equation*}
WLOG pick $k_2 = j_2$. Hence $I= \lbrace 1,j_2 \rbrace \implies p' = 1- p_{j_1} - p_{j_2}, \ q' = 0$. Inserting these into the following modified delta divergence equation;

\begin{equation*}
    {\delta}_{\textit{MP}}(p,q) = \left[ \sum_{i\in I} | p_i - q_i | \right] =  1-p_{j_1} + p_{j_2}
\end{equation*}

Following the decision condition $ {\delta}_{\textit{MP}}(p,q)<\bar{\tau}\implies 1-p_{j_1}+p_{j_2}< \bar{\tau} \implies p_{j_1}-p_{j_2} > 1-\bar{\tau}$.

Therefore using the ball definition for $\delta= {\delta}_{\textit{MP}}$ and $q=[1,0,\cdots,0]$ $\tilde{B}^{\bar{\tau}}_\delta(q)= \lbrace x | \delta(q,x)<\bar{\tau}, x\in\Delta_n \rbrace\implies$
\begin{equation*}
\begin{split}
    \tilde{B}^{\bar{\tau}}_\delta(a)= \lbrace x \lvert x_j - x_k >1-\bar{\tau} , j=\arg\max_i x_i, k = \arg\max_{i\neq j} x_i\\
    , x\in\Delta_n \rbrace
    \end{split}
\end{equation*}
\end{proof}
\hrule
\begin{proof}[\textbf{Proof Observation \ref{obs:taubar_reqs}}]

Define the following two sets;
\begin{equation*}
    \begin{split}
        C_\tau = \lbrace p | \arg\max_i p_i = 1, \max_i p_i = \tau
        \rbrace\\
        B_{\bar{\tau}} = \left\lbrace p| p_1 - p_m=1-\bar{\tau}, 1=\arg\max_i p_i, m=\arg\max_{i\neq 1}p_i\right\rbrace
    \end{split}
\end{equation*}
Observe, $\max_{p\in B_{\bar{\tau}}}p_i = p_1$. We seek the following;
\begin{equation*}
    \begin{split}
    &\max_{p\in B_{\bar{\tau}}} p_1 \ \text{ s.t. } 1 = \arg\max_i p_i\\
    \implies &\max_{p\in B_{\bar{\tau}}} 1-\bar{\tau}+p_m \ \text{ s.t. } 1 = \arg\max_i p_i, \ m = \arg\max_{i\neq 1} p_i \\
    \implies &\max_{p\in B_{\bar{\tau}}} p_m \ \text{ s.t. } 1 = \arg\max_i p_i,\ m = \arg\max_{i\neq 1} p_i
    \end{split}
\end{equation*}
Trivially, $\max p_m = 1-p_1 \implies p_1 = 1-\bar{\tau}+1-p_1\rightarrow p_1 = 1-\bar{\tau}/2$. Observe that $p_i=0$ $\forall i\neq\lbrace 1,m \rbrace$ and hence $p=w_n(1-\bar{\tau}/2)$. If $\bar{\tau}=2\tau -2$ then $p = w_n(\tau)$ and $\arg\max_{p\in B_{2-2\tau}}p_1 = w_n(\tau)$.

Observe $w_n(\tau)\in C_\tau$, $w_n(\tau)\in B_{2-2\tau}$ and by the equality condition in both sets, $w_n(\tau) = (C_\tau\cap B_{2-2\tau})$.
\end{proof}
\hrule
\begin{proof}[\textbf{Proof Observation \ref{obs:taubar_reqs2}}]

Define the following;
\begin{equation*}
    \begin{split}
        B_{\bar{\tau}} = \left\lbrace p| p_1 - p_m=1-\bar{\tau}, 1=\arg\max_i p_i, m=\arg\max_{i\neq 1}p_i\right\rbrace
    \end{split}
\end{equation*}
By definition, $\min_{p\in B_{\bar{\tau}}}p_i = p_1$. $p\in B_{\bar{\tau}}\implies p_1 - p_m = 1 -\bar{\tau}$. Following a similar procedure as described in proof for Observation \ref{obs:taubar_reqs}. $\arg\min_p p_1 \sim \arg\min_p 1-\bar{\tau}+ p_m$, hence we seek to minimize $p_m$ such that $ m=\arg\max_{i\neq 1}p_i$. $p_m$ is trivially minimized by setting $p_m = (1-p_1)/(n-1)$ hence resulting in $p=v_n(p_1)$.

Observe that $p_1 - p_m = 1 - \bar{\tau} = p_1 - (1-p_1)/(n-1)\implies p_1  = (1 +\bar{\tau}(n-1))/n$.
\end{proof}
\hrule
\begin{lemma}
\label{prop:seq_in_constant_evidence}
$p(\sigma\lvert \varepsilon_{0:s})\in \Delta_n \forall s\in\lbrace 0,1,\cdots \rbrace$ with $p(\sigma\lvert\varepsilon_{0:s})= p(\sigma) \oplus p(\varepsilon_0|\sigma)\oplus \cdots \oplus p(\varepsilon_s|\sigma)$, and $\mathcal{R}_S$ denotes the stopping region s.t. $(S): p(\sigma\lvert \varepsilon_{0:s})\in \mathcal{R}_s$. WLOG, let $1^\text{th}$ location belong to the target class. Given $p(\varepsilon_s \lvert \sigma) \propto [\varepsilon, 1,\cdots,1]$ then the following relations hold;
\begin{equation*}
    \begin{split}
        (1)S_{R(M1)} = \left\lbrace p | p_1 \geq \tau, p\in\Delta_n \right\rbrace \\
        \implies \arg\min_s (p(\sigma\lvert \varepsilon_{0:s})\in \mathcal{R}_S) = \hat{s}>\log_{\varepsilon}\frac{(1-p_1)\tau}{(1-\tau)p_1} \\
        (2)S_{R(MP)} = \left\lbrace p | p_1-p_{\hat{i}} \geq \tau_{(2)}, \hat{i}=\arg\max_{i\neq 1} p_i, p\in\Delta_n \right\rbrace \\
        \implies \arg\min_s (p(\sigma\lvert \varepsilon_{0:s})\in \mathcal{R}_S) = \hat{s}>\log_{\varepsilon}\frac{(1-p_1)\tau_{(2)}+p_{\hat{i}}}{(1-\tau_{(2)})p_1}\\
        (3)S_{R(M2)} = \left\lbrace p | \| p \|_2 \geq \tau_{(3)}, p\in\Delta_n \right\rbrace \\
        \implies \arg\min_s (p(\sigma\lvert \varepsilon_{0:s})\in \mathcal{R}_S) = \hat{s} \\
        >\log_{\varepsilon}\frac{(1-p_1)(\tau_{(3)} + (2\tau_{(3)}-1)^{1/2})}{(1-\tau_{(3)})^2p_1}
    \end{split}
\end{equation*}
\end{lemma}

\begin{proof}

$p(\varepsilon_s \lvert \sigma) \propto [\varepsilon, 1,\cdots,1]\implies p(\sigma\lvert \varepsilon_{0:s})\propto [p_1 \varepsilon^s,p_2,\cdots,p_n]0$. WLOG assume $2= \arg\max_{i\neq1}p_i$;

\begin{equation*}
\begin{split}
    (1) \  p(\sigma\lvert \varepsilon_{0:s})\in S_{R(M1)} \\
    \implies \frac{p_1 \varepsilon^s}{p_1 \varepsilon^s + p_2 +p_3+\cdots+p_n}>\tau \\
    \rightarrow \frac{p_1 \varepsilon^s}{p_1 \varepsilon^s + (1-p_1)}>\tau \rightarrow s>\log_{\varepsilon}\frac{(1-p_1)\tau}{(1-\tau)p_1}\\
    (2) \  p(\sigma\lvert \varepsilon_{0:s})\in S_{R(MP)}\\
    \implies \frac{p_1 \varepsilon^s - p_2}{p_1 \varepsilon^s + p_2 +p_3+\cdots+p_n}>\tau_{(2)} \\
    \rightarrow \frac{p_1 \varepsilon^s-p_2}{p_1 \varepsilon^s + (1-p_1)}>\tau_{(2)} \rightarrow s>\log_{\varepsilon}\frac{(1-p_1)\tau_{(2)}+p_{2}}{(1-\tau_{(2)})p_1}
\end{split}
\end{equation*}
Similar approach is applied on $\|.\|_2^2$;
\begin{equation*}
\begin{split}
    (3) \  p(\sigma\lvert \varepsilon_{0:s})\in S_{R(M2)}\\ \implies \frac{(p_1 \varepsilon^s)^2 +{p_2}^2+\cdots+{p_n}^2}{(p_1 \varepsilon^s +(1-p_1))^2}\geq \tau_{(3)} \\
    \rightarrow 
    (\varepsilon^s)^2 {p_1}^2(1-\tau_{(3)}) +\varepsilon^s (-2p_1\tau_{(3)}(1-p_1) \\
    )+({p_2}^2+{p_3}^2+\cdots+{p_n}^2-\tau_{(3)}(1-p_1)^2)\geq0
\end{split}
\end{equation*}
Observe that the equation above is a quadratic polynomial of $\varepsilon^s$ and hence roots are obtained via the discriminant $\delta$;

\begin{equation*}.
\begin{split}
    \delta = \left( 4\tau_{(3)}^2{p_1}^2(1-p_1)^2 - 4{p_1}^2(1-\tau_{(3)})({p_2}^2+{p_3}^2\right.\\
    \left.+\cdots +{p_n}^2 - T(1-p_1)^2 \right)^{1/2}\\
    \geq \left( 4\tau_{(3)}^2{p_1}^2(1-p_1)^2 - 4{p_1}^2(1-\tau_{(3)})((1-p_1)^2\right.\\
    \left. - T(1-p_1)^2 \right)^{1/2} \\
    = 4{p_1}^2 (1-{p_1^2})(\tau_{(3)}^2 - (1-\tau_{(3)})^2)^{1/2}
\end{split}
\end{equation*}
$s\in\mathbb{N}\implies \delta\in\mathbb{R}\implies \tau_{(3)}>0.5\implies \delta \geq 2{p_1}^2(1-p_1)(2\tau-1)^{1/2}$. From this point one can state the following;
\begin{equation*}
\begin{split}
    \varepsilon^s \geq \frac{2\tau_{(3)} p_1 (1-p_1)+\delta}{2{p_1}^2(1-\tau_{(3)})^2}\\
    \geq
    \frac{2\tau_{(3)} p_1 (1-p_1)+ 4{p_1}^2 (1-{p_1^2})(\tau_{(3)}^2 - (1-\tau_{(3)})^2)^{1/2}}{2{p_1}^2(1-\tau_{(3)})^2}\\
    = \frac{(1-p_1)(\tau_{(3)} + (2\tau_{(3)}-1)^{1/2})}{p_1(1-\tau_{(3)})^2}\\
    \rightarrow  s \geq \log_{\varepsilon}\frac{(1-p_1)(\tau_{(3)} + (2\tau_{(3)}-1)^{1/2})}{p_1(1-\tau_{(3)})^2}
\end{split}
\end{equation*}
\end{proof}
\hrule
\begin{lemma}
\label{prop:seq_in_lognorm_pos_evidence}

Following Lemma \ref{prop:seq_in_constant_evidence}

$p(\sigma\lvert \varepsilon_{0:s})\in \Delta_n \forall s\in\lbrace 0,1,\cdots \rbrace$ with $p(\sigma\lvert\varepsilon_{0:s})= p(\sigma) \oplus p(\varepsilon_0|\sigma)\oplus \cdots \oplus p(\varepsilon_s|\sigma)$, and $\mathcal{R}_S$ denotes the stopping region s.t. $(S): p(\sigma\lvert \varepsilon_{0:s})\in \mathcal{R}_s$. WLOG, let $1^\text{th}$ location belong to the target class. Given $p(\varepsilon_s \lvert \sigma) \propto [\varepsilon, 1,\cdots,1]$ where $\varepsilon\sim \text{\emph{lognormal}}(\mu,c^2)$. Let \emph{erf}(.) denote the error function, then the following relations hold;
\begin{equation*}
    \begin{split}
        (1)S_{R(M1)} = \left\lbrace p | p_1 \geq \tau, p\in\Delta_n \right\rbrace \\
        \implies p(p(\sigma\lvert\varepsilon_{0:s})\in S_{R(M1)}) = \frac{1}{2}- \frac{1}{2}\text{\emph{erf}}\left( \frac{\log(k_1)-s\mu}{\sqrt{2s} c} \right) \\
        \text{\emph{where }} k_1 =((1-p_1)\tau)/((1-\tau)p_1)\\
        (2)S_{R(MP)} = \left\lbrace p | p_1-p_{\hat{i}} \geq \tau_{(2)}, \hat{i}=\arg\max_{i\neq 1} p_i, p\in\Delta_n \right\rbrace \\
        \implies p(p(\sigma\lvert\varepsilon_{0:s})\in S_{R(MP)}) = \frac{1}{2}- \frac{1}{2}\text{\emph{erf}}\left( \frac{\log(k_2)-s\mu}{\sqrt{2s} c} \right) \\
        \text{\emph{where }} k_2 =((1-p_1)\tau_{(2)}+p_{\hat{i}})/((1-\tau_{(2)})p_1)\\
        (3)S_{R(M2)} = \left\lbrace p | \| p \|_2 \geq \tau_{(3)}, p\in\Delta_n \right\rbrace \\
        \implies p(p(\sigma\lvert\varepsilon_{0:s})\in S_{R(M2)}) = \frac{1}{2}- \frac{1}{2}\text{\emph{erf}}\left( \frac{\log(k_3)-s\mu}{\sqrt{2s} c} \right) \\
        \text{\emph{where }} k_3 =\frac{(1-p_1)(\tau_{(3)} + (2\tau_{(3)}-1)^{1/2})}{(1-\tau_{(3)})^2p_1}\\
    \end{split}
\end{equation*}
\end{lemma}

\begin{proof}
Observe the following trivial steps;
\begin{equation*}
\begin{split}
    \varepsilon_i\sim \text{lognormal}(\mu,c^2\implies \left(\prod_{i=1}^s\right) \varepsilon_i \sim \text{lognormal}(s\mu,s c^2)\\ p(x>y) = 1- p(x<y)
\end{split}
\end{equation*}
Following the numbers obtained in Lemma\ref{prop:seq_in_constant_evidence}, it is trivial to calculate respective values using lognormal cdf.
\end{proof}
\hrule
%
%
%
% \begin{prop}
% \label{prop:comparison_true_positive}

% Following Lemma \ref{prop:seq_in_lognorm_pos_evidence};
% \begin{equation*}
%     \begin{split}
%         &\mathcal{R}_{S(1)} = \left\lbrace p | p_1 \geq \tau, p\in\Delta_n \right\rbrace\\
%         &\mathcal{R}_{S(\text{\emph{MP}})} = \left\lbrace p | p_1-p_{\hat{i}} \geq \tau_{(2)}, \hat{i}=\arg\max_{i\neq 1} p_i, p\in\Delta_n \right\rbrace \\
%     \end{split}
% \end{equation*}
% %
% If $\tau_{(2)}= 1-2\tau$, for a given $s$ probability of true positive for each case satisfies the following;
% %
% \begin{equation*}
%     p(p(\sigma|\varepsilon_{0:s})\in \mathcal{R}_{S(\text{\emph{MP}})}) > 
%       p(p(\sigma|\varepsilon_{0:s})\in \mathcal{R}_{S(1)})
% \end{equation*}
% \end{prop}

\begin{proof}[\textbf{Proof Proposition \ref{prop:perf_guarantees}}]
Following Lemma \ref{prop:seq_in_lognorm_pos_evidence}

This is trivially a comparison of $k_1,k_2$ presented in Lemma \ref{prop:seq_in_lognorm_pos_evidence}. Observe the following;
\begin{equation*}
\begin{split}
p(p(\sigma|\varepsilon_{0:s})\in {S}_{R(\text{{MP}})}) - 
      p(p(\sigma|\varepsilon_{0:s})\in S_{R(M1)}) \\
      = \left(- \text{{erf}}\left( \frac{\log(k_1)-s\mu}{\sqrt{2s} c} \right) \right) - \left( - \text{{erf}}\left( \frac{\log(k_2)-s\mu}{\sqrt{2s} c} \right)\right)
    % \mathcal{R}_{S(1)} \approx - \text{{erf}}\left( \frac{\log(k_1)-s\mu}{\sqrt{2s} c} \right) \ \ 
    % \mathcal{R}_{S(\text{MP})} \approx - \text{{erf}}\left( \frac{\log(k_2)-s\mu}{\sqrt{2s} c} \right) \ \ 
\end{split}
\end{equation*}
Observe that $\log(.)$ and $\text{erf}(.)$ are monotonically increasing. Hence $f(x) = - \text{{erf}}\left( \frac{\log(c)-s\mu}{\sqrt{2s} c} \right)$ is monotonically decreasing wrt. the argument $x$. Therefore, it is sufficient to compare $k_i \ \forall i$ to conclude the order between the identities. We calculate $k_i$ wrt. $\tau$ only;
\begin{equation*}
\begin{split}
    k_1 = \frac{(1-p_1)\tau}{(1-\tau)p_1}\\ 
    k_2 = \frac{(1-p_1)(1-\bar{\tau})+p_{\hat{i}}}{(1-\tau_{(2)})p_1}=\frac{(1-p_1)(2\tau-1) + p_{\hat{i}}}{(1-2\tau +1)p_1}\\
    = \frac{2(1-p_1)\tau - (1-p_1-p_{\hat{i}})}{2(1-\tau)p_1}\\
     = \frac{(1-p_1)\tau}{(1-\tau)p_1} - \frac{1-p_1-p_{\hat{i}}}{(1-\tau)p_1} \rightarrow p_1 + p_{\hat{i}} \leq 1 \implies k_2 <k_1\\
\end{split}
\end{equation*}
Hence this concludes $f(k_1)<f(k_2)\implies p(p(\sigma|\varepsilon_{0:s})\in S_{R(\text{{MP}})}) > 
      p(p(\sigma|\varepsilon_{0:s})\in S_{R(M1)})$ implying that the probability of true selection in (MP) is higher.

Accordingly probabilities of stopping with an incorrect decision is calculated as the following; WLOG choose $i=2$,
\begin{equation*}
\begin{split}
     p(\sigma|\varepsilon_{0:s})\in S'_{{R}(M1)} 
     \rightarrow \frac{p_2}{p_1 \varepsilon + (1-p_1)} >\tau\\
     \rightarrow \varepsilon < \frac{p_2 - (1-p_1)\tau}{\tau p_1}\implies
     p(p(\sigma|\varepsilon_{0:s})\in S'_{{R}(M1)}) \\
     = p\left(\varepsilon < \frac{p_2 - (1-p_1)\tau}{\tau p_1}\right)
\end{split}    
\end{equation*}
False alarm probability for (MP) is calculated as the following, probability of a competitor ($p_2$) exceeding the probability of the target class ($p_1$) by the margin. We use the relation $\bar{\tau} = 2-2\tau$ and write the following;
\begin{equation*}
\begin{split}
    p(\sigma|\varepsilon_{0:s})\in S'_{{R}(\text{{MP}})} \rightarrow \frac{p_2}{p_1 \varepsilon +(1-p_1)} - \frac{p_1\varepsilon}{p_1 \varepsilon +(1-p_1)}\\
    >(1-\bar{\tau}) \rightarrow \frac{p_2 - (1-p_1)(2\tau - 1)}{p_1(1+(2\tau -1))} > \varepsilon \\
    \rightarrow \frac{p_2 - (1-p_1)(2\tau-1)}{p_1(2\tau)}>\varepsilon \implies  p(p(\sigma|\varepsilon_{0:s})\in \S_{R(\text{\emph{MP}})})\\
    = p\left(\varepsilon < \frac{p_2 - (1-p_1)(2\tau-1)}{p_1(2\tau)}\right)
\end{split}
\end{equation*}
%
%
% Second, probability of the competitor ($p_2$) exceeding the probability of another class ($p_3$) by the margin;
% \begin{equation*}
% \begin{split}
%     &p(\sigma|\varepsilon_{0:s})\in \bar{\mathcal{R}}_{{S}(\text{{MP}})} \rightarrow \frac{p_2}{p_1 \varepsilon +(1-p_1)}- \frac{p_3}{p_1 \varepsilon +(1-p_1)} > \tau_{(2)} 
%     \rightarrow \frac{p_2-p_3 -(1-p_1)(2\tau-1)}{p_1(2\tau-1)}>\varepsilon\\
%     &\implies p(p(\sigma|\varepsilon_{0:s})\in \bar{\mathcal{R}}_{{S}(\text{\emph{MP}})}) = p\left(\varepsilon < \frac{p_2-p_3 -(1-p_1)(2\tau-1)}{p_1(2\tau-1)}\right)
% \end{split}
% \end{equation*}
Observe that, this happens due to $p_2$ being the second highest competitor in the classification task. Following previous statements, one can compare values for $\varepsilon$ cdfs for (M1), (MP) and $\bar{\text{(MP)}}$. The probabilities are listed in the following using the definitions $\bar{\tau} = 2 - 2\tau$ and $\tilde{\tau} = \frac{(n-1)(2\tau-1)+ 1}{n}$;
\begin{equation*}
    \begin{split}
       p(p(\sigma|\varepsilon_{0:s})\in S'_{{R}(M1)}) (\tau) = p\left(\varepsilon < \frac{p_2 -    (1-p_1)\tau}{\tau p_1}\right)\\
       p(p(\sigma|\varepsilon_{0:s})\in \bar{\mathcal{R}}_{{S}(\text{\emph{MP}})}) (\bar{\tau}) = p\left(\varepsilon < \frac{p_2 - (1-p_1)(2\tau-1)}{p_1(2\tau)}\right) \\ 
       p(p(\sigma|\varepsilon_{0:s})\in S'_{{R}\bar{(M1)}}) (\tilde{\tau}) \\
       = p\left(\varepsilon < \frac{ n p_2 -    (1-p_1)((n-1)(2\tau-1)+ 1)}{((n-1)(2\tau-1)+ 1) p_1}\right)\\
    \end{split}
\end{equation*}
Furthermore we define the following;
\begin{equation*}
    \begin{split}
        &k_1' = \frac{p_2 - (1-p_1)\tau}{\tau p_1}\\
        &k_2' = \frac{p_2 - (1-p_1)(2\tau-1)}{p_1(2\tau)}\\
        &\bar{k}_1' = \frac{ n p_2 - (1-p_1)((n-1)(2\tau-1)+ 1)}{((n-1)(2\tau-1)+ 1) p_1}\\
    \end{split}
\end{equation*}
Observe the following relations;
\begin{equation*}
    \begin{split}
        &k_2' - k_1' = \frac{1-p_1-p_2}{2\tau p_1}\geq0 \implies k_2'> k_1' \\
        \end{split}    
\end{equation*}
and,
\begin{equation*}
    \begin{split}
        \frac{\delta \bar{k}_1'}{\delta n} = \frac{2p_2 (1-\tau)}{p_1(2\tau(n-1)-n+2)^2} \geq 0 \implies \frac{\bar{k}_1'}{\delta n} 
        \big\lvert_{n=2}\\
        \implies p_2 = 1-p_1 \implies
        k'_2 = \frac{p_2 - (1-p_1)(2\tau-1)}{p_1(2\tau)} \\
        = \frac{(1-p_1)(2-2\tau)}{2p_1\tau} = \frac{(1-p_1)(1-\tau)}{\tau p_1} \\
        \& \ \        \bar{k}_1' = \frac{2p_2 - (1-p_1)(2\tau)}{(2\tau)p_1} = \frac{(1-p_1)(1-\tau)}{\tau p_1} \\
        \rightarrow k'_2 =  \left.\bar{k}_1' \right\lvert_{n=2}        \implies\bar{k}_1'> k_2' \text{ from monotonic increasing.}
    \end{split}    
\end{equation*}
Therefore we conclude $\bar{k}_1'> k_2'>k_1'$. Similarly with the first part of the proof this implies; $p(p(\sigma|\varepsilon_{0:s})\in S'_{{R}\bar{(M1)}}) (\tilde{\tau}) >  p(p(\sigma|\varepsilon_{0:s})\in S_{R(\text{\emph{MP}})}) (\bar{\tau})>   p(p(\sigma|\varepsilon_{0:s})\in S'_{{R}(M1)})(\tau)$ implying that the incorrect selection can be sandwiched.
\end{proof}
\hrule

\clearpage

 }

\end{document}